\pdfoutput=1
\documentclass[11pt]{article}
\usepackage{fullpage}
\usepackage{xcolor}
\usepackage{amsfonts}
\usepackage{amsmath}
\usepackage{amsthm}
\usepackage{natbib}
\usepackage{booktabs}
\usepackage[ruled]{algorithm2e}
\usepackage{thm-restate}
\usepackage[pagebackref,colorlinks=true,linkcolor=blue,citecolor=blue,linktocpage]{hyperref}
\usepackage{cleveref}
\usepackage{tablefootnote}
\usepackage{graphicx}
\graphicspath{ {./imgs/} }
\SetKw{Break}{break}
\title{
Comparative Learning: A Sample Complexity Theory for Two Hypothesis Classes
}
\author{Lunjia Hu\thanks{Stanford University. Supported by Moses Charikar’s and Omer Reingold's Simons Investigators awards and Omer Reingold’s NSF Award IIS-1908774. Email: \texttt{lunjia@stanford.edu}} \and Charlotte Peale\thanks{Stanford University. Supported by the Simons Foundation Collaboration on the Theory of Algorithmic Fairness. Email: \texttt{cpeale@stanford.edu}}}
\date{}

\newcommand{\vc}{{\mathsf{VC}}}
\newcommand{\error}{{\mathsf{error}}}
\newcommand{\bE}{{\mathbb E}}
\newcommand{\bZ}{{\mathbb Z}}
\newcommand{\bR}{{\mathbb R}}
\newcommand{\mae}{{\textnormal{\sffamily MA-error}}}
\newcommand{\mce}{{\textnormal{\sffamily MC-error}}}

\newcommand{\sps}[1]{^{({#1})}}
\newcommand{\sign}{{\mathsf{sign}}}
\newcommand{\ldim}{{\mathsf{Ldim}}}
\newcommand{\lone}{{\ell_1\textnormal{\sffamily -error}}}
\newcommand{\one}{{\mathbf 1}}
\newcommand{\fat}{{\mathsf{fat}}}
\newcommand{\rea}{{\mathsf{ReaL}}}
\newcommand{\srea}{{\#\mathsf{ReaL}}}
\newcommand{\agn}{{\mathsf{AgnL}}}
\newcommand{\sagn}{{\#\mathsf{AgnL}}}
\newcommand{\comp}{{\mathsf{CompL}}}
\newcommand{\scomp}{\#{\mathsf{CompL}}}
\newcommand{\cm}{{\mathsf{CorM}}}
\newcommand{\scm}{{\#\mathsf{CorM}}}
\newcommand{\dcm}{{\mathsf{DCorM}}}
\newcommand{\sdcm}{{\#\mathsf{DCorM}}}
\newcommand{\wcm}{{\textnormal{\sffamily W-CorM}}}
\newcommand{\wdcm}{{\textnormal{\sffamily W-DCorM}}}
\newcommand{\wcomp}{{\textnormal{\sffamily W-CompL}}}
\newcommand{\swcm}{{\#\textnormal{\sffamily W-CorM}}}
\newcommand{\swdcm}{{\#\textnormal{\sffamily W-DCorM}}}
\newcommand{\compr}{{\mathsf{CompR}}}
\newcommand{\scompr}{{\#\mathsf{CompR}}}
\newcommand{\learn}{{\mathsf{Learn}}}
\newcommand{\slearn}{\#{\mathsf{Learn}}}
\newcommand{\ptimes}{{\,\diamondsuit\,}}
\newcommand{\ba}{{\mathbf a}}
\newcommand{\bA}{{\mathbf A}}
\newcommand{\by}{{\mathbf y}}
\newcommand{\agree}{{\mathsf{AgreeL}}}
\newcommand{\sagree}{{\#\mathsf{AgreeL}}}
\newcommand{\ma}{{\mathsf{MA}}}
\newcommand{\sma}{\#{\mathsf{MA}}}
\newcommand{\mc}{{\mathsf{MC}}}
\newcommand{\smc}{\#{\mathsf{MC}}}

\newcommand{\proj}{{\mathsf{proj}}}
\newcommand{\diff}{{\mathrm d}}
\newcommand{\for}{{$\mathbf{for}$\ }}
\newcommand{\bif}{{$\mathbf{if}$\ }}

\newcommand{\mistake}{{\mathsf{mistake}}}
\newcommand{\ber}{{\mathsf{Ber}}}
\newcommand{\reao}{{\mathsf{ReaOL}}}
\newcommand{\agno}{{\mathsf{AgnOL}}}
\newcommand{\compo}{{\mathsf{CompOL}}}
\newcommand{\blambda}{{\mathbf{\Lambda}}}
\newcommand{\bsigma}{{\boldsymbol\sigma}}
\newcommand{\p}{{\pi}}
\newcommand{\while}{{$\mathbf{while}$\ }}
\newcommand{\ce}{{\textnormal{\sffamily C-error}}}
\newcommand{\dstr}{{P}}
\newcommand{\sce}{{\textnormal{\sffamily sign-C-error}}}
\newcommand{\unif}{{\mathsf{unif}}}

\newtheorem{theorem}{Theorem}[section]
\newtheorem{claim}[theorem]{Claim}
\newtheorem{lemma}[theorem]{Lemma}
\newtheorem{definition}{Definition}[section]
\newtheorem{remark}{Remark}[section]
\newtheorem{question}{Open Question}[section]

\allowdisplaybreaks
\begin{document}
\maketitle
\thispagestyle{empty}
\begin{abstract}
In many learning theory problems, a central role is played by a hypothesis class: we might assume that the data is labeled according to a hypothesis in the class (usually referred to as the \emph{realizable setting}), or we might evaluate the learned model by comparing it with the best hypothesis in the class (the \emph{agnostic setting}).
Taking a step beyond these classic setups that involve only a single hypothesis class, we study a variety of problems that involve two hypothesis classes simultaneously.

We introduce \emph{comparative learning} as a combination of the realizable and agnostic settings in PAC learning: given \emph{two} binary hypothesis classes $S$ and $B$, we assume that the data is labeled according to a hypothesis in the \emph{source class} $S$ and require the learned model to achieve an accuracy comparable to the best hypothesis in the \emph{benchmark class} $B$. 
Even when both $S$ and $B$ have infinite VC dimensions, comparative learning can still have a small sample complexity.
We show that the sample complexity of comparative learning is characterized by the \emph{mutual VC dimension} $\vc(S,B)$ which we define to be the maximum size of a subset shattered by both $S$ and $B$. 
We also show a similar result in the online setting, where we give a regret characterization in terms of the analogous \emph{mutual Littlestone dimension} $\ldim(S,B)$. These results also hold for partial hypotheses.

We additionally show that the insights necessary to characterize the sample complexity of comparative learning can be applied to other tasks involving two hypothesis classes.
In particular, we characterize the sample complexity of realizable \emph{multiaccuracy} and \emph{multicalibration} using the \emph{mutual fat-shattering dimension}, an analogue of the mutual VC dimension for real-valued hypotheses. This not only solves an open problem proposed by Hu, Peale, Reingold (2022), but also leads to independently interesting results extending classic ones about regression, boosting, and covering number to our two-hypothesis-class setting. 
\end{abstract}
\newpage
\tableofcontents
\newpage
\section{Introduction}

The seminal theoretical framework of \emph{PAC learning} \citep{valiant1984theory} provides a formalization of machine learning that allows for rigorous theoretical analysis. In PAC learning, a learning algorithm (learner) receives individual/label pairs $(x, y) \in X \times \{-1, 1\}$ as input data, drawn i.i.d.\ from an unknown distribution $\mu$. The learner's goal is to output a model $f: X \rightarrow \{-1, 1\}$ that assigns each individual in $X$ a binary label. The performance of the model $f$ is measured by its classification error, 
$$\error(f) := {\Pr}_{(x,y)\sim\mu}[f(x)\ne y].$$
Because the classification error is evaluated over the entire distribution $\mu$, a good learner must go beyond simply memorizing the individuals and labels seen in the input data and be able to correctly predict the labels of unseen individuals as well. This can be a difficult task, and to make it possible to achieve a meaningfully small error given a limited amount of input data, additional assumptions or relaxations are needed. This leads to two
standard settings of PAC learning: \emph{realizable} and \emph{agnostic} learning. In realizable learning, we assume that all data points are labeled according to an unknown hypothesis $h:X\rightarrow\{-1,1\}$, i.e., $y = h(x)$ for every data point $(x,y)$ drawn from $\mu$, and we assume that $h$ belongs to a hypothesis class $H$ known to the learner. Under this assumption, realizable learning requires the output model $f$ to achieve a low classification error ($\error(f)\le\varepsilon$) with large probability. In agnostic learning there is also a hypothesis class $H$ known to the learner, but it does not impose any assumption on the data.
Instead,
we aim for a relaxed goal specified by $H$: achieving $\error(f) \le \inf_{h\in H}\error(h) + \varepsilon$ with large probability.

At a high level, both realizable and agnostic learning involve the introduction of a hypothesis class $H$, but $H$ plays a very different role in each setting. In realizable learning, $H$ constrains the 
potential \emph{source hypotheses} that might determine the ground-truth labeling of the data. In contrast, agnostic learning places no assumptions on the ground-truth labeling, but instead uses $H$ as a \emph{benchmark class} and only requires the learner to perform well compared to the best benchmark hypothesis in $H$. Thus, realizable and agnostic learning highlight two natural ways to simplify a learning task: 
constrain the potential hypotheses that the ground-truth labeling is generated from, or constrain the set of hypotheses that the output model is compared against.

Our work originates from the observation that these two ways of simplifying a learning task need not be mutually exclusive. Instead, they can be treated as two ``knobs'' that can be simultaneously adjusted to create new hybrid learning tasks. For any \emph{two} hypothesis classes $S$ and $B$, we can define a learning task by letting them play the two roles of $H$ in the realizable and agnostic settings, respectively. That is, we assume that there exists a \emph{source hypothesis} $s\in S$ such that $y = s(x)$ for every data point $(x,y)$ drawn from $\mu$, and we aim for achieving, with large probability, an error comparable to the best \emph{benchmark hypothesis} $b\in B$: $\error(f) \le \min_{b\in B}\error(b) + \varepsilon$. We term this hybrid notion \emph{comparative learning}.

Our research reveals that the notion of comparative learning is far more insightful than just a thought experiment: it serves as an unexplored playground for the study of \emph{sample complexity}, and the new connections we establish to characterize the sample complexity of comparative learning can be fruitfully applied to open questions about existing learning tasks.
Here, ``sample complexity'' refers to one of the key characteristics of every learning task: the minimum number of data points needed by a learner to solve the task.
VC theory provides a thorough understanding of the sample complexity of classic PAC learning in both the realizable and agnostic settings:
in both cases it is characterized by the \emph{VC dimension} of the hypothesis class $H$, defined as the maximum size of a subset of $X$ on which all possible labelings of the individuals can be realized by some hypothesis in $H$ (we say a set is \emph{shattered} by $H$ when this condition holds; see \Cref{sec:preli} for the exact definition)
\citep{vapnik1971vcdim,MR1072253,MR1088804}.
Since then, understanding the sample complexity of a wide variety of new and existing learning tasks has remained an exciting area of research. 
These tasks include 
online learning \citep{littlestone1988learning,ben2009agnostic,MR4398855,filmus2022optimal}, 
reliable and useful learning \citep{MR1109724,MR1076241,MR1303564,kivinen1990reliable},
statistical query learning \citep{kearns1993efficient,blum1994weakly},
learning real-valued hypotheses \citep{MR1279411,MR1328428,MR1408000}, 
multiclass learning \citep{MR1322634,brukhim2022characterization}, 
learning partial hypotheses \citep{MR2042042,MR4399723},
active learning \citep{MR2472318,MR3108162,MR3734243,hopkins2020noise,MR4232108,hopkins2020power}, property testing \citep{MR1450632,MR1800309,MR4398860}, differentially private learning \citep{MR4003389,MR4232052,MR4398834,sivakumar2021multiclass,jung2020equivalence,golowich2021differentially}, 
bounded-memory learning \citep{gonen2020towards},
and online learning in the smoothed analysis model \citep{haghtalab2020smoothed,MR4399748}.
A commonality of these learning tasks is that each of them only explicitly involves a single hypothesis class, and thus the sample complexity 
is studied in terms of complexity measures of single hypothesis classes,
such as the VC dimension, the Littlestone dimension, the statistical query dimension, the fat-shattering dimension, and the DS dimension. To tightly characterize the sample complexity of comparative learning where a \emph{pair} of hypothesis classes $S$ and $B$ are involved, it is not sufficient to apply existing complexity measures to $S$ and $B$ separately (see \Cref{sec:intro-duality} for a more detailed discussion). Instead, we must create new notions that measure the complexity of the \emph{interaction} between the two classes. 
We show that the correct way to measure the complexity of this interaction in comparative learning is to look at the subsets  of $X$ that $S$ and $B$ \emph{both} shatter, and we define the \emph{mutual VC dimension}, $\vc(S , B)$, to be the maximum size of such subsets.
We show that the mutual VC dimension gives both upper and lower bounds on the sample complexity of comparative learning.
Similarly, in an online analogue of comparative learning, we define the \emph{mutual Littlestone dimension} and prove upper and lower \emph{regret} bounds.

Our sample complexity characterization for comparative learning turns out to be a powerful tool for studying the sample complexity of other tasks involving two hypothesis classes. In fact, 
our interest in comparative learning is derived in part from open questions related to the sample complexity of realizable \emph{multiaccuracy} (MA) and \emph{multicalibration} (MC)
\citep{hebert2018multicalibration,kim2019multiaccuracy,hu2022metric}. 
In these tasks,
the hypothesis class $H$ plays the same role as in realizable learning, while the classification error $\error(f)$ is replaced with an alternative error measure $\mae_D(f)$ or $\mce_D(f)$ specified by an additional hypothesis class $D$ that is sometimes called the \emph{distinguisher class}.\footnote{
The name ``distinguisher class'' comes from the observation that the \emph{no-access outcome indistinguishability} task studied in \citep{hu2022metric} can be equivalently framed as multiaccuracy \cite[see][Section 2.1.2]{hu2022metric}.
In addition to the difference in the error from realizable learning, realizable multiaccuracy and multicalibration also allow the hypothesis class $H$ and the model $f$ to be real-valued (see \Cref{sec:MA-MC} and \Cref{sec:ma-mc}).
It is also possible to replace the error in \emph{agnostic} learning with $\mae$ and $\mce$ to get agnostic multiaccuracy and agnostic multicalibration, but \citet{hu2022metric} show that the sample complexity of agnostic multiaccuracy exhibits a non-monotone dependence on the complexity of the distinguisher class $D$. We focus on defining multiaccuracy and multicalibration in the realizable setting throughout the paper. 
}
For example, the multiaccuracy error $\mae_D(f)$ is defined as follows:
\[
\mae_D(f) := {\sup}_{d\in D}|\bE_{(x,y)\sim\mu}[(f(x) - y)d(x)]|,
\]
where the supremum is over all the \emph{distingushers} $d:X\rightarrow [-1,1]$ in the distinguisher class $D$.
As demonstrated by \citet*{hu2022metric}, the freedom in choosing the class $D$ allows the error to adapt to different goals that may arise in practice. 

The introduction of the distinguisher class $D$ makes sample complexity characterization challenging because
the characterization needs to depend on both the class $H$ in realizable learning and the additional distinguisher class $D$. \citet*{hu2022metric} give a sample complexity characterization for realizable multiaccuracy using a particular \emph{metric entropy} defined for every pair $(H,D)$ (see \Cref{sec:MA-MC} for more details), but their characterization is in the \emph{distribution-specific} setting where the marginal distribution $\mu|_X$ of $x$ in a pair $(x,y)$ generated from the data distribution $\mu$ is fixed and known to the learner. In contrast, the VC dimension characterization for PAC learning is in the \emph{distribution-free} setting where the learner has no explicit knowledge about $\mu|_X$ and must perform well for every $\mu|_X$. 
The sample complexity characterization for realizable multiaccuracy in the distribution-free setting is left as an open question by \citet*{hu2022metric}.

In this work, we answer this open question by characterizing the sample complexity of realizable multiaccuracy and multicalibration in the distribution-free setting using the \emph{mutual fat-shattering dimension}, which we define similarly to the mutual VC dimension but for two \emph{real-valued} hypothesis classes.
Our results on comparative learning turn out to be especially useful for obtaining this characterization
because there is an intimate relationship between achieving the comparative learning goal $\error(f) \le \min_{b\in B}\error(b) + \varepsilon$ and achieving a low multiaccuracy (or multicalibration) error: $\mae_B(f) \le \varepsilon$. Here, the benchmark class $B$ in comparative learning plays the role of the distinguisher class $D$ in multiaccuracy and multicalibration. 
This relationship has been observed by \citet*{hebert2018multicalibration} and \citet*{DBLP:conf/innovations/GopalanKRSW22} in a single-hypothesis-class setting, i.e., without the assumption that the labels from the data distribution are generated according to a hypothesis in a pre-specified source class.
We generalize this relationship to our two-hypothesis-class setting by showing a reduction from realizable multiaccuracy and multicalibration to comparative learning while preserving the interaction between the source and distinguisher/benchmark classes.
This reduction leads to a number of new learning tasks that also involve a pair of hypothesis classes.
Specifically, the reduction is accomplished via an intermediate task which we call \emph{correlation maximization}, and we show that with some adaptation the reduction also allows us to efficiently \emph{boost} a \emph{weak} comparative learner to a strong one. Once we achieve multiaccuracy and multicalibration, we apply the \emph{omnipredictor} result by \citet*{DBLP:conf/innovations/GopalanKRSW22} to solve \emph{comparative regression}, an analogue of comparative learning but with real-valued hypotheses and general convex and Lipschitz loss functions. We believe that there is a rich collection of learning tasks where two or more hypothesis classes may interact in  interesting ways, and our work is just a small step towards a better understanding of a tiny fraction of these tasks.

\subsection{Sample Complexity of Comparative Learning}
\label{sec:intro-duality}
As mentioned earlier, VC theory has provided a thorough understanding of the sample complexities of both realizable and agnostic learning. %

For any binary hypothesis class $H$, VC theory characterizes the sample complexity of realizable and agnostic learning using the VC dimension $\vc(H)$ of the hypothesis class $H$, a combinatorial quantity with a simple definition: the maximum size of a subset of $X$ \emph{shattered} by $H$ (see \Cref{sec:preli} for exact definition) \citep{vapnik1971vcdim,MR1072253,MR1088804}. Moreover, the optimal sample complexity in both the realizable and agnostic settings can be achieved by a simple algorithm: the empirical risk minimization algorithm (ERM), which outputs the hypothesis in $H$ with the minimum empirical error on the input data points.

Because our notion of comparative learning combines these two settings, it would seem natural to use techniques from VC theory to understand its sample complexity as well. Compared to realizable learning for $S$, comparative learning for $(S,B)$ has a relaxed goal (specified by the benchmark class $B$), and thus any learner solving realizable learning for $S$ also solves comparative learning for $(S,B)$. This gives us a sample complexity upper bound in terms of $\vc(S)$ for comparative learning. Similarly, any learner solving agnostic learning for $B$ also solves comparative learning for $(S,B)$ because comparative learning only makes additional assumptions on data (specified by the source class $S$), so we get another sample complexity upper bound in terms of $\vc(B)$.

However, perhaps surprisingly, these sample complexity upper bounds provided by the classic VC theory are not optimal.
Even when $\vc(S)$ and $\vc(B)$ are both infinite, comparative learning may still have a finite sample complexity. %
Imagine that the domain $X$ of individuals is partitioned into two large subsets $X_1$ and $X_2$. Suppose the source class $S$ consists of all binary hypotheses $s:X\rightarrow\{-1,1\}$ satisfying $s(x) = 1$ for every $x\in X_1$, and the benchmark class $B$ consists of all binary hypotheses $b:X\rightarrow\{-1,1\}$ satisfying $b(x) = 1$ for every $x\in X_2$. Both $\vc(S)$ and $\vc(B)$ can be large and even infinite, but comparative learning in this case requires no data points: the learner can simply output the model $f$ that maps every $x\in X$ to $1$ because no benchmark hypothesis in $B$ can achieve a smaller error than $f$ when the data points $(x,y)\sim \mu$ satisfy $y = s(x)$ for a source hypothesis $s\in S$ (see Figure~\ref{fig:disjoint_cxty}).
Beyond demonstrating that comparative learning may require far fewer samples than what our initial na\"ive upper bound might suggest, this example also shows that the standard empirical risk minimization (ERM) algorithm used for PAC learning does not give us the optimal sample complexity for comparative learning.
Assume that the source hypothesis $s\in S$ maps every $x\in X$ to $1$ and $\mu$ is the uniform distribution over $X\times \{1\}$. In this case $\min_{b\in B}\error(b) = 0$ and thus comparative learning requires a low classification error $\error(f) \le \varepsilon$ with large probability. We have shown that this requirement can be achieved without any input data points, but the ERM algorithm cannot achieve this requirement in general unless there are many input data points:
there can be many hypotheses in $S$ and $B$ that achieve zero empirical error on the input data points, but when the data points are few, most of such hypotheses do not achieve low classification error over the entire distribution $\mu$ (see Figure~\ref{fig:failed_erm}).

\begin{figure}
    \centering
    \includegraphics[width=\textwidth]{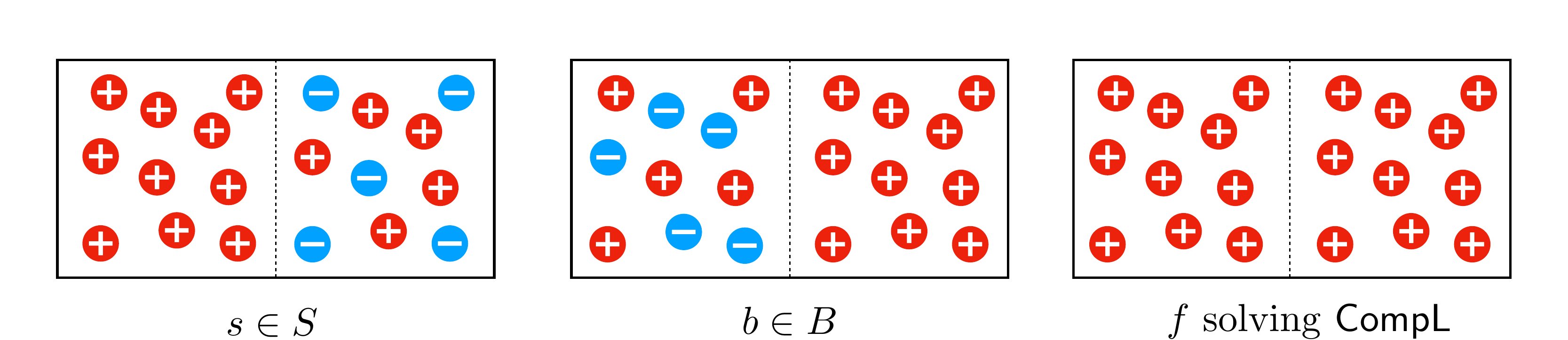}
    \caption{An example where comparative learning requires no data points when $\vc(S)$ and $\vc(B)$ are both infinite. The left two images show examples of hypotheses in $S$ and $B$, both of which are very complex, but on disjoint portions of the domain. In this case, 
a learner that always outputs the model $f$ in the rightmost image solves comparative learning because $f$ always achieves smaller or equal error compared to any benchmark hypothesis $b\in B$ when the ground-truth labelling is generated by a source hypothesis $s\in S$. See in-text description for more details.}
    \label{fig:disjoint_cxty}
\end{figure}

\begin{figure}
    \centering
    \includegraphics[width=\textwidth]{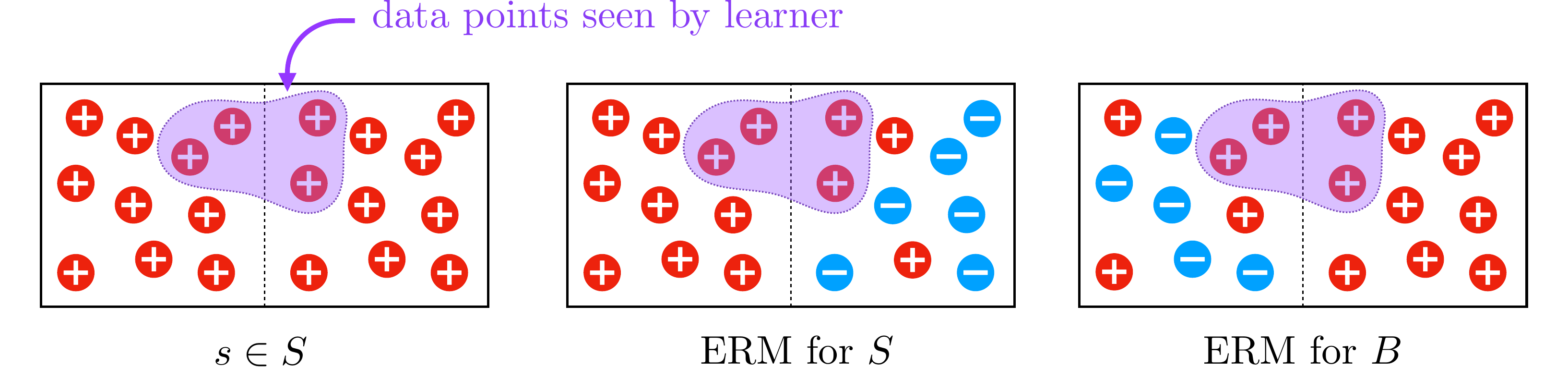}
    \caption{Empirical risk minimization (ERM) may fail to give us optimal sample complexity in the same setting as Figure~\ref{fig:disjoint_cxty}, where $S$ and $B$ are both very complex, but on disjoint domains. When the source hypothesis $s\in S$ is the constant function shown in the left image, the right two images show examples of output models of ERM when run on $S$ and $B$. Neither model is guaranteed to achieve the low error required by comparative learning. See in-text description for more details.
    }
    \label{fig:failed_erm}
\end{figure}

The example above shows that the VC dimensions $\vc(S)$ and $\vc(B)$ alone are not informative enough to characterize the sample complexity of comparative learning. These VC dimensions only tell us the complexity of $S$ and $B$ \emph{separately}, but we also need to know the complexity of their \emph{interaction}. 
We measure the complexity of this interaction by defining the \emph{mutual VC dimension} $\vc(S , B)$ to be the maximum size of a subset of $X$ shattered by \emph{both} $S$ and $B$, and we give a tight characterization for the sample complexity of comparative learning in terms of $\vc(S,B)$.

As discussed earlier, new ideas are needed to prove this sample complexity characterization. In particular, we need to design a learner that is different from the ERM algorithm. Our technique is based on an interesting connection to \emph{learning partial binary hypotheses}, a learning task considered first by \citet{bartlett1995more,MR2042042} and studied more systematically in a recent work by \citet*{MR4399723}. A partial binary hypothesis is a function $h:X\rightarrow\{-1,1,*\}$ that may assign some individuals $x\in X$ the \emph{undefined label} $h(x) = *$.
The notion of partial hypotheses is motivated in previous work either as an intermediate step towards understanding \emph{real-valued} hypotheses or as a way to describe data-dependent assumptions that could not be captured by the standard PAC learning model.
In this work, we show that partial hypotheses have yet another application and can be used to express the interaction between a source hypothesis $s\in S$ and a benchmark hypothesis $b\in B$ in comparative learning:
we construct an \emph{agreement hypothesis} $\ba_{s,b}$ which is a partial hypothesis assigning the undefined label $*$ to an individual $x$ whenever $s(x)$ is different from $b(x)$, and giving the same label to $x$ as $s$ and $b$ if $s(x)$ equals $b(x)$. We show that comparative learning for $(S,B)$ can be reduced to \emph{agnostically} learning the class $\bA_{S,B}$ which consists of all the partial hypotheses $\ba_{s,b}$ for $s\in S$ and $b\in B$, and conversely, we show that \emph{realizable} learning for $\bA_{S,B}$ reduces to comparative learning for $(S,B)$. Our sample complexity characterization for comparative learning then follows immediately from the results by \citet{MR4399723} for learning partial hypotheses. 
Moreover, our characterization holds even when the source hypotheses and benchmark hypotheses themselves are partial. We also show that this connection between comparative learning and learning the agreement hypotheses $\ba_{s,b}$ extends to the online setting, allowing us to show a \emph{regret} characterization for \emph{comparative online learning}.

Our definition of the mutual VC dimension is clearly symmetric: $\vc(S,B) = \vc(B,S)$, and thus our sample complexity characterization for comparative learning reveals an intriguing phenomenon which we call \emph{sample complexity duality}: comparative learning for $(S,B)$ and comparative learning for $(B,S)$ always have similar sample complexities. In other words, swapping the roles of the source class and the benchmark class does not change the sample complexity by much.
Previously, \citet*{hu2022metric} show that this phenomenon holds for realizable multiaccuracy in 
the distribution-specific setting, drawing an insightful connection to a long-standing open question in convex geometry: the \emph{metric entropy duality conjecture} \citep{MR0361822,MR1008716,MR2113023,MR2105957,MR2296760}.
Our sample complexity characterizations imply that 
sample complexity duality also holds in the distribution-\emph{free} setting for realizable multiaccuracy as well as multicalibration.
We also show that sample complexity duality does \emph{not} hold for many learning tasks that we consider, including distribution-specific comparative learning, distribution-specific realizable multicalibration, correlation maximization, and comparative regression. In \Cref{table:duality} we list whether sample complexity duality holds in general for every two-class learning task we consider in this paper in both the distribution-specific and distribution-free settings. %
\begin{table}
\centering
\begin{tabular}{lll}
\toprule
& Distribution-specific%
& Distribution-free  \\ \hline
Comparative learning & no & yes\\
Correlation maximization & no & no\\
Realizable multiaccuracy & yes* \citep{hu2022metric} & yes\\
Realizable multicalibration & no & yes \\
Comparative regression & no & no\\
\bottomrule
\end{tabular}
\caption{Duality (yes) VS non-duality (no). *The sample complexity duality result in \citep{hu2022metric} for distribution-specific multiaccuracy assumes that all hypotheses are total.}
\label{table:duality}
\end{table}

\subsection{Multiaccuracy and Multicalibration}
\label{sec:MA-MC}
A direct motivation of our work is a recent paper by \citet*{hu2022metric} that studies a learning task called \emph{multiaccuracy}, which was introduced by \citet*{hebert2018multicalibration} and \citet*{kim2019multiaccuracy} originally as a notion of multi-group fairness. In multiaccuracy, the learned model (presumably making predictions about people) is required to be accurate in expectation when conditioned on each sub-community in a rich class (possibly defined based on demographic groups and their intersections). This ensures that the predictions made by the model are not systematically biased in any of the sub-communities.

Taking a broad perspective beyond fairness, \citet*{hu2022metric} view multiaccuracy as providing a general, meaningful, and flexible performance measure for prediction models, and study PAC learning with the usual classification error replaced by this new performance measure from multiaccuracy. To be specific, let us consider a \emph{real-valued} source hypothesis class $S$ consisting of source hypotheses $s:X\rightarrow[-1,1]$. We use $S$ to replace the binary hypothesis class $H$ in realizable learning and
assume that every input data point $(x,y)\in X\times[-1,1]$ is generated i.i.d.\ from a distribution $\mu$ satisfying $\bE_{(x,y)\sim\mu}[y|x] = s(x)$ for an unknown $s\in S$. Suppose a learner which tries to learn $s$ given the input data points produces an output model $f:X\rightarrow[-1,1]$. This is a more general setting than binary classification because we allow $f(x)$ and $s(x)$ to take any value in the interval $[-1,1]$, and accordingly, let us use the $\ell_1$ error $\lone(f):=\bE_{x\sim\mu|_X}[|f(x) - s(x)|]$ as a generalization of the classification error (as in, e.g., \citep{MR1408000}).
The \emph{multiaccuracy error} of $f$ is defined to be
\begin{equation}
\label{eq:mae-total}
\mae_{\mu,B}(f) := {\sup}_{b\in B}|\bE_{(x,y)\sim\mu}[(f(x) - y)b(x)]|,
\end{equation}
where $B$ is a \emph{distinguisher class} consisting of \emph{distinguishers} $b:X\rightarrow[-1,1]$. 
Here we use $B$ (rather than $D$) to denote the distinguisher class because a key idea we use in our work is to relate the distinguisher class to the benchmark class in comparative learning.
Due to our assumption $\bE_{(x,y)\sim\mu}[y|x] = s(x)$, the multiaccuracy error can be written equivalently as
\[
\mae_{\mu,B}(f) = {\sup}_{b\in B}|\bE_{x\sim\mu|_X}[(f(x) - s(x))b(x)]|.
\]

The multiaccuracy error is a generalization and relaxation of the $\ell_1$ error in that if we choose the distinguisher class $B$ to contain all distinguishers $b:X\rightarrow [-1,1]$, then the two errors are equal: $\mae_B(f) = \lone(f)$. 
The multiaccuracy error can become a more suitable performance measure than the $\ell_1$ error if we customize $B$ to reflect the goal we want to achieve: we can choose $B$ to consist of indicator functions of demographic groups to achieve a fairness goal, and we can also choose $B$ to specifically catch serious errors that we want to avoid (see \citet*{hu2022metric} for more discussions).

The sample complexity of achieving a small $\lone(f)$ has been studied by \citet{MR1279411}, \citet{MR1328428} and \citet{MR1408000}, who give a characterization in the distribution-free setting using the \emph{fat-shattering dimension} of the source class $S$, defined as the maximum size of a subset of $X$ \emph{fat-shattered} by $S$ (see \Cref{sec:vc} for a precise definition).
Their results are further improved by \citet{bartlett1995more,MR1629694} and \citet{MR1754872}.
For a general distinguisher class $B$, the sample complexity of achieving a small $\mae_B(f)$ depends on both classes $S$ and $B$, and thus it becomes more challenging to characterize.
In the distribution-specific setting where $\mu|_X$ is fixed and known to the learner,
\citet*{hu2022metric} characterize the sample complexity of achieving $\mae_B(f)\le \varepsilon$ using $\log N_{\mu|_X,B}(S,\Theta(\varepsilon))$: the metric entropy of $S$ w.r.t.\ the \emph{dual Minkowski norm} defined based on $B$ and $\mu|_X$. They also give an equivalent characterization using $\log N_{\mu|_X,S}(B,\Theta(\varepsilon))$ with the roles of $S$ and $B$ swapped.
In the distribution-free setting where $\mu|_X$ is not known to the learner, they only give a sample complexity characterization when $S$ contains all functions $s:X\rightarrow [-1,1]$ using the fat-shattering dimension of $B$, and they leave the case of a general source class $S$ as an open question. In this work, we answer this open question by giving a sample complexity characterization for arbitrary $S$ and $B$ in the distribution-free setting using the \emph{mutual fat-shattering dimension} of $(S,B)$, which we define to be the largest size of a subset of $X$ fat-shattered by both $S$ and $B$.

To prove this sample complexity characterization for distribution-free realizable multiaccuracy, we need a lower and an upper bound on the sample complexity. While we prove the lower bound using relatively standard techniques, the upper bound is much more challenging to prove. We prove the upper bound by reducing multiaccuracy for $(S,B)$ to comparative learning for multiple pairs of \emph{binary} hypothesis classes $(S',B')$ with $\vc(S',B')$ bounded in terms of the mutual fat-shattering dimension of $(S,B)$.
We implement this reduction via an intermediate task which we call \emph{correlation maximization}, and
the main challenge here is that our learner $L$ solving comparative learning for $(S',B')$ is limited by the source class $S'$ and can only handle data points realizable by a binary hypothesis in $S'$. Therefore, we must carefully transform the data points from multiaccuracy to ones acceptable by the comparative learner $L$. We implement this transformation by combining a rejection sampling technique with a \emph{non-uniform covering} type of technique used in a recent work by \citet{hopkins2022realizable}. The difference between the real-valued class $B$ and the binary class $B'$ also poses a challenge, which we solve by taking multiple choices of $B'$ and show that, roughly speaking, the convex hull of the chosen $B'$ approximately includes $B$.

Our characterization using the mutual fat-shattering dimension holds not only for multiaccuracy, but also for a related task called \emph{multicalibration} \citep*{hebert2018multicalibration}. Here, we replace $\mae$ by the multicalibration error:
\begin{align}
\mce_{\mu,B}(f) :={}& \sup_{b\in B}\sum_{v\in V}|\bE_{(x,y)\sim\mu}[(f(x) - y)b(x)\one (f(x) = v)]| \notag \\
= {} & \sup_{b\in B}\sum_{v\in V}|\bE_{x\sim\mu|_X}[(f(x) - s(x))b(x)\one (f(x) = v)]|,\label{eq:mce-total}
\end{align}
where $V$ is the range of $f$ which we require to be countable.
Multicalibration provides a strong guarantee: \citet*{DBLP:conf/innovations/GopalanKRSW22} show that it implies a notion called \emph{omnipredictors}, allowing us to use our multicalibration results to show a sample complexity upper bound for \emph{comparative regression}.

Our results imply that multiaccuracy and multicalibration share the same sample complexity characterization in the distribution-free realizable setting. In comparison, we show that this is \emph{not} the case in the distribution-specific setting where there is a strong sample complexity separation between them (\Cref{remark:mc-non-duality}). This strong separation only appears in our two-hypothesis-class setting: if the source class $S$ contains all hypotheses $s:X\rightarrow [-1,1]$, then realizable multiaccuracy and multicalibration share the same sample complexity characterization (the metric entropy of $B$ in the distribution-specific setting, and the fat-shattering dimension of $B$ in the distribution-free setting).

\subsection{Our Contributions}

Below we summarize the main contributions of our paper. 

\paragraph{Comparative Learning.} We introduce the task of comparative learning (\Cref{def:comp}) by combining realizable learning and agnostic learning. Specifically, we define comparative learning for any pair of hypothesis classes $S$ and $B$ each consisting of partial binary hypotheses $h:X\rightarrow\{-1,1,*\}$ (denoted by $S,B\subseteq\{-1,1,*\}^X$). As in realizable learning, we assume the learner receives data points $(x,y)\in X\times\{-1,1\}$ generated i.i.d.\ from a distribution $\mu$ satisfying $\Pr_{(x,y)\sim\mu}[s(x) = y] = 1$ for a source hypothesis $s\in S$ (in particular, $\Pr_{(x,y)\sim\mu}[s(x) = *] = 0$). As in agnostic learning, we require the learner to output a model $f:X\rightarrow\{-1,1\}$ satisfying 
\begin{equation}
\label{eq:comp-goal-intro}
{\Pr}_{(x,y)\sim\mu}[f(x)\ne y] \le {\inf}_{b\in B}{\Pr}_{(x,y)\sim\mu}[b(x)\ne y] + \varepsilon
\end{equation}
with probability at least $1-\delta$. 

We characterize the sample complexity of comparative learning, denoted by $\scomp(S,B,\varepsilon,\delta)$, using the mutual VC dimension $\vc(S,B)$ which we define 
as the largest size of a subset $X'\subseteq X$ shattered by both $S$ and $B$ (see \Cref{sec:vc} for the formal definition of shattering). In \Cref{thm:comp}, assuming $\varepsilon,\delta\in (0,1/4)$ and $\vc(S,B)\ge 2$, we show a sample complexity upper bound of
\begin{equation}
\label{eq:comp-upper-intro}
\scomp(S,B,\varepsilon,\delta) \le O\left(\frac{\vc(S,B)}{\varepsilon^2}\log^2\left(\frac{\vc(S,B)}{\varepsilon}\right) + \frac 1{\varepsilon^2}\log\left(\frac 1\delta\right)\right),
\end{equation}
and a lower bound of
\begin{equation}
\label{eq:comp-lower-intro}
\scomp(S,B,\varepsilon,\delta) \ge \Omega\left(\frac{\vc(S,B)}{\varepsilon} + \frac 1\varepsilon \log\left(\frac 1\delta\right)\right).
\end{equation}
These bounds imply that the sample complexity of comparative learning is finite if and only if the mutual VC dimension $\vc(S,B)$ is finite. We show a similar sample complexity characterization for a learning task involving an arbitrary number of hypothesis classes in \Cref{sec:agree}.

\paragraph{Correlation Maximization.} 
As an intermediate step towards characterizing the sample complexity of realizable multiaccuracy and multicalbration,
we extend comparative learning to real-valued hypothesis classes by introducing correlation maximization (\Cref{def:cm}). Here, the hypothesis classes $S$ and $B$ can contain any partial real-valued hypotheses $h:X\rightarrow[-1,1]\cup\{*\}$ (denoted by $S,B\subseteq([-1,1]\cup\{*\})^X$), and every data point $(x,y)\in X\times[-1,1]$ can have a label $y$ taking any value in $[-1,1]$. We assume that the data points are drawn i.i.d.\ from a distribution $\mu$ over $X\times[-1,1]$ satisfying $\bE_{(x,y)\sim\mu}[y|x] = s(x)$ for a source hypothesis $s\in S$, and we require the output model $f:X\to \{-1,1\}$ to satisfy 
\[
\bE_{(x,y)\sim\mu}[yf(x)] \ge {\sup}_{b\in B}\bE_{(x,y)\sim\mu}[y\ptimes b(x)] - \varepsilon
\]
with probability at least $1-\delta$. 
Here, we define the \emph{generalized product} $u_1\ptimes u_2$ for $u_1\in\bR$ and $u_2\in [-1,1]\cup\{*\}$ such that $u_1\ptimes u_2 = u_1u_2$ if $u_2\in [-1,1]$, and $u_1\ptimes u_2 = -|u_1|$ if $u_2 = *$. 
The requirement that the output model $f:X\to \{-1,1\}$ produces binary values $f(x)\in \{-1,1\}$ rather than real values $f(x)\in [-1,1]$ is naturally satisfied by our learners for correlation maximization, but it is not essential to any of our results related to correlation maximization. 
In the special case where $S$ and $B$ are both binary, correlation maximization and comparative learning become equivalent 
for values of $\varepsilon$ differing by exactly a factor of $2$, i.e., the goal \eqref{eq:comp-goal-intro} of comparative learning can be equivalently written as
\[
\bE_{(x,y)\sim\mu}[yf(x)] \ge {\sup}_{b\in B}{\Pr}_{(x,y)\sim\mu}[y\ptimes b(x)] - 2\varepsilon.
\]

We give an upper bound on the sample complexity of correlation maximization using the mutual fat-shattering dimension $\fat_\eta(S,B)$ which we define as the largest size of a subset $X'\subseteq X$ that is $\eta$-fat shattered by both $S$ and $B$ (see \Cref{sec:vc} for the definition of fat shattering). In \Cref{thm:cm}, assuming $\varepsilon,\delta\in (0,1/2)$, we show that the sample complexity of correlation maximization is upper bounded by
\[
O\left(\frac{\fat_{\varepsilon/5}(S,B)}{\varepsilon^4}\log^2\left(\frac{\fat_{\varepsilon/5}(S,B)}{\varepsilon}\right)\log\left(\frac 1\varepsilon\right) + \frac 1{\varepsilon^4}\log\left(\frac 1{\varepsilon}\right)\log\left(\frac 1{\delta}\right)\right).
\]
We also consider a deterministic-label setting, which is a special case of correlation maximization where the data distribution $\mu$ satisfies $\Pr_{(x,y)\sim\mu}[y = s(x)] = 1$ for a source class $s\in S$. In this case, we prove the following improved sample complexity upper bound (\Cref{thm:dcm-real}):
\[
O\left(\frac{\fat_{\varepsilon/5}(S,B)}{\varepsilon^2}\log^2\left(\frac{\fat_{\varepsilon/5}(S,B)}{\varepsilon}\right) + \frac 1{\varepsilon^2}\log\left(\frac 1{\varepsilon\delta}\right)\right).
\]
We also show that the mutual fat-shattering dimension does not in general give a lower bound for the sample complexity of correlation maximization. This is because sample complexity duality does not hold for correlation maximization (see \Cref{sec:non-duality-cm}). In \Cref{thm:cm} we state a refined sample complexity upper bound for correlation maximization and we leave it as an open question to determine whether there is a matching lower bound (see \Cref{q:cm}).

\paragraph{Realizable Multiaccuracy and Multicalibration.} We study multiaccuracy and multicalibration in the same setting as in \citep{hu2022metric} with a focus on the distribution-free realizable setting (\Cref{def:ma,def:mc}). As in correlation maximization, the classes $S,B\subseteq([-1,1]\cup\{*\})^X$ can contain any partial real-valued hypotheses $h:X\to [-1,1]\cup\{*\}$, and we assume that the data distribution $\mu$ satisfies $\bE_{(x,y)\sim\mu}[y|x] = s(x)$ for a source hypothesis $s\in S$. The goal is to output a model $f:X\to[-1,1]$ such that $\mae_{\mu,B}(f) \le \varepsilon$ (in multiaccuracy) or $\mce_{\mu,B}(f) \le \varepsilon$ (in multicalibration) with probability at least $1-\delta$, where we generalize the definitions of $\mae_{\mu,B}$ and $\mce_{\mu, B}$ to partial hypothesis classes $B$. In \Cref{thm:ma-mc,thm:ma-mc-lower}, assuming $\varepsilon,\delta\in (0,1/2)$, we show the following lower and upper bounds on the sample complexity of realizable multiaccuracy and multicalibration (denoted by $\sma(S,B,\varepsilon,\delta)$ and $\smc(S,B,\varepsilon,\delta)$, respectively):
\begin{align}
&\Omega\left( \fat_{\sqrt{3\varepsilon}}(S,B) \right) - 1\notag\\
\le {} & \sma(S,B,\varepsilon,\delta)\notag\\
\le {} & \smc(S,B,\varepsilon,\delta)\notag\\
\le {} & O \left( 
	\frac{\fat_{\varepsilon/7}(S,B)}{\varepsilon^6}\log^2\left(
		\frac{\fat_{\varepsilon/7}(S,B)}{\varepsilon}
	\right)
	\log\left(\frac 1\varepsilon\right) + \frac{1}{\varepsilon^6}\log\left(\frac 1\varepsilon\right)\log\left(\frac 1{\varepsilon\delta}\right) \label{eq:ma-mc}
\right).
\end{align}

This implies that the sample complexity of realizable multiaccuracy and multicalibration is finite for every $\varepsilon > 0$ if and only if $\fat_\eta(S,B)$ is finite for every $\eta > 0$. Also, the sample complexity is polynomial in $1/\varepsilon$ if and only if $\fat_\eta(S,B)$ is polynomial in $1/\eta$. This answers an open question in \citep{hu2022metric}. We also show an improved sample complexity upper bound in \Cref{thm:ma-mc-binary} for the special case where $S$ is binary.

Our sample complexity upper and lower bounds stated in \Cref{thm:ma-mc,thm:ma-mc-lower} are actually stronger, and they use a finer definition of the mutual fat-shattering dimension. Specifically, if we define $\fat_{\eta_1,\eta_2}(S,B)$ to be the largest size of a subset $X'\subseteq X$ that is $\eta_1$-fat shattered by $S$ and $\eta_2$-fat shattered by $B$, then $\sma(S,B,\varepsilon,\delta)$ and $\smc(S,B,\varepsilon,\delta)$ are both finite if $\fat_{\eta_1,\eta_2}(S,B)$ is finite for some $\eta_1,\eta_2$ satisfying $2\eta_1 + 4\eta_2 < \varepsilon$, and $\sma(S,B,\varepsilon,\delta)$ and $\smc(S,B,\varepsilon,\delta)$ are both infinite if $\fat_{\eta_1,\eta_2}(S,B)$ is infinite for some $\eta_1,\eta_2$ satisfying $\eta_1\eta_2 > 2\varepsilon$. An open question is whether this gap can be closed to provide an exact characterization of the finiteness of $\sma(S,B,\varepsilon,\delta)$ and $\smc(S,B,\varepsilon,\delta)$ for every choice of $(S,B,\varepsilon,\delta)$.

\paragraph{Covering Number Bound.}
The sample complexity characterization for distribution-specific realizable multiaccuracy in \citep{hu2022metric} is in terms of a covering number defined for every pair of total hypothesis classes $(S,B)$. 
A consequence of our sample complexity characterization for distribution-free realizable multiaccuracy and multicalibration is an upper bound on this covering number in terms of the mutual fat-shattering dimension of $(S,B)$. This can be viewed as a generalization of a classic upper bound on the covering number of a binary hypothesis class $H$ in terms of its VC dimension. Interestingly, our covering number upper bounds in the two-hypothesis-class setting hold despite the fact that a corresponding \emph{uniform convergence bound} does not hold. See \Cref{remark:covering} for more details.

\paragraph{Boosting.} Analogous to the weak agnostic learning task considered by \citet{MR2582918} and \citet{DBLP:conf/innovations/Feldman10}, we introduce \emph{weak} comparative learning (\Cref{def:wcomp}), where the goal \eqref{eq:comp-goal-intro} of comparative learning is relaxed to 
\[
{\Pr}_{(x,y)\sim\mu}[f(x)\ne y] \le 1/2 - \gamma,
\]
under the additional assumption that
\[
{\inf}_{b\in B}{\Pr}_{(x,y)\sim\mu}[b(x)\ne y] \le 1/2 - \alpha.
\]
Here, $\alpha,\gamma\in (0,1/2)$ are parameters of the weak comparative learning task.
Extending results in \citep{DBLP:conf/innovations/Feldman10}, 
we show an efficient boosting algorithm that solves (strong) comparative learning
given oracle access to a learner solving weak comparative learning (\Cref{thm:boosting}). 
This result also applies to correlation maximization for real-valued $S$ and $B$ in the deterministic-label setting.

\paragraph{Comparative Regression.} 
We define \emph{comparative regression} by allowing the classes $S$ and $B$ in comparative learning to be real-valued and replacing the classification error $\Pr_{(x,y)\sim\mu}[f(x)\ne y]$ with the expected loss $\bE_{(x,y)\sim\mu}[\ell(y,f(x))]$ for a general loss function $\ell$. 
Specifically,
we take a partial hypothesis class $S\subseteq([-1,1]\cup\{*\})^X$ as the source class, and for simplicity, we take a \emph{total} hypothesis class $B\subseteq [-1,1]^X$ as the benchmark class.
Given a loss function $\ell:[-1,1]\times[-1,1]\to \bR$, we define the comparative regression task as follows. We assume that the data distribution $\mu$ over $X\times[-1,1]$ satisfies $\bE_{(x,y)\sim\mu}[y|x] = s(x)$ for a source hypothesis $s\in S$, and the goal is to output a model $f:X\to [-1,1]$ such that the following holds with probability at least $1-\delta$:
\[
\bE_{(x,y)\sim\mu}[\ell(y,f(x))] \le {\inf}_{b\in B}\bE_{(x,y)\sim\mu}[\ell(y,b(x))] + \varepsilon.
\]

As an application of our sample complexity characterization for realizable multicalibration and the omnipredictors result by \citet{DBLP:conf/innovations/GopalanKRSW22},
in \Cref{thm:compr}, we give a sample complexity upper bound in terms of $\fat_\eta(S,B)$ for a special case of comparative regression (\Cref{def:compr}) where we assume that the label $y$ in each data point is binary and the loss function $\ell$ is convex and Lipschitz. We leave the study of other interesting settings of comparative regression to future work.

\paragraph{Comparative Online Learning.} We extend our notion of comparative learning to the online setting, where we assume that the data points $(x,y)$ are given sequentially, and the learner is required to predict the label of the individual $x$ in each data point before its true label $y$ is shown. For binary hypothesis classes $S,B\subseteq\{-1,1,*\}^X$, we introduce \emph{comparative online learning} (\Cref{def:compo}) where we assume that every data point $(x,y)$ satisfies $y = s(x)$ for some source hypothesis $s\in S$ and we measure the performance of the learner by its \emph{regret}, defined as the number of mistakes it makes minus the minimum number of mistakes made by a benchmark hypothesis $b\in B$.
The goal of comparative online learning is to ensure that the expected regret does not exceed $\varepsilon n$, where $n$ is the total number of data points given to the learner.
In \Cref{sec:online}, we introduce the \emph{mutual Littlestone dimension} $m:=\ldim(S,B)$ and show that it characterizes the smallest $\varepsilon$ achievable in comparative online learning, denoted by $\varepsilon^*$ (\Cref{thm:online}):
\[
\min\left\{\frac 12,\frac{m}{2n}\right\} \le \varepsilon^* \le O\left(\sqrt{\frac mn\log\frac {2m + n}m}\right).
\]
To match the form of our other sample complexity bounds, we can fix $\varepsilon\in (0,1/2)$ and bound the smallest $n$ (denoted by $n^*$) for which we can ensure that the expected regret does not exceed $\varepsilon n$:
\[
\frac{m}{2\varepsilon} \le n^* \le O\left(\frac{m}{\varepsilon^2}\log\frac 1\varepsilon\right).
\]

\paragraph{Sample Complexity Duality.}
Learning tasks involving two hypothesis classes can potentially satisfy \emph{sample complexity duality}, meaning that the sample complexity of the task changes minimally when we swap the roles of the two hypothesis classes.
\citet{hu2022metric} show that sample complexity duality holds for distribution-specific realizable multiaccuracy, assuming that the source class $S$ and the distinguisher class $B$ are both total. 
Specifically, for $S,B\subseteq[-1,1]^X,\varepsilon,\delta\in (0,1/2)$ and a distribution $\mu_X$ over $X$, defining $m:=\sma\sps{\mu_X}(S,B,\varepsilon/8,\delta)$ to be the sample complexity of realizable multiaccuracy with source class $S$ and distinguisher class $B$ in the distribution-specific setting where the data distribution $\mu$ satisfies $\mu|_X = \mu_X$, \citet{hu2022metric} show that
\[
\sma\sps{\mu_X}(B,S,\varepsilon,\delta) \le O\Big(\varepsilon^{-2}(m + \log(1/\delta))\Big).
\]
Results in our work imply that sample complexity duality also holds for comparative learning and realizable multiaccuracy/multicalibration in the distribution-free setting. Specifically, if we define $m := \scomp(S,B,\allowbreak \varepsilon,\delta) + 1$ for $\varepsilon,\delta\in (0,1/4)$ and any partial binary hypothesis classes $S,B\subseteq\{-1,1,*\}^X$, 
then the following holds by \eqref{eq:comp-upper-intro} and \eqref{eq:comp-lower-intro}:
\[
\scomp(B,S,\varepsilon,\delta) \le O\left(\frac{m\log^2 m}{\varepsilon} + \frac{1}{\varepsilon^2}\log\left(\frac 1 \delta\right)\right).
\]
Similarly, 
if we define $m:= \sma(S,B,\varepsilon^2/147,\delta) + 1$
for $\varepsilon,\delta\in (0,1/2)$ and any partial real-valued hypothesis classes $S,B\subseteq ([-1,1]\cup\{*\})^X$,
then the following holds because of \eqref{eq:ma-mc}:
\[
\sma(B,S,\varepsilon,\delta) \le O \left( 
	\frac{m}{\varepsilon^6}\log^2\left(
		\frac{m}{\varepsilon}
	\right)
	\log\left(\frac 1\varepsilon\right) + \frac{1}{\varepsilon^6}\log\left(\frac 1\varepsilon\right)\log\left(\frac 1{\delta}\right) 
\right),
\]
and the same inequality holds after replacing $\sma$ with $\smc$.
In \Cref{sec:non-duality}, we show that sample complexity duality does \emph{not} hold for other learning tasks we consider in this paper, completing \Cref{table:duality}.

\subsection{Related Work}
Motivated by multi-group/sub-group fairness, many recent papers also study learning tasks involving two (or more) hypothesis classes.
Multi-group agnostic learning, introduced by \citet{DBLP:conf/innovations/BlumL20} and \citet{rothblum2021multi}, involves a subgroup class $G$ and a benchmark class $B$, where each subgroup $g\in G$ is a subset of the individual set $X$. The goal in multi-group agnostic learning is to learn a model such that the loss experienced by each subgroup $g\in G$ is not much larger than the minimum loss for that group achievable by a benchmark $b\in B$.
\citet{tosh2022simple} show sample complexity upper bounds for multi-group agnostic learning in terms of the individual complexities of $G$ and $B$. Thus, the upper bound does not depend on the interaction of the two classes. In contrast, the individual complexities of the source and benchmark (resp.\ distinguisher) classes are not sufficient for our sample complexity characterizations for comparative learning (resp.\ realizable multiaccuracy and multicalibration). 
\citet*{10.1145/3531146.3533172} propose algorithms that can improve the loss on subgroups in the spirit of multi-group agnostic learning based on suggestions from auditors. 
A constrained loss minimization task introduced by
\citet*{kearns2018preventing} also involves a subgroup class and a benchmark class, but the subgroup class is used to impose (fairness) constraints on the learned model and the loss/error is evaluated over the entire population (not on each subgroup). The results in \citet{kearns2018preventing}
assume that the complexities of both classes are bounded, whereas our sample complexity upper bounds (for different tasks) in this paper can be finite even when the complexities of both classes are infinite.
Motivated by the goal of learning \emph{proxies} for sensitive features that can be used to achieve fairness in downstream learning tasks,
\citet*{diana2022proxies} consider a learning task involving \emph{three} hypothesis classes: a source class, a \emph{proxy class}, and a \emph{downstream class}. Again, the sample complexity upper bounds in \citep{diana2022proxies} are in terms of the individual complexities of these classes.
\citet*{shabat2020sample} and \citet*{rosenberg2022exploration} show uniform convergence bounds for multicalibration in a two-hypothesis-class setting, but their bounds are yet again in terms of the individual complexities of the two classes and are finite only when the complexities of both classes are finite.

The notions of multiaccuracy and multicalibration can be viewed in the framework of outcome indistinguishability \citep*{MR4398905,dwork2022beyond}. 
Multicalibrated predictors have been applied to solve loss minimization for rich families of loss functions and/or under a variety of constraints, leading to the notion of omnipredictors \citep*{DBLP:conf/innovations/GopalanKRSW22,hu2022omnipredictors,globus2022multicalibrated}.
Recently, \citet*{gopalan2022loss} show that certain omnipredictors can be obtained from the weaker condition of \emph{calibrated multiaccuracy}.
Multicalibrated predictors can also be used for statistical inference on rich families of target distributions \citep*{kim2022universal}.
The notion of multicalibration has been extended to various settings in \citep*{jung2021moment,zhao2021calibrating,gopalan2022multicalibrated,gopalan2022low}.

Many of our results in this paper are based on sample complexity characterizations of learning partial hypotheses by \citet*{MR4399723}. Some of the key techniques used in \citep{MR4399723} include the $1$-inclusion graph algorithm \citep*{MR1109741}, sample compression schemes \citep*{littlestone1986relating}, sample compression generalization bounds \citep*{graepel2005pac}, and a reduction from agnostic learning to realizable learning \citep*{david2016supervised}.

\subsection{Paper Organization}
The remainder of the paper is organized as follows. In \Cref{sec:preli}, we introduce basic notation and definitions that will be used throughout. In \Cref{sec:comp}, we characterize the sample complexity of comparative learning, and in \Cref{sec:cm} we extend comparative learning to real-valued hypothesis classes and show a sample complexity upper bound for correlation maximization. \Cref{sec:ma-mc} employs the results of \Cref{sec:cm} to derive upper and lower bounds for the sample complexity of realizable multiaccuracy and multicalibration. \Cref{sec:boosting,,sec:compr,,sec:online} study boosting, regression, and online learning, respectively, in the comparative learning setting. Additional discussions of sample complexity duality and extensions to the comparative learning model can be found in the appendix.

\section{Preliminaries}
\label{sec:preli}
Throughout the paper, we use $X$ to denote a non-empty set and we refer to the elements in $X$ as \emph{individuals}. We use the term \emph{hypothesis} to refer to an arbitrary function $h:X\rightarrow \bR\cup \{*\}$ assigning a label $h(x)$ to each individual $x\in X$. 
The label $h(x)$ can be a real number or the \emph{undefined label}~$*$.
We say a hypothesis $h$ is \emph{total} if $h(x)\ne *$ for every $x\in X$. 
When we do not require a hypothesis $h$ to be total, we often say $h$ is \emph{partial} to emphasize that $h$ may or may not be total.
We say a hypothesis $h$ is \emph{binary} if $h(x)\in \{-1,1,*\}$ for every $x\in X$, and we say $h$ is \emph{real-valued} if $h$ may or may not be binary.

A \emph{hypothesis class} $H$ is a set consisting of hypotheses $h:X\rightarrow \bR\cup\{*\}$, i.e., $H\subseteq (\bR\cup\{*\})^X$ where we use $B^A$ to denote the set of all functions $f:A\rightarrow B$ for any two sets $A$ and $B$. 
A \emph{total} hypothesis class is a set $H\subseteq \bR^X$, and
a \emph{binary} hypothesis class is a set $H\subseteq \{-1,1,*\}^X$. 
We say a hypothesis class $H$ is \emph{partial} if it may or may not be total, and we say $H$ is \emph{real-valued} if it may or may not be binary.

To avoid measurability issues, all probability distributions in this paper are assumed to be discrete, i.e., to have a countable support. For any distribution $\mu$ over $X\times \bR$, we use $\mu|_X$ to denote the marginal distribution of $x$ with $(x,y)$ drawn from $\mu$.

\subsection{VC and Fat-shattering Dimensions for Partial Hypothesis Classes}
\label{sec:vc}
The \emph{VC dimension} was introduced by \citet{vapnik1971vcdim} for any total binary hypothesis class. As in \citep{bartlett1995more} and \citep{MR4399723}, we consider a natural generalization of the VC dimension to all partial binary hypothesis classes $H\subseteq\{-1,1,*\}^X$ as follows.
We say a subset $X'\subseteq X$ is \emph{shattered} by $H$ if for every total binary function $\xi:X'\rightarrow\{-1,1\}$ there exists $h\in H$ such that $h(x) = \xi(x)$ for every $x\in X'$. The VC dimension of $H$ is defined to be 
\[
\vc(H) := \sup\{|X'|:X'\subseteq X,\textnormal{$X'$ is shattered by $H$}\}.
\]

An analogous notion of the VC dimension for real-valued hypothesis classes is the \emph{fat-shattering dimension}
introduced by \citet{MR1279411}.
The fat-shattering dimension was originally defined for total hypothesis classes, but it is natural to generalize it to all partial hypothesis classes in a similar fashion to the generalization of the VC dimension to partial binary classes:
given a hypothesis class $H\subseteq (\bR\cup\{*\})^X$ and a margin $\eta \ge 0$, we say a subset $X'\subseteq X$ is $\eta$-fat shattered by $H$ w.r.t.\ a reference function $r:X'\rightarrow \bR$ if for every total binary function $\xi:X'\rightarrow\{-1,1\}$, there exists $h\in H$ such that for every $x\in X'$,
\[
h(x)\ne * \textnormal{ and } \xi(x)(h(x) - r(x)) > \eta.
\]
We sometimes omit the mention of $r$ and say $X'$ is $\eta$-fat shattered by $H$ if such a function $r$ exists.
The $\eta$-fat-shattering dimension of $H$ is defined to be
\[
\fat_\eta (H) := \sup\{|X'|: X'\subseteq X, \textnormal{$X'$ is $\eta$-fat shattered by $H$}\}.
\]

\subsection{An Abstract Learning Task}
\label{sec:abstract}
We study a variety of learning tasks throughout the paper, and to help define each task concisely, we first define an abstract learning task $\learn$, of which each specific task we consider is a special case. 

Let $Z$ and $F$ be two non-empty sets.
In the abstract learning task $\learn$, an algorithm (\emph{learner}) takes data points in $Z$ as input and it outputs a model in $F$. We choose a \emph{distribution class} $\dstr$ consisting of distributions $\mu$ over $Z$, and for each distribution $\mu\in \dstr$, we choose a subset $F_\mu\subseteq F$ to be the \emph{admissible set}.
When the input data points are drawn i.i.d.\ from a distribution $\mu\in \dstr$, we require the learner to output a model $f$ in the admissible set $F_\mu$ with large probability.
Formally, for $n\in \bZ_{\ge 0}$ and $\delta\in \bR_{\ge 0}$, we say a (possibly inefficient and randomized) learner $L$ solves the learning task $\learn_n(Z,F,\dstr,(F_\mu)_{\mu\in \dstr},\delta)$ if
\begin{enumerate}
\item $L$ takes $n$ data points $z_1,\ldots,z_n\in Z$ as input;
\item $L$ outputs a model $f\in F$;
\item For any distribution $\mu\in \dstr$, if the data points $z_1,\ldots,z_n$ are drawn i.i.d.\ from $\mu$, then with probability at least $1-\delta$, the output model $f$ belongs to $F_\mu$. The probability is over the randomness in the data points $z_1,\ldots,z_n$ and the internal randomness in learner $L$.
\end{enumerate}
By a slight abuse of notation, we also use $\learn_n(Z,F,\dstr,(F_\mu)_{\mu\in \dstr},\delta)$ to denote the set of all learners $L$ that solve the learning task.
Clearly, the learner set $\learn_n(Z,F,\dstr,(F_\mu)_{\mu\in \dstr},\delta)$ is monotone w.r.t.\ $n$: for any nonnegative integers $n$ and $n'$ satisfying  $n\le n'$, we have
\[
\learn_n(Z,F,\dstr,(F_\mu)_{\mu\in \dstr},\delta) \subseteq \learn_{n'}(Z,F,\dstr,(F_\mu)_{\mu\in \dstr},\delta)
\]
because when given $n'$ data points, a learner can choose to ignore $n' - n$ data points and only use the remaining $n$ data points.
We define the \emph{sample complexity} $\slearn(Z,F,\dstr,(F_\mu)_{\mu\in \dstr},\delta)$ to be the smallest $n$ for which there exists a learner in $\learn_n(Z,F,\dstr,(F_\mu)_{\mu\in \dstr},\delta)$:
\begin{equation}
\label{eq:samp-learn}
\slearn(Z,F,\dstr,(F_\mu)_{\mu\in \dstr},\delta) := \inf\{n\in \bZ_{\ge 0}:\learn_n(Z,F,\dstr,(F_\mu)_{\mu\in \dstr},\allowbreak \delta)\ne \emptyset\}.
\end{equation}
\subsection{Learning Partial Binary Hypotheses}
We define realizable learning and agnostic learning for any partial binary hypothesis class $H$ as special cases of the abstract learning task $\learn$ in \Cref{sec:abstract}. These learning tasks have been studied by \citet{bartlett1995more,MR2042042,MR4399723}, and the results in these previous works are important for many of our results throughout the paper.

\begin{definition}[Realizable learning ($\rea$)]
\label{def:rea}
Given a partial binary hypothesis class $H\subseteq\{-1,1,*\}^X$, an error bound $\varepsilon \ge 0$, a failure probability bound $\delta \ge 0$, and a nonnegative integer $n$, we define $\rea_n(H,\varepsilon,\delta)$ to be $\learn_n(Z,F,\dstr,(F_\mu)_{\mu\in \dstr},\delta)$ where $Z = X \times\{-1,1\}$, $F = \{-1,1\}^X$, $\dstr$ consists of all distributions $\mu$ over $X\times\{-1,1\}$ satisfying $\Pr_{(x,y)\sim\mu}[h(x) = y] = 1$ for some $h\in H$, and $F_\mu$ consists of all models $f:X\rightarrow\{-1,1\}$ satisfying
\[
{\Pr}_{(x,y)\sim\mu}[f(x)\ne y]\le \varepsilon.
\]
\end{definition}
A key assumption in realizable learning is that any data distribution $\mu\in \dstr$ is consistent with some hypothesis $h\in H$, i.e., $\Pr_{(x,y)\sim\mu}[h(x) = y] = 1$. In particular, this implies that $\Pr_{(x,y)\sim\mu}[h(x) = *] = 0$ because $y\in \{-1,1\}$ cannot be the undefined label $*$. In agnostic learning, we remove such assumptions on the data distribution:

\begin{definition}[Agnostic learning ($\agn$)]
\label{def:agn}
Given a partial binary hypothesis class $H\subseteq\{-1,1,*\}^X$, an error bound $\varepsilon \ge 0$, a failure probability bound $\delta \ge 0$, and a nonnegative integer $n$, we define $\agn_n(H,\varepsilon,\delta)$ to be $\learn_n(Z,F,\dstr,(F_\mu)_{\mu\in \dstr},\delta)$ where $Z = X \times\{-1,1\}$, $F = \{-1,1\}^X$, $\dstr$ consists of all distributions $\mu$ over $X\times\{-1,1\}$, and $F_\mu$ consists of all models $f:X\rightarrow\{-1,1\}$ satisfying
\begin{equation}
\label{eq:agn-goal}
{\Pr}_{(x,y)\sim\mu}[f(x)\ne y]\le {\inf}_{h\in H}\Pr[h(x) \ne y] + \varepsilon.
\end{equation}
\end{definition}
There is no assumption on the data distributions $\mu\in \dstr$ in agnostic learning: $\mu$ can be any distribution over $X\times\{-1,1\}$. The hypothesis class $H$ is used to relax the objective in agnostic learning: instead of requiring the error $\Pr_{(x,y)\sim\mu}[f(x) \ne y]$ of the model $f$ to be at most $\varepsilon$, we compare the error of $f$ with the smallest error of a hypothesis $h\in H$ as in \eqref{eq:agn-goal}. Note that for $(x,y)\in X\times\{-1,1\}$ and $h:X\to\{-1,1,*\}$, we have $h(x)\ne y$ whenever $h(x) = *$.

For every learning task we define throughout the paper, we also implicitly define the corresponding sample complexity as in \eqref{eq:samp-learn}. For example, the sample complexity of realizable learning is
\[
\srea(H,\varepsilon,\delta) := \inf\{n\in \bZ_{\ge 0}:\rea_n(H,\varepsilon,\delta)\ne \emptyset\}.
\]
We omit the sample complexity definitions for all other learning tasks.

\subsection{Other Notation}
For a statement $P$, we define its indicator $\one(P)$ such that $\one(P) = 1$ if $P$ is true, and $\one(P) = 0$ if $P$ is false. We define $\sign:\bR\rightarrow\{-1,1\}$ such that for every $u\in \bR$, $\sign(u) = 1$ if $u \ge 0$, and $\sign(u) = -1$ if $u < 0$. 
For functions $f_1:U_1\to U_2$ and $f_2:U_2 \to U_3$, we use $f_2\circ f_1:U_1\to U_3$ to denote their composition, i.e., $(f_2\circ f_1)(u) = f_2(f_1(u))$ for every $u\in U_1$.
We use $\log(\cdot)$ to denote the base-$2$ logarithm. For $u\in\bR$, we define $\log_+(u) := \log(\max\{2,u\})$.
For $u\in [-1,1]$, we use $\ber^*(u)$ to denote the distribution over $\{-1,1\}$ with mean $u$ (by analogy with the Bernoulli distribution over $\{0,1\}$).

\section{Sample Complexity of Comparative Learning}
\label{sec:comp}
Given a source class $S\subseteq\{-1,1,*\}^X$ and a benchmark class $B\subseteq\{-1,1,*\}^X$, we formally define the task of comparative learning below by combining the distribution assumption in realizable learning and the relaxed objective in agnostic learning:
\begin{definition}[Comparative learning ($\comp$)]
\label{def:comp}
Given two binary hypothesis classes $S,B\subseteq\{-1,1,*\}^X$, an error bound $\varepsilon \ge 0$, a failure probability bound $\delta \ge 0$, and a nonnegative integer $n$, we define $\comp_n(S,B,\varepsilon,\delta)$ to be $\learn_n(Z,F,\dstr,(F_\mu)_{\mu\in \dstr},\delta)$ where $Z = X\times\{-1,1\}$, $F = \{-1,1\}^X$, $\dstr$ consists of all distributions $\mu$ over $X\times \{-1,1\}$ such that $\Pr_{(x,y)\sim\mu}[s(x) = y] = 1$ for some $s\in S$, and $F_\mu$ consists of all models $f:X\rightarrow\{-1,1\}$ such that 
\begin{equation}
\label{eq:comp-goal}
{\Pr}_{(x,y)\sim\mu}[f(x)\ne y] \le {\inf}_{b\in B}{\Pr}_{(x,y)\sim\mu}[b(x)\ne y] + \varepsilon.
\end{equation}
\end{definition}
The data distribution $\mu\in \dstr$ in comparative learning is constrained to be consistent with a \emph{source hypothesis} $s\in S$, i.e., $\Pr_{(x,y)\sim\mu}[s(x) = y] = 1$, and the error of the output model $f$ is compared with the smallest error of a \emph{benchmark hypothesis} $b\in B$ as in \eqref{eq:comp-goal}.

In this section, we characterize the sample complexity of comparative learning for every source class $S\subseteq\{-1,1,*\}^X$ and every benchmark class $B\subseteq \{-1,1,*\}^X$ by proving \Cref{thm:comp} below. Our characterization is based on the \emph{mutual VC dimension} $\vc(S,B)$, which we define as follows:
\begin{equation}
\label{eq:vc}
\vc(S,B) := \{|X'|:X'\subseteq X, \textnormal{$X'$ is shattered by both $S$ and $B$}\}.
\end{equation}
\begin{theorem}
\label{thm:comp}
Let $S,B\subseteq \{-1,1,*\}^X$ be binary hypothesis classes. For any $\varepsilon,\delta \in (0,1/4)$,
the sample complexity of comparative learning satisfies the following upper bound:
\begin{align}
\scomp(S,B,\varepsilon,\delta) & = O\left(\frac{\vc(S,B)}{\varepsilon^2}\log_{+}^2\left(\frac{\vc(S,B)}{\varepsilon}\right) + \frac 1{\varepsilon^2}\log\left(\frac 1\delta\right)\right).\label{eq:comp-upper}
\end{align}
When $\vc(S,B)\ge 2$, we have the following lower bound:%
\begin{align}
\scomp(S,B,\varepsilon,\delta) & = \Omega\left(\frac{\vc(S,B)}{\varepsilon} + \frac 1\varepsilon \log\left(\frac 1\delta\right)\right).\label{eq:comp-lower}
\end{align}
\end{theorem}

Our proof of \Cref{thm:comp} is based on results by \citet{MR4399723} that characterize the sample complexity of realizable and agnostic learning for a partial hypothesis class $H\subseteq\{-1,1,*\}^X$:
\begin{theorem}[\citep{MR4399723}]
Let $H\subseteq \{-1,1,*\}^X$ be a binary hypothesis class. For any $\varepsilon,\delta \in (0,1/4)$,
the sample complexity of agnostic learning satisfies
\begin{align}
\sagn(H,\varepsilon,\delta) & = O\left(\frac{\vc(H)}{\varepsilon^2}\log^2_+\left(\frac{\vc(H)}{\varepsilon}\right) + \frac 1{\varepsilon^2}\log\left(\frac 1\delta\right)\right).\label{eq:agn-upper}
\end{align}
When $\vc(H)\ge 2$, the sample complexity of realizable learning satisfies%
\begin{align}
\srea(H,\varepsilon,\delta) & = \Omega\left(\frac{\vc(H)}{\varepsilon} + \frac 1\varepsilon \log\left(\frac 1\delta\right)\right).\label{eq:rea-lower}
\end{align}
\end{theorem}

We prove the sample complexity upper bound \eqref{eq:comp-upper} by reducing comparative learning for a pair of binary hypothesis classes $(S,B)$ to agnostic learning for a single partial hypothesis class $\bA_{S,B}$ we define below. 

For every pair of hypotheses $s,b:X\to\{-1,1,*\}$, we define an \emph{agreement hypothesis} $\ba_{s,b}:X\rightarrow \{-1,1,*\}$ by
\[
\ba_{s,b}(x) = \begin{cases}
0,& \textnormal{if } s(x) = b(x) = 0;\\
1,& \textnormal{if } s(x) = b(x) = 1;\\
*, & \textnormal{otherwise.}
\end{cases}
\]
For every pair of hypothesis classes $S,B\subseteq\{-1,1,*\}^X$, we define the \emph{agreement hypothesis class} $\bA_{S,B}$ to be $\{\ba_{s,b}: s\in S,b\in B\}\subseteq\{-1,1,*\}^X$. 

The following claim follows immediately from the definition of $\ba_{s,b}$:
\begin{claim}
\label{claim:agree}
For every $(x,y)\in X\times\{-1,1\}$ and every pair of hypotheses $s,b:X\to \{-1,1,*\}$, we have $\ba_{s,b}(x) = y$ if and only if $s(x) = b(x) = y$.
\end{claim}

The following claim shows that the mutual VC dimension of $(S,B)$ is equal to the VC dimension of $\bA_{S,B}$:
\begin{claim}
\label{claim:comp-vc}
Let $S,B\subseteq \{-1,1,*\}^X$ be binary hypothesis classes.
Then $\vc(\bA_{S,B}) = \vc(S,B)$.
\end{claim}
\begin{proof}
A subset $X'\subseteq X$ is shattered by both $S$ and $B$ if and only if for every $\xi:X'\to\{-1,1\}$, there exists $s\in S$ and $b\in B$ such that
\begin{equation}
\label{eq:vc-mutual-1}
s(x) = b(x) = \xi(x) \text{ for every } x\in X'.
\end{equation}
Similarly, by the definition of $\bA_{S,B}$, a subset $X'\subseteq X$ is shattered by $\bA_{S,B}$ if and only if for every $\xi:X'\to\{-1,1\}$, there exists $s\in S$ and $b\in B$ such that
\begin{equation}
\label{eq:vc-mutual-2}
\ba_{s,b}(x) = \xi(x) \text{ for every } x\in X'.
\end{equation}
By \Cref{claim:agree}, the conditions \eqref{eq:vc-mutual-1} and \eqref{eq:vc-mutual-2} are equivalent.
\end{proof}

We are now ready to state and prove the reduction that allows us to prove \eqref{eq:comp-upper}:
\begin{lemma}
\label{lm:comp-upper-reduction}
Let $S,B\subseteq \{-1,1,*\}^X$ be binary hypothesis classes. For any $\varepsilon,\delta\in \bR_{\ge 0}$ and $n\in\bZ_{\ge 0}$, we have
$\agn_n(\bA_{S,B},\varepsilon,\delta) \subseteq \comp_n(S,B,\varepsilon,\delta)$. In other words, any learner solving agnostic learning for $\bA_{S,B}$ also solves comparative learning for $(S,B)$ with the same parameters $\varepsilon$ and $\delta$.
\end{lemma}
\begin{proof}
Let $L$ be a learner in $\agn_n(\bA_{S,B},\varepsilon,\delta)$. 
For $s\in S$, let $\mu$ be a distribution over $X\times\{-1,1\}$ satisfying $\Pr_{(x,y)\sim\mu}[s(x) = y] = 1$.
By the guarantee of $L \in \agn_n(\bA_{S,B},\varepsilon,\delta)$, given $n$ data points $(x_1,y_1),\ldots,(x_n,y_n)$ drawn i.i.d.\ from $\mu$, with probability at least $1-\delta$, $L$ outputs a model $f$ satisfying
\begin{align*}
{\Pr}_{(x,y)\sim\mu}[f(x)\ne y] & \le \inf_{h\in \bA_{S,B}}{\Pr}_{(x,y)\sim\mu}[h(x)\ne y] + \varepsilon\\
& = \inf_{s'\in S,b\in B}{\Pr}_{(x,y)\sim\mu}[\ba_{s',b}(x)\ne y] + \varepsilon\tag{by definition of $\bA_{S,B}$}\\
& \le \inf_{b\in B}{\Pr}_{(x,y)\sim\mu}[\ba_{s,b}(x)\ne y] + \varepsilon \\
& = \inf_{b\in B}{\Pr}_{(x,y)\sim\mu}[s(x)\ne y \textnormal{ or } b(x)\ne y] + \varepsilon \tag{by \Cref{claim:agree}}\\
& = \inf_{b\in B}{\Pr}_{(x,y)\sim\mu}[b(x)\ne y] + \varepsilon \tag{by $\Pr_{(x,y)\sim\mu}[s(x) = y] = 1$}.
\end{align*}
This proves that $L\in \comp_n(S,B,\varepsilon,\delta)$, as desired.
\end{proof}
Our upper bound \eqref{eq:comp-upper} follows immediately from \Cref{claim:comp-vc}, \Cref{lm:comp-upper-reduction}, and \eqref{eq:agn-upper}. We defer the detailed proof to the end of the section.
To prove the lower bound \eqref{eq:comp-lower}, we reduce realizable learning for $\bA_{S,B}$ to comparative learning for $(S,B)$:
\begin{lemma}
\label{lm:comp-lower-reduction}
Let $S,B\subseteq \{-1,1,*\}^X$ be binary hypothesis classes.
For any $\varepsilon,\delta\in \bR_{\ge 0}$ and $n\in\bZ_{\ge 0}$, we have
$\comp_n(S,B,\varepsilon,\delta)\subseteq \rea_n(\bA_{S,B},\varepsilon,\delta)$.
In other words, any learner solving comparative learning for $(S,B)$ also solves realizable learning for $\bA_{S,B}$ with the same parameters $\varepsilon$ and $\delta$.
\end{lemma}
\begin{proof}
Let $L$ be a learner in $\comp_n(S,B,\varepsilon,\delta)$. Let $\mu$ be a distribution over $X\times\{-1,1\}$ satisfying
\begin{equation}
\label{eq:comp-lower-rea-assumption}
{\Pr}_{(x,y)\sim\mu}[h(x) = y] = 1 \text{ for some }h\in \bA_{S,B}.
\end{equation}

By the definition of $\bA_{S,B}$, our assumption \eqref{eq:comp-lower-rea-assumption} implies that $
{\Pr}_{(x,y)\sim\mu}[\ba_{s,b}(x) = y] = 1
$ for some $s\in S$ and $b\in B$.
By \Cref{claim:agree}, we have ${\Pr}_{(x,y)\sim\mu}[s(x) = y] = {\Pr}_{(x,y)\sim\mu}[b(x) = y] = 1$.

By the guarantee of $L \in \comp_n(S,B,\varepsilon,\delta)$, given $n$ data points $(x_1,y_1),\ldots,(x_n,y_n)$ drawn i.i.d.\ from $\mu$, with probability at least $1-\delta$, $L$ outputs a model $f$ satisfying
\[
{\Pr}_{(x,y)\sim\mu}[f(x)\ne y] \le \inf_{b'\in B}{\Pr}_{(x,y)\sim\mu}[b'(x)\ne y] + \varepsilon = \varepsilon,
\]
where the last equation holds because $\Pr_{(x,y)\sim\mu}[b(x) = y] = 1$. The inequality above implies $L\in \rea_n(\bA_{S,B},\varepsilon,\delta)$, as desired.
\end{proof}
\begin{proof}[Proof of \Cref{thm:comp}]
Define $m := \vc(S,B)$. By \Cref{claim:comp-vc}, we have $m = \vc(\bA_{S,B})$.
Our upper bound \eqref{eq:comp-upper} holds because
\begin{align*}
\scomp(S,B,\varepsilon,\delta) & \le \sagn(\bA_{S,B},\varepsilon,\delta) \tag{by \Cref{lm:comp-upper-reduction}}\\
& \le O\left(\frac{m}{\varepsilon^2}\log^2_+\left(\frac{m}{\varepsilon}\right) + \frac 1{\varepsilon^2}\log\left(\frac 1\delta\right)\right). \tag{by \eqref{eq:agn-upper}}
\end{align*}
Our lower bound \eqref{eq:comp-lower} holds because
\begin{align*}
\scomp(S,B,\varepsilon,\delta) & \ge \srea(\bA_{S,B},\varepsilon,\delta) \tag{by \Cref{lm:comp-lower-reduction}}\\
& \ge \Omega\left(\frac{m}{\varepsilon} + \frac 1\varepsilon \log\left(\frac 1\delta\right)\right).\tag{by \eqref{eq:rea-lower}}
\end{align*}
\end{proof}
\section{Sample Complexity of Correlation Maximization}
\label{sec:cm}
As we define in \Cref{sec:comp},
the comparative learning task $\comp$ requires the hypothesis classes $S$ and $B$ to be binary.
Here we introduce a natural generalization of $\comp$ to real-valued hypothesis classes $S,B\subseteq([-1,1]\cup \{*\})^X$ which we call \emph{correlation maximization}.

We first generalize the product $u_1u_2$ of two real numbers $u_1,u_2\in\bR$ to the case where $u_2$ may be the undefined label $*$. Specifically, for $u_1\in \bR$ and $u_2\in [-1,1]\cup\{*\}$, we define their \emph{generalized product} $u_1\ptimes u_2$ to be
\[
u_1\ptimes u_2 := 
\begin{cases}
u_1u_2,& \textnormal{if }u_2\in[-1,1],\\
-|u_1|, & \textnormal{if }u_2 = *.
\end{cases}
\]
The idea behind the definition is that when $u_2 = *$, we treat $u_2$ as being an unknown number $u'$ in $[-1,1]$ and define the product $u_1\ptimes u_2$ to be the smallest possible value of $u_1u'$, i.e., $u_1\ptimes u_2 = \inf_{u'\in[-1,1]}u_1u' = - |u_1|$.

This generalized product allows us to rewrite the goal \eqref{eq:comp-goal} of comparative learning. For any $y\in \{-1,1\}$ and $u\in \{-1,1,*\}$, it is easy to verify that 
\begin{equation}
\label{eq:error-correlation}
\one(y\ne u) = \frac 12(1 - y\ptimes u).
\end{equation}
Therefore, the goal \eqref{eq:comp-goal} of comparative learning can be equivalently written as
\[
\bE_{(x,y)\sim\mu}[yf(x)] \ge {\sup}_{b\in B}\bE_{(x,y)\sim\mu}[y\ptimes b(x)] - 2\varepsilon.
\]

This reformulation is meaningful even when we relax $B$ to be a \emph{real-valued} hypothesis class $B\subseteq([-1,1]\cup\{*\})^X$. If we also relax the source class $S$, we obtain the definition of correlation maximization: 
\begin{definition}[Correlation maximization ($\cm$)]
\label{def:cm}
Given two real-valued hypothesis classes $S,B\subseteq([-1,1]\cup\{*\})^X$, an error bound $\varepsilon \ge 0$, a failure probability bound $\delta \ge 0$, and a nonnegative integer $n$, we define $\cm_n(S,B,\varepsilon,\delta)$ to be $\learn_n(Z,F,\dstr,(F_\mu)_{\mu\in \dstr},\delta)$ with $Z,F,\dstr,F_\mu$ chosen as follows. We choose $Z = X\times[-1,1]$ and $F = \{-1,1\}^X$. The distribution class $\dstr$ consists of all distributions $\mu$ over $X\times [-1,1]$ satisfying the following property:
\begin{equation}
\label{eq:cm-assumption}
\text{there exists $s\in S$ such that ${\Pr}_{x\sim\mu|_X}[s(x)\ne *] = 1$ and $\bE_{(x,y)\sim\mu}[y|x] = s(x)$}.
\end{equation}
The admissible set $F_\mu$ consists of all models $f:X\rightarrow \{-1,1\}$ satisfying
\[
\bE_{(x,y)\sim\mu}[yf(x)] \ge {\sup}_{b\in B}\bE_{(x,y)\sim\mu}[y\ptimes b(x)] - \varepsilon.
\]
\end{definition}
The name ``correlation maximization'' comes from viewing $\bE_{(x,y)\sim\mu}[yf(x)]$ as the (uncentered) correlation between random variables $y$ and $f(x)$. 
In correlation maximization, any data distribution $\mu\in \dstr$ needs to satisfy $\bE_{(x,y)\sim\mu}[y|x] = s(x)$ for a source hypothesis $s\in S$. This restricts the conditional \emph{expectation} of $y$ given $x$, but we allow the conditional \emph{distribution} of $y$ given $x$ to be otherwise unrestricted. That is, when conditioned on $x\in X$ being fixed, the label $y\in [-1,1]$ could be deterministically equal to $s(x)$, but $y$ could be also be random as long as it has conditional expectation $s(x)$. 
This makes the task challenging, and
when we design learners for correlation maximization, we find it helpful to first consider the simpler task where $y = s(x)$ holds deterministically given $x\in X$:
\begin{definition}[Deterministic-label Correlation Maximization ($\dcm$)]
\label{def:dcm}
We define $\dcm_n(S,B,\allowbreak \varepsilon,\delta)$ in the same way as we define $\cm_n(S,B,\varepsilon,\delta)$ in \Cref{def:cm}, except that we replace \eqref{eq:cm-assumption} with the stronger assumption
\begin{equation}
\label{eq:deterministic-label}
\text{there exists $s\in S$ such that ${\Pr}_{(x,y)\sim\mu}[s(x) = y] = 1$}.
\end{equation}
\end{definition}

In this section, we show a sample complexity upper bound for correlation maximization for any source class $S\subseteq([-1,1]\cup\{*\})^X$ and benchmark class $B\subseteq([-1,1]\cup\{*\})^X$. 
Since $S$ and $B$ may no longer be binary, we cannot apply the mutual VC dimension as a way to characterize their complexity. Instead, we turn to a classic generalization of the VC dimension for real-valued hypotheses, the fat-shattering dimension (see \Cref{sec:vc} for definition). Our upper bound is in terms of the \emph{mutual fat-shattering dimension} defined as follows, which generalizes the mutual VC dimension to real-valued hypothesis classes.

Given a pair of real-valued hypothesis classes $S,B\subseteq([-1,1]\cup\{*\})^X$ and a margin $\eta\in \bR_{\ge 0}$, we define the \emph{mutual fat-shattering dimension} $\fat_\eta(S,B)$ as follows:
\begin{equation}
\label{eq:fat}
\fat_\eta(S,B) := \sup\{|X'|:X'\subseteq X,\textnormal{$X'$ is $\eta$-fat shattered by both $S$ and $B$}\}.
\end{equation}
In other words, $\fat_\eta(S,B)$ is the largest size of a subset $X'\subseteq X$ such that $X'$ is $\eta$-fat shattered by $S$ w.r.t.\ a function $r_1:X'\rightarrow\bR$ and $X'$ is $\eta$-fat shattered by $B$ w.r.t.\ a function $r_2:X'\rightarrow\bR$ (recall the definition of fat shattering in \Cref{sec:vc}).

Another equivalent way to define the mutual fat-shattering dimension for real-valued hypothesis classes $S,B\subseteq([-1,1]\cup\{*\})^X$ is by transforming them into binary classes and using the mutual VC dimension after the transformation. 
These transformations are also crucial in our proof of the sample complexity upper bound for correlation maximization in this section.
Given a real-valued hypothesis $h:X\rightarrow[-1,1]\cup\{*\}$, a reference function $r:X\rightarrow \bR$, and a margin $\eta\in\bR_{\ge 0}$, we define a binary hypothesis $h_\eta\sps r:X\rightarrow \{-1,1,*\}$ such that
\[
h_\eta\sps r(x) = 
\begin{cases}
1, & \textnormal{if }h(x)\ne * \textnormal{ and } h(x) > r(x) + \eta;\\
-1, & \textnormal{if }h(x)\ne * \textnormal{ and } h(x) < r(x) - \eta;\\
*, & \textnormal{otherwise}.
\end{cases}
\]
Given a real-valued hypothesis class $H\subseteq ([-1,1]\cup\{*\})^X$, we define the binary hypothesis class $H_\eta\sps r\subseteq \{-1,1,*\}^X$ as
\[
H_\eta\sps r = \{h_\eta\sps r:h\in H\}.
\]
We can now transform any real-valued hypothesis class $H\subseteq ([-1,1]\cup\{*\})^X$ into a binary hypothesis class $H_\eta\sps r$ for every choice of $\eta\in\bR_{\ge 0}$ and $r:X\rightarrow \bR$.
This allows us to measure the complexity of a pair of real-valued hypothesis classes $S,B\subseteq([-1,1]\cup\{*\})^X$ using the mutual VC dimensions $\vc(S_{\eta_1} \sps {r_1}, B_{\eta_2}\sps{r_2})$ of the binary hypothesis classes $S_{\eta_1} \sps {r_1}, B_{\eta_2}\sps{r_2}$ for various choices of $\eta_1,\eta_2,r_1,r_2$. 
The following claim shows that the mutual fat-shattering dimension $\fat_\eta(S,B)$ can be defined equivalently in this way.
\begin{claim}
\label{claim:fat-alt}
Let $S,B\subseteq([-1,1]\cup\{*\})^X$ be real-valued hypothesis classes.
For every $\eta\in \bR_{\ge 0}$,
$\fat_\eta(S,B) = \sup_{r_1,r_2}\vc(S_\eta\sps{r_1},B_\eta\sps{r_2})$,
where the supremum is over all function pairs $r_1,r_2:X\rightarrow\bR$.
\end{claim}
The claim follows from the fact that a subset $X'\subseteq X$ is $\eta$-fat shattered by $S$ if and only if $X'$ is shattered by the binary hypothesis class $S_\eta\sps {r}$ for some $r:X\to \bR$, and the same holds with $S$ replaced by $B$.

Before we state our sample complexity upper bound for correlation maximization in \Cref{thm:cm},
we make some additional definitions to simplify the statement.
Let $h:X\rightarrow [-1,1]\cup\{*\}$ be a real-valued hypothesis and $H\subseteq ([-1,1]\cup \{*\})^X$ be a real-valued hypothesis class.
For every real number $\theta\in \bR$, we use $h_{\eta}\sps {\theta}$ and $H_\eta\sps \theta$ to denote $h_{\eta}\sps {r}$ and $H_\eta\sps r$ with the reference function $r:X\rightarrow \bR$ being the constant function satisfying $r(x) = \theta$ for every $x\in X$. 
\begin{theorem}
\label{thm:cm}
Let $S,B\subseteq([-1,1]\cup\{*\})^X$ be real-valued hypothesis classes. 
For $\eta_1,\eta_2,\beta,\delta\in (0,1/2)$, defining $m:= \sup_{\theta\in \bR}\vc(S_{\eta_1}\sps 0, B_{\eta_2}\sps \theta)$, we have
\begin{align*}
& \scm(S,B,\beta + 2\eta_1 + 2\eta_2,\delta)\\
\le {} & O\left(
\frac{m}{\beta^4}\log^2_+\left(\frac{m}{\beta}\right)\log\left(\frac 1{\eta_1}\right) + \frac 1{\beta^4}\log\left(\frac 1{\eta_1}\right)\log\left(\frac 1\delta\right) + \frac{1}{\beta^2}\log\left(\frac 1{\eta_2}\right) \right).
\end{align*}
Moreover, for every $\varepsilon\in (0,1/2)$, choosing $\beta = \eta_1 = \eta_2 = \varepsilon/5$, we have $m \le \fat_{\varepsilon/5}(S,B)$ and
\[
\scm(S,B,\varepsilon,\delta) \le O\left(\frac{m}{\varepsilon^4}\log^2_+\left(\frac{m}{\varepsilon}\right)\log\left(\frac 1\varepsilon\right) + \frac 1{\varepsilon^4}\log\left(\frac 1{\varepsilon}\right)\log\left(\frac 1{\delta}\right)\right).
\]
\end{theorem}
It remains an open question whether there is a sample complexity lower bound that matches \Cref{thm:cm}, although our sample complexity characterization for multiaccuracy and multicalibration in \Cref{sec:ma-mc} does not rely on such a lower bound.
Qualitatively, \Cref{thm:cm}  implies that $\scm(S,B,\beta + 2\eta_1 + 2\eta_2,\delta)$ is finite if $\sup_{\theta\in\bR}\vc(S_{\eta_1}\sps 0,B_{\eta_2}\sps \theta)$ is finite. We thus propose the following question about a qualitative lower bound:
\begin{question}
\label{q:cm}
Let $S,B\subseteq([-1,1]\cup\{*\})^X$ be real-valued hypothesis classes. 
Suppose $\vc(S_{\eta_1}\sps 0,B_{\eta_2}\sps \theta)$ is infinite for some $\eta_1,\eta_2>0$ and $\theta\in \bR$. Does this imply that $\scm(S,B,\varepsilon,\delta)$ is infinite for some $\varepsilon,\delta > 0$?
\end{question}

We devote the remaining of the section to proving \Cref{thm:cm}. The main idea is to reduce correlation maximization for $(S,B)$ to comparative learning for $(S_{\eta_1}\sps 0, B_{\eta_2}\sps \theta)$ for suitable choices of $\theta\in \bR$. 
In order to transform the data points drawn according to a real-valued source hypothesis $s\in S$ into data points generated from a binary source hypothesis in $S_{\eta_1}\sps 0$, we start with the simpler deterministic-label setting (\Cref{def:dcm}) and apply a rejection sampling procedure. We then reduce the general setting to the deterministic-label setting using a ``non-uniform covering'' type technique inspired by \citet{hopkins2022realizable}.
Once we transform the data points, we apply our learner for comparative learning to get a model achieving a small error (or equivalently, a large correlation) compared to the binary benchmark hypotheses in $B_{\eta_2}\sps \theta$. To translate this into a comparison with the real-valued benchmark hypotheses in $B$, we show that, roughly speaking, $B$ is approximately included in the convex hull of $\bigcup_\theta B_{\eta_2}\sps \theta$ for a sufficiently rich collection of $\theta$'s. We implement these ideas in full detail in the following three subsections, starting with the simpler special case using deterministic labels and binary benchmarks and building to the general case.

\subsection{Deterministic Labels and Binary Benchmarks}
\label{sec:dcm}
We start with the special case of deterministic-label correlation maximization (\Cref{def:dcm}) with the assumption that the benchmark class $B$ is binary.

\begin{theorem}
\label{thm:dcm}
Let $S\subseteq([-1,1]\cup\{*\})^X$ be a real-valued hypothesis class and $B\subseteq\{-1,1,*\}^X$ be a binary hypothesis class. 
For $\eta \ge 0,\varepsilon,\delta\in (0,1/2)$, defining $m:= \vc(S_\eta\sps 0,B)$, we have
\[\sdcm(S,B,\varepsilon + 2\eta,\delta) \le O\left(\frac{m}{\varepsilon^2}\log^2_+\left(\frac{m}{\varepsilon}\right) + \frac 1{\varepsilon^2}\log\left(\frac 1{\delta}\right)\right).
\]
\end{theorem}

To prove \Cref{thm:dcm}, we design a learner (\Cref{alg:dcm}) and show that it solves $\dcm$ with the desired sample complexity in the lemma below:
\begin{lemma}
\label{lm:dcm}
Let $S\subseteq([-1,1]\cup\{*\})^X$ be a real-valued hypothesis class and $B\subseteq\{-1,1,*\}^X$ be a binary hypothesis class. 
Suppose the parameters of \Cref{alg:dcm} satisfy $\eta\ge 0,\varepsilon,\delta\in (0,1/2),n \ge \sup_{\varepsilon'\ge \varepsilon}\frac{\varepsilon'}{\varepsilon}\scomp(S_{\eta}\sps{0},B, \varepsilon'/4,\delta/2) $ and $n\ge C\varepsilon^{-1}\log(1/\delta)$ for a sufficiently large absolute constant $C > 0$.
Then,
\Cref{alg:dcm} belongs to $\dcm_n(S,B,\varepsilon + 2\eta,\delta)$. 
\end{lemma}
\LinesNumbered
\begin{algorithm}
\caption{Deterministic-label correlation maximization for $(S,B)$ with binary $B$.}
\label{alg:dcm}
\SetKwInOut{Parameters}{Parameters}
\SetKwInOut{Input}{Input}
\SetKwInOut{Output}{Output}
\Parameters{$S\subseteq([-1,1]\cup\{*\})^X,B\subseteq\{-1,1,*\}^X,n\in \bZ_{\ge 0},\varepsilon,\eta,\delta \in \bR_{\ge 0}$.}
\Input{data points $(x_1,y_1),\ldots,(x_n,y_n)\in X\times[-1,1]$.}
\Output{model $f:X\rightarrow\{-1,1\}$.}
Initialize $\Psi$ to be an empty dataset\label{line:dcm-init}\;
\For{$i = 1,\ldots,n$\label{line:dcm-for-1}}{
\If{$|y_i| > \eta$}
{
With probability $|y_i|$, add the data point $(x_i,\sign(y_i))$ into $\Psi$\;
}
}\label{line:dcm-for-2}
Let $n'$ be the number of data points in $\Psi$\label{line:dcm-gamma}\;
Invoke comparative learner $L\in \comp_{n'}(S_{\eta}\sps{0},B,\varepsilon n/(4n'),\delta/2)$ on $\Psi$ to obtain $f:X\rightarrow\{-1,1\}$\label{line:dcm-invoke}\tcc*{If $n' = 0$, choose an arbitrary $f:X\rightarrow\{-1,1\}$.}
\Return $f$\;
\end{algorithm}

\Cref{thm:dcm} follows from \Cref{lm:dcm} and \Cref{thm:comp}:
\begin{proof}[Proof of \Cref{thm:dcm}]
Since $\scomp(S_\eta\sps 0, B,\varepsilon'',\delta) = 0$ whenever $\varepsilon'' > 1$,
by \Cref{thm:comp}, for every $\varepsilon'\ge \varepsilon$, 
\[
\scomp(S_\eta\sps 0, B, \varepsilon'/4,\delta/2) = O\left(\frac{m}{(\varepsilon')^2}\log^2_+\left(\frac{m}{\varepsilon}\right) + \frac 1{(\varepsilon')^2}\log\left(\frac 1\delta\right)\right),
\]
which implies that 
\[
\sup_{\varepsilon'\ge \varepsilon}\frac{\varepsilon'}{\varepsilon}\scomp(S_{\eta}\sps{0},B, \varepsilon'/4,\delta/2) = O\left(\frac{m}{\varepsilon^2}\log^2_+\left(\frac{m}{\varepsilon}\right) + \frac 1{\varepsilon^2}\log\left(\frac 1\delta\right)\right).
\]
By \Cref{lm:dcm}, there exists $n = O\left(\frac{m}{\varepsilon^2}\log^2_+\left(\frac{m}{\varepsilon}\right) + \frac 1{\varepsilon^2}\log\left(\frac 1\delta\right)\right)$ such that $\dcm_n(S,B,\varepsilon + 2\eta,\delta)\ne \emptyset$, as desired.
\end{proof}

Before we prove \Cref{lm:dcm}, we first discuss the idea behind \Cref{alg:dcm}. The key idea is to reduce the correlation maximization task ($\dcm$) for $(S,B)$ to the comparative learning task ($\comp$) for $(S_{\eta}\sps 0,B)$.
In particular, we need to transform the input data points for $\dcm$ to valid input data points for $\comp$. 
Each data point $(x,y)$ in $\dcm$ is generated i.i.d.\ from a distribution $\mu$ satisfying $\Pr_{(x,y)\sim\mu}[s(x) = y] = 1$ for some real-valued hypothesis $s\in S$, and the label $y = s(x)$ may take any value in $[-1,1]$. However, the label $y$ needs to be binary in any data point $(x,y)$ for $\comp$.
Thus for every data point $(x,y)$ in $\dcm$, we want to replace it by $(x,\sign(y))$. If we directly use the data points $(x,\sign(y))$ as the input for $\comp$, we would get a model $f:X\rightarrow\{-1,1\}$ such that $\Pr_{(x,y)\sim\mu}[f(x)\ne\sign(y)]$ is small, or equivalently (by \eqref{eq:error-correlation}), 
\begin{equation}
\label{eq:dcm-idea-1}
\bE_{(x,y)\sim\mu}[\sign(y)f(x)] = \bE_{x\sim\mu|_X}[\sign(s(x))f(x)]
\end{equation}
is large (note that $\Pr_{(x,y)\sim\mu}[s(x) = y] = 1$). However, our goal in $\dcm$ is to maximize 
\begin{equation}
\label{eq:dcm-idea-2}
\bE_{(x,y)\sim\mu}[yf(x)] = \bE_{x\sim\mu|_X}[s(x)f(x)].
\end{equation}
To relate \eqref{eq:dcm-idea-1} and \eqref{eq:dcm-idea-2},
we note that $|s(x)|\sign(s(x)) = s(x)$, and thus
\[
\bE_{x\sim\mu|_X}[|s(x)|\sign(s(x))f(x)] = \bE_{x\sim\mu|_X}[s(x)f(x)].
\]
Therefore, if we construct a new distribution $\mu'$ so that $\frac{\mu'|_X(x)}{\mu|_X(x)} \propto |s(x)|$, where $\mu|_X(x),\mu'|_X(x)$ are the probability mass on $x\in X$ from $\mu|_X$ and $\mu'|_X$, respectively, then by the equation above, we have
\[
\bE_{x\sim\mu'|_X}[\sign(s(x))f(x)] = C\,\bE_{x\sim\mu|_X}[s(x)f(x)]
\]
for a constant $C$ independent of $f$. This means that if we replace the distribution $\mu|_X$ in \eqref{eq:dcm-idea-1} with the new distribution $\mu'|_X$, we get the desired \eqref{eq:dcm-idea-2} up to scaling. %

\Cref{alg:dcm} uses a natural rejection sampling procedure to generate the new data points $(x,\sign(s(x)))$ with $x$ drawn from $\mu'|_X$: for every data point $(x,y)$ drawn from $\mu$, we remove the data point with probability $1 - |y| = 1 - |s(x)|$, and replace the data point by $(x,\sign(y))$ with the remaining probability $|y| = |s(x)|$.
More precisely, \Cref{alg:dcm} makes a slight adjustment: when $|y|\le \eta$, we remove the example with probability $1$ instead of $1 - |y|$. 
We show that the change to $\mu'|_X$ caused by this adjustment is insignificant for our purposes, and the adjustment ensures that each remaining data point $(x,\sign(y))$ satisfies $\sign(y) = s_{\eta}\sps {0}(x)$, i.e., $(x,\sign(y))$ is a valid data point in the comparative learning task for source hypothesis $s_{\eta}\sps {0}\in S_{\eta}\sps {0}$.

\LinesNumbered

We are now ready to prove \Cref{lm:dcm}. We first analyze the distribution of the new data points in $\Psi$ generated from the rejection sampling procedure at Lines~\ref{line:dcm-init}-\ref{line:dcm-for-2} in \Cref{alg:dcm}.
Let $\mu$ be a distribution over $X\times[-1,1]$ such that $\Pr_{(x,y)\sim\mu}[s(x) = y] = 1$ for some $s\in S$. 
For a chosen parameter $\eta\in\bR_{\ge 0}$ in \Cref{alg:dcm},
let $\nu$ be the joint distribution of $(x,y,u)\in X\times[-1,1]\times\{0,1\}$ where $(x,y)\sim\mu$ and
\begin{equation}
\label{eq:dcm-cond-u}
\Pr[u = 1|x,y] = \begin{cases}
0, & \textnormal{if }|y| \le \eta;\\ 
|y|, & \textnormal{otherwise.}
\end{cases}
\end{equation}
Let $\mu'$ denote the conditional distribution of $(x,\sign(y))\in X\times\{-1,1\}$ given $u = 1$ for $(x,y,u)\sim\nu$. The following claim follows directly from the description of \Cref{alg:dcm}:
\begin{claim}
\label{claim:dcm-new-points}
Let the distributions $\mu,\nu,\mu'$ be defined as above.
Assume that the input data points $(x_i,y_i)_{i=1}^n$ to \Cref{alg:dcm} are generated i.i.d.\ from $\mu$.
For every $i = 1,\ldots, n$, let $u_i$ denote the indicator for the event that a new data point is added to $\Psi$ in the $i$-th iteration of the \for loop at Lines~\ref{line:dcm-for-1}-\ref{line:dcm-for-2}. 
Then $(x_i, y_i, u_i)_{i=1}^n$ are distributed independently according to $\nu$.
Also, when conditioned on $n'$, the data points in $\Psi$ at \Cref{line:dcm-gamma} are distributed independently according to $\mu'$.
\end{claim}

Define $\tilde s:X\rightarrow [-1,1]\cup\{*\}$ such that
\begin{equation}
\label{eq:tilde-s}
\tilde s(x) = \begin{cases}
0, & \textnormal{if } s(x) \ne * \textnormal{ and } |s(x)| \le \eta;\\
s(x), & \textnormal{otherwise}.
\end{cases}
\end{equation}
Since $\Pr_{(x,y)\sim\mu}[s(x) = y] = 1$, equations \eqref{eq:dcm-cond-u} and \eqref{eq:tilde-s} imply 
that 
\begin{align}
& \Pr_{(x,y,u)\sim\nu}[u = 1|x,y] = |\tilde s(x)|,\label{eq:dist-u-cond}\\
\textnormal{and thus } & \Pr_{(x,y,u)\sim\nu}[u = 1] = \bE_{x\sim\mu|_X}|\tilde s(x)|.\label{eq:dist-u}
\end{align}
 (Note that $(x,y,u)$ generated from $\nu$ satisfies $s(x)\ne *$, or equivalently, $\tilde s(x)\ne *$ with probability $1$.)
The following claim allows us to evaluate expectations over the distribution $\mu'$:
\begin{claim}
\label{claim:dcm-rejection-0}
Let $\mu,\mu',s,\tilde s$ be defined as above.
Assuming $\bE_{x\sim\mu|_X}|\tilde s(x)| > 0$,
for every bounded function $g:X\times \{-1,1\}\rightarrow \bR$, we have
\begin{equation}
\label{eq:dist-change}
\bE_{(x,y)\sim\mu'}[g(x,y)] =\frac{1}{\bE_{x\sim\mu|_X}|\tilde s(x)|}\bE_{x\sim\mu|_X}\Big[g\Big(x,\sign(s(x))\Big)|\tilde s(x)|\Big].
\end{equation}
\end{claim}
\begin{proof}
By the definition of $\mu'$, 
\begin{equation}
\label{eq:dist-change-0}
\bE_{(x,y)\sim\mu'}[g(x,y)] = \frac{\bE_{(x,y,u)\sim\nu}[g(x,\sign(y))\one(u = 1)]}{\Pr_{(x,y,u)\sim\nu}[u = 1]}.
\end{equation}
Moreover, 
\begin{align}
\bE_{(x,y,u)\sim\nu}[g(x,\sign(y))\one(u = 1)] & = \bE_{(x,y)\sim\mu}[g(x,\sign(y)){\Pr}_{(x,y,u)\sim\nu}[u=1|x,y]]\notag\\
& = \bE_{(x,y)\sim\mu}[g(x,\sign (y))|\tilde s(x)|]\tag{by \eqref{eq:dist-u-cond}}\\
& = \bE_{x\sim\mu|_X}[g(x,\sign (s(x)))|\tilde s(x)|]. \tag{by $\Pr_{(x,y)\sim\mu}[s(x) = y] = 1$}
\end{align}
Plugging this equation and \eqref{eq:dist-u} into \eqref{eq:dist-change-0}, we get \eqref{eq:dist-change}.
\end{proof}
Using \Cref{claim:dcm-rejection-0}, we prove two helper claims \Cref{claim:dcm-rejection-feasible} and \Cref{claim:rejection}. \Cref{claim:dcm-rejection-feasible} shows that the distribution $\mu'$ satisfies the realizability assumption for the comparative learning task for $(S_{\eta}\sps 0,B)$. \Cref{claim:rejection} allows us to relate the guarantees of correlation maximization on $\mu$ and $\mu'$.
\begin{claim}
\label{claim:dcm-rejection-feasible}
Let $\eta,\mu,\mu',s,\tilde s$ be defined as above. Assume $\bE_{x\sim\mu|_X}|\tilde s(x)| > 0$. Then, $\Pr_{(x,y)\sim\mu'}[s_{\eta}\sps {0} (x) = y] = 1.$
\end{claim}
\begin{proof}
Plugging $g(x,y) = \one(s_{\eta}\sps 0(x) = y)$ into \eqref{eq:dist-change}, it suffices to prove that
\[
\bE_{x\sim\mu|_X}\left[\one\left(s_\eta\sps 0 (x) = \sign(s(x))\right)|\tilde s(x)|\right] = \bE_{x\sim\mu|_X}|\tilde s(x)|.
\]
This holds because by the definition of $s_\eta\sps 0$ and $\tilde s$, we have $s_\eta\sps 0(x) = \sign (s(x))$ whenever $\tilde s(x) \notin \{0,*\}$.
\end{proof}

Recall from the beginning of the section that for $u_1\in \bR$ and $u_2\in [-1,1]\cup\{*\}$, we define their \emph{generalized product} $u_1\ptimes u_2$ to be
\[
u_1\ptimes u_2 := 
\begin{cases}
u_1u_2,& \textnormal{if }u_2\in[-1,1],\\
-|u_1|, & \textnormal{if }u_2 = *.
\end{cases}
\]

\begin{claim}
\label{claim:rejection}
Let $\mu,\mu',\tilde s$ be defined as above.
Assume $\bE_{x\sim\mu|_X}|\tilde s(x)| > 0$.
Then for every $b:X\rightarrow [-1,1]\cup \{*\}$,
\[
\bE_{(x,y)\sim\mu'}[y\ptimes b(x)] = \frac 1{\bE_{x\sim\mu|_X}|\tilde s(x)|}\bE_{x\sim\mu|_X}[\tilde s(x)\ptimes b(x)].
\]
\end{claim}
\begin{proof}
Plugging $g(x,y) = y\ptimes b(x)$ into \eqref{eq:dist-change}, we get
\[
\bE_{(x,y)\sim\mu'}[y\ptimes b(x)] = \frac 1{\bE_{x\sim\mu|_X}|\tilde s(x)|}\bE_{x\sim\mu|_X}[(\sign (s(x))\ptimes b(x))|\tilde s(x)|].
\]
It is clear that $\sign(s(x)) = \sign(\tilde s(x))$ whenever $\tilde s(x)\notin \{0,*\}$. It is then easy to verify that $(\sign (s(x))\ptimes \allowbreak b(x))|\tilde s(x)| = \tilde s(x)\ptimes b(x)$ holds regardless of whether $b(x) = *$, assuming $\tilde s(x)\ne *$. Plugging this into the equation above completes the proof.
\end{proof}

We are now ready to prove \Cref{lm:dcm}.
\begin{proof}[Proof of \Cref{lm:dcm}]
We first show that $\comp_{n'}(S_{\eta}\sps{0},B,\varepsilon n/(4n'),\delta/2)\ne \emptyset$ whenever \Cref{line:dcm-invoke} is executed (assuming $n' > 0$). Define $\varepsilon' = \varepsilon n/n'$. It is clear that $n'\le n$, so $\varepsilon' \ge \varepsilon$. By our assumption,
\[
n \ge \frac{\varepsilon'}{\varepsilon}\scomp(S_{\eta}\sps{0},B, \varepsilon'/4,\delta/2).
\]
Plugging $\varepsilon' = \varepsilon n/n'$ into the inequality above, we get $n' \ge \scomp(S_{\eta}\sps{0},B,\varepsilon n/(4n'),\delta/2)$. This implies that $\comp_{n'}(S_{\eta}\sps{0},B,\varepsilon n/(4n'),\delta/2)\ne \emptyset$ as desired.

It remains to show that when the input data points of \Cref{alg:dcm} are generated i.i.d.\ from a distribution $\mu$ satisfying $\Pr_{(x,y)\sim\mu}[s(x) = y] = 1$ for some $s\in S$, with probability at least $1-\delta$, the output model $f$ satisfies
\begin{equation}
\label{eq:dcm-analysis-goal}
\bE_{(x,y)\sim\mu}[yf(x)] \ge {\sup}_{b\in B}\bE_{(x,y)\sim\mu}[y\ptimes b(x)] - \varepsilon - 2\eta.
\end{equation}
Since $\Pr_{(x,y)\sim\mu}[s(x) = y] = 1$, \eqref{eq:dcm-analysis-goal} is equivalent to 
\[
\bE_{x\sim\mu|_X}[s(x)f(x)] \ge {\sup}_{b\in B}\bE_{x\sim\mu|_X}[s(x)\ptimes b(x)] - \varepsilon - 2\eta.
\]
Define $\tilde s:X\to[-1,1]\cup\{*\}$ as in \eqref{eq:tilde-s}. Since $|s(x) - \tilde s(x)|\le \eta$ whenever $s(x)\ne *$, a sufficient condition for the inequality above is
\begin{equation}
\label{eq:dcm-analysis-goal-1}
\bE_{x\sim\mu|_X}[\tilde s(x)f(x)] \ge {\sup}_{b\in B}\bE_{x\sim\mu|_X}[\tilde s(x)\ptimes b(x)] - \varepsilon.
\end{equation}
Define $\rho:= \bE_{x\sim\mu|_X}|\tilde s(x)|$. It is clear that $\bE_{x\sim\mu|_X}[\tilde s(x)f(x)]$ and $\bE_{x\sim\mu|_X}[\tilde s(x)\ptimes b(x)]$ both lie in the interval $[-\rho,\rho]$ for every $b\in B$, so \eqref{eq:dcm-analysis-goal-1} holds trivially if $\rho \le \varepsilon/2$. We thus assume $\rho\ge \varepsilon/2$ without loss of generality. 

By \Cref{claim:rejection}, \eqref{eq:dcm-analysis-goal-1} is equivalent to
\begin{equation}
\label{eq:dcm-analysis-goal-3}
\bE_{(x,y)\sim\mu'}[yf(x)] \ge {\sup}_{b\in B}\bE_{(x,y)\sim\mu'}[y\ptimes b(x)] - \varepsilon/\rho.
\end{equation}
It thus suffices to show that \eqref{eq:dcm-analysis-goal-3} holds with probability at least $1-\delta$.

By the multiplicative Chernoff bound and our assumptions that $\rho \ge \varepsilon/2$ and $n\ge \frac{C}{\varepsilon}\log(1/\delta)$ for a sufficiently large absolute constant $C > 0$, with probability at least $1-\delta/2$, we have 
\begin{equation}
\label{eq:dcm-analysis-1}
n' = \sum_{i=1}^{n}u_i \ge \rho n/2,
\end{equation}
where $u_i\in \{0,1\}$ is defined in \Cref{claim:dcm-new-points} and satisfies $\Pr[u_i = 1] = \rho$ by \eqref{eq:dist-u}. By \Cref{claim:dcm-new-points}, \Cref{claim:dcm-rejection-feasible}, and the guarantee of $L\in\comp_{n'}(S_{\eta}\sps{0},\allowbreak B,\varepsilon n\sps 1/(4n'),\delta/2)$ at \Cref{line:dcm-invoke}, with probability at least $1-\delta/2$,
\[
{\Pr}_{(x,y)\sim\mu'}[f(x) \ne y] \le {\inf}_{b\in B}{\Pr}_{(x,y)\sim\mu'}[b(x) \ne y] + \varepsilon n/(4n'),
\]
or equivalently (by \eqref{eq:error-correlation}),
\begin{equation}
\label{eq:dcm-analysis-2}
\bE_{(x,y)\sim\mu'}[yf(x)] \ge {\sup}_{b\in B}\bE_{(x,y)\sim\mu'}[y\ptimes b(x)] - \varepsilon n/(2n').
\end{equation}
By the union bound, with probability at least $1-\delta$, \eqref{eq:dcm-analysis-1} and \eqref{eq:dcm-analysis-2} both hold, in which case \eqref{eq:dcm-analysis-goal-3} holds by plugging \eqref{eq:dcm-analysis-1} into \eqref{eq:dcm-analysis-2}.
\end{proof}

\subsection{Deterministic Labels with Real-Valued Benchmarks}
Now we prove a sample complexity upper bound for deterministic-label correlation maximization without the assumption that the benchmark class $B$ is binary:
\begin{theorem}
\label{thm:dcm-real}
Let $S,B\subseteq([-1,1]\cup\{*\})^X$ be real-valued hypothesis classes. 
For $\eta_1\ge 0,\eta_2,\beta\in (0,1/2)$, defining $m:=\sup_{\theta\in \bR}\vc(S_{\eta_1}\sps 0, B_{\eta_2}\sps \theta)$, we have
\begin{align*}
\sdcm(S,B,\beta + 2\eta_1 + 2\eta_2,\delta)
\le {} O\left(\frac{m}{\beta^2}\log^2_+\left(\frac{m}{\beta}\right)  + \frac{1}{\beta^2}\log\left(\frac 1{\eta_2\delta}\right)\right).
\end{align*}
Moreover, for $\varepsilon\in (0,1/2)$, choosing
$\beta = \eta_1 = \eta_2 = \varepsilon/5$, we have $m \le \fat_{\varepsilon/5}(S,B)$ and
\[
\sdcm(S,B,\varepsilon,\delta) \le O\left(\frac{m}{\varepsilon^2}\log^2_+\left(\frac{m}{\varepsilon}\right) + \frac 1{\varepsilon^2}\log\left(\frac 1{\varepsilon\delta}\right)\right).
\]
\end{theorem}

We prove \Cref{thm:dcm-real} using \Cref{alg:dcm-real} and the following lemma:
\begin{lemma}
\label{lm:dcm-real}
Let $S,B\subseteq([-1,1]\cup\{*\})^X$ be real-valued hypothesis classes. 
Suppose the parameters of \Cref{alg:dcm-real} satisfy $\eta_1 \ge 0, \varepsilon_1,\eta_2,\delta\in (0,1/2), n\sps 1 \ge C(\varepsilon_1 + \eta_2)^{-1}\log(1/\delta)$, and $n\sps 2 \ge C\varepsilon_2^{-2}\log(1/\eta_2\delta)$ for $\varepsilon_2\in (0,1)$ and a sufficiently large absolute constant $C > 0$. Suppose we choose $t$ to be the maximum integer satisfying $(2t + 1)\eta_2 < 1$ in \Cref{alg:dcm-real}. Assume in addition that
\[
n\sps 1 \ge {\sup}_{\theta\in \bR}{\sup}_{\varepsilon'\ge \varepsilon_1}\scomp(S_{\eta_1}\sps 0,B_{\eta_2}\sps \theta,\varepsilon'/4,\delta/4).
\]
Then \Cref{alg:dcm-real} belongs to $\dcm_{n}(S,B,\varepsilon_1 + \varepsilon_2 + 2\eta_1 + 2\eta_2,\delta)$, where $n = n\sps 1 + n\sps 2$.
\end{lemma}
\begin{algorithm}
\caption{Deterministic-label correlation maximization for $(S,B)$}
\label{alg:dcm-real}
\SetKwInOut{Parameters}{Parameters}
\SetKwInOut{Input}{Input}
\SetKwInOut{Output}{Output}
\Parameters{$S,B\subseteq([-1,1]\cup\{*\})^X$, $n,n\sps 1,n\sps 2\in \bZ_{> 0}$ satisfying $n = n\sps 1 + n\sps 2$, $t\in \bZ_{\ge 0},\eta_1,\eta_2,\varepsilon_1,\delta \in \bR_{\ge 0}$.}
\Input{data points $(x_1,y_1),\ldots,(x_n,y_n)\in X\times[-1,1]$.}
\Output{model $f^*:X\rightarrow\{-1,1\}$.}
Partition the input data points into two datasets: $\Psi\sps 1 = \left(\left(x_i\sps 1, y_i\sps 1\right)\right)_{i=1}^{n\sps 1}$ and $\Psi\sps 2 = \left(\left(x_i\sps 2, y_i\sps 2\right)\right)_{i=1}^{n\sps 2}$\;
Initialize $\Psi$ to be an empty dataset\;
\For{$i = 1,\ldots,n\sps 1$}{
\If{$|y_i\sps 1| > \eta_1$}
{
With probability $|y_i\sps 1|$, add the data point $(x_i\sps 1,\sign(y_i\sps 1))$ into $\Psi$\;
}
}
Let $n'$ be the number of data points in $\Psi$\;
\For{$j = -t,\ldots,t$}{
Invoke learner in $\comp_{n'}(S_{\eta_1}\sps{0},B_{\eta_2}\sps {2\eta_2j},\varepsilon_1 n\sps 1/(4n'),\delta/4)$ on $\Psi$ to obtain $f_j:X\rightarrow\{-1,1\}$\;
}
Define $f_{t + 1}, f_{-t - 1}:X\to \{-1,1\}$ to be constant functions: $f_{t + 1}(x) = 1$ and $f_{-t - 1}(x) = -1$ for every $x\in X$\;
\Return $f^*$ that maximizes $Q_f := \frac{1}{n\sps 2}\sum_{i=1}^{n\sps 2}y_i\sps 2 f(x_i\sps 2)$ over $f\in \{f_{-t - 1},\ldots,f_{t + 1}\}$\label{line:dcm-real-return}\;
\end{algorithm}
We first prove \Cref{thm:dcm-real} using \Cref{lm:dcm-real} and then prove \Cref{lm:dcm-real} after that.
\begin{proof}[Proof of \Cref{thm:dcm-real}]
We choose $\varepsilon_1 = \varepsilon_2 = \beta/2$ in \Cref{lm:dcm-real}.
Using the same argument as in the proof of \Cref{thm:dcm}, we have
\[
{\sup}_{\theta\in \bR}{\sup}_{\varepsilon'\ge \varepsilon_1}\scomp(S_{\eta_1}\sps 0,B_{\eta_2}\sps \theta,\varepsilon'/4,\delta/4) = O\left(\frac{m}{\beta^2}\log^2_+\left(\frac{m}{\beta}\right)  + \frac{1}{\beta^2}\log\left(\frac 1{\delta}\right)\right).
\]
Therefore, there exist
\begin{align*}
n\sps 1 & = O\left(\frac{m}{\beta^2}\log^2_+\left(\frac{m}{\beta}\right)  + \frac{1}{\beta^2}\log\left(\frac 1{\delta}\right)\right) \textnormal{ and }\\
n\sps 2 & = O\left(\frac 1{\beta^2}\log\left(\frac 1{\eta_2\delta}\right)\right)
\end{align*}
that satisfy the requirement of \Cref{lm:dcm-real}, which implies that $\dcm_n(S,B,\beta + 2\eta_1 + 2\eta_2,\delta)\ne\emptyset$ for
\[
n = n\sps 1 + n\sps 2 = O\left(\frac{m}{\beta^2}\log^2_+\left(\frac{m}{\beta}\right)  + \frac{1}{\beta^2}\log\left(\frac 1{\eta_2\delta}\right)\right).
\]
The fact that $m \le \fat_{\varepsilon/5}(S,B)$ when $\eta_1 = \eta_2 =\varepsilon/5$ follows from \Cref{claim:fat-alt}.
\end{proof}
Now we prove \Cref{lm:dcm-real} by analyzing \Cref{alg:dcm-real}.
\Cref{alg:dcm-real} is very similar to \Cref{alg:dcm}; the key difference is that \Cref{alg:dcm-real} transforms the real-valued benchmark class $B$ into binary classes $B_{\eta_2}\sps{2\eta_2j}$ for $j = -t,\ldots,t$. Thus, our proof of \Cref{lm:dcm-real} focuses on relating the class $B$ to the classes $B_{\eta_2}\sps{2\eta_2j}$.
The following claim shows that any benchmark $b\in B$ can be approximately expressed as a particular linear combination of benchmarks in $B_{\eta_2}\sps{2\eta_2j}$.
\begin{claim}
\label{claim:B-convex-hull}
Let $b:X\to [-1,1]\cup\{*\}$ be a real-valued hypothesis.
Consider a fixed $x\in X$ that satisfies $b(x) \ne *$. 
For $\eta_2\in (0,1/2)$, let $t$ denote the maximum integer satisfying $(2t + 1)\eta_2 < 1$.
For every $j = -t,\ldots,t$, define $p_j\in [-1,1]$ such that $p_j = b_{\eta_2}\sps {2\eta_2 j}(x)$ if $b_{\eta_2}\sps {2\eta_2 j}(x) \ne *$ and $p_j$ can be an arbitrary value in $[-1,1]$ if $b_{\eta_2}\sps {2\eta_2 j}(x) = *$. Then,
\[
\left|b(x) - \eta_2\sum_{j = -t}^t p_j\right| \le 2\eta_2.
\]
\end{claim}
\begin{proof}
We first show that 
\begin{equation}
\label{eq:B-convex-hull-1}
b(x) - \eta_2 \sum_{j = -t}^t p_j \le 2\eta_2.
\end{equation}
Since $p_j \ge -1$ for every $j = -t,\ldots,t$,
the inequality above is trivial if $b(x) \le (-2t - 1)\eta_2$. We thus assume $b(x) > (-2t - 1)\eta_2$. Now for some $j'\in \{-t,\ldots,t + 1\}$, we have $b(x)\in ((2j' - 1)\eta_2,(2j' + 1)\eta_2]$, and by the definition of $p_j$ and $b_{\eta_2}\sps{2\eta_2j}$, we have $p_j = b_{\eta_2}\sps{2\eta_2j}(x) = 1$ for $j = -t,\ldots,j' - 1$. Therefore, 
\[
\sum_{j = -t}^t p_j = \sum_{j = -t}^{j' - 1}p_j + \sum_{j = j'}^t p_j = (t + j') + \sum_{j = j'}^t p_j \ge (t + j') - (t + 1 - j') = 2j' - 1.
\]
Inequality \eqref{eq:B-convex-hull-1} follows by the inequality above and the fact $b(x) \le (2j' + 1)\eta_2$. The other direction $b(x) - \eta_2 \sum_{j = -t}^t p_j \ge -2\eta_2$ can be proved similarly.
\end{proof}
Using \Cref{claim:B-convex-hull},
we prove the following claim relating the maximum correlation achievable by a benchmark $b\in B$ to the maximum correlation achievable by a benchmark $b'\in \bigcup_{j=-t}^t B_{\eta_2}\sps{2\eta_2j}$.
\begin{claim}
\label{claim:B-transform}
Let $B\subseteq([-1,1]\cup\{*\})^X$ be a real-valued hypothesis class. For $\eta_2\in (0,1/2)$, let $t$ denote the maximum integer satisfying $(2t + 1)\eta_2 < 1$. Let $s:X\to [-1,1]\cup\{*\}$ be a hypothesis (not necessarily in $S$) and $\mu_X$ be a distribution over $X$ satisfying $\Pr_{x\sim\mu_X}[s(x) = *] = 0$. Then,
\[
\max\left\{0,{\sup}_{j = -t,\ldots,t}{\sup}_{b'\in B_{\eta_2}\sps{2\eta_2j}}\bE_{x\sim\mu_X}[s(x)\ptimes b'(x)]\right\} \ge {\sup}_{b\in B}\bE_{x\sim\mu_X}[s(x)\ptimes b(x)] - 2\eta_2.
\]
\end{claim}
\begin{proof}
Let us fix an arbitrary $b\in B$. We first show that for every $x\in X$ satisfying $s(x)\ne *$,
\begin{equation}
\label{eq:B-transform-1}
\eta_2\sum_{j= -t}^t s(x)\ptimes b_{\eta_2}\sps {2\eta_2j}(x) \ge s(x)\ptimes b(x) - 2\eta_2.
\end{equation}
If $b(x) = *$, then $b_{\eta_2}\sps {2\eta_2j}(x) = *$ for every $j$. In this case, inequality \eqref{eq:B-transform-1} simplifies to 
\[
-\eta_2(2t + 1)|s(x)| \ge -|s(x)| - 2\eta_2,
\]
which holds because $\eta_2(2t + 1) < 1$.
We thus assume $b(x) \ne *$. By the definition of $\diamondsuit$, for every $j= -t,\ldots,t$, there exists $p_j\in [-1,1]$ such that $s(x) p_j = s(x)\ptimes b_{\eta_2}\sps{2\eta_2j}(x)$ where we always choose $p_j = b_{\eta_2}\sps{2\eta_2j}(x)$ if $b_{\eta_2}\sps{2\eta_2j}(x) \ne *$. By \Cref{claim:B-convex-hull}, inequality \eqref{eq:B-transform-1} follows from the following chain:
\begin{align*}
\eta_2\sum_{j= -t}^t s(x)\ptimes b_{\eta_2}\sps {2\eta_2j}(x) = {} & \eta_2\sum_{j= -t}^t s(x)p_j\\
\ge {} &  s(x) b(x) - 2\eta_2| s(x)|\\
\ge {} &  s(x)b(x) - 2\eta_2\\
= {} &  s(x)\ptimes b(x) - 2\eta_2.
\end{align*}

Taking expectation over $x\sim\mu_X$, \eqref{eq:B-transform-1} implies
\[
\eta_2 \sum_{j = -t}^t\bE_{x\sim\mu_X}[ s(x) \ptimes b_{\eta_2}\sps{2\eta_2j}(x)] \ge \bE_{x\sim\mu_X}[ s(x)\ptimes b(x)] - 2\eta_2.
\]
This means that there exists $j\in \{-t,\ldots,t\}$ satisfying 
\[
\bE_{x\sim\mu_X}[s(x) \ptimes b_{\eta_2}\sps{2\eta_2j}(x)] \ge \frac{1}{(2t + 1)\eta_2}(\bE_{x\sim\mu_X}[ s(x)\ptimes b(x)] - 2\eta_2).
\]
Since $(2t + 1)\eta_2 < 1$, the inequality above implies
\[
\max\left\{0, {\sup}_{j = -t,\ldots,t}{\sup}_{b'\in B_{\eta_2}\sps{2\eta_2j}}\bE_{x\sim\mu_X}[s(x) \ptimes b'(x)]\right\} \ge \bE_{x\sim\mu_X}[ s(x)\ptimes b(x)] - 2\eta_2.
\]
The proof is completed because the above inequality holds for every $b\in B$.
\end{proof}

We are now ready to prove the following lemma showing that before \Cref{alg:dcm-real} returns at \Cref{line:dcm-real-return}, with large probability, at least one of the candidate models $f_j$ achieves a large correlation compared to the benchmarks in $B$. By \Cref{claim:B-transform}, we just need to compare with the binary benchmarks in $B_{\eta_2}\sps{2\eta_2j}$, allowing us to follow the same proof as in \Cref{sec:dcm}.
\begin{lemma}
\label{lm:dcm-exist}
In the setting of \Cref{lm:dcm-real}, assume that the input data points to \Cref{alg:dcm-real} are generated i.i.d.\ from a distribution $\mu$ satisfying $\Pr_{(x,y)\sim\mu}[s(x) = y] = 1$ for some $s\in S$.
Then with probability at least $1-\delta/2$,
at \Cref{line:dcm-real-return},
there exists $j\in \{-t - 1,\ldots,t + 1\}$ such that 
\begin{equation}
\label{eq:dcm-exist-0}
\bE_{(x,y)\sim\mu}[yf_j(x)] \ge {\sup}_{b\in B}\bE_{(x,y)\sim\mu}[y\ptimes b(x)] - \varepsilon_1 - 2\eta_1 - 2\eta_2.
\end{equation}
\end{lemma}
\begin{proof}
Since $\Pr_{(x,y)\sim\mu}[s(x) = y] =1$, inequality \eqref{eq:dcm-exist-0} is equivalent to
\begin{equation}
\label{eq:dcm-exist-1}
\bE_{x\sim\mu|_X}[s(x)f_j(x)] \ge {\sup}_{b\in B}\bE_{x\sim\mu|_X}[s(x)\ptimes b(x)] - \varepsilon_1 - 2\eta_1 - 2\eta_2.
\end{equation}
Define $\tilde s:X\to [-1,1]\cup\{*\}$ as in \eqref{eq:tilde-s} with $\eta$ replaced by $\eta_1$.
Since $|s(x) - \tilde s(x)| \le \eta_1$ whenever $s(x)\ne *$, it suffices to show that with probability at least $1-\delta/2$, there exists $j\in \{-t - 1,\ldots,t + 1\}$ such that 
\begin{equation}
\label{eq:dcm-exist}
\bE_{x\sim\mu|_X}[\tilde s(x)f_j(x)] \ge {\sup}_{b\in B}\bE_{x\sim\mu|_X}[\tilde s(x)\ptimes b(x)] - \varepsilon_1 - 2\eta_2.
\end{equation}
Define $\rho:= \bE_{x\sim\mu|_X}|\tilde s(x)|$. It is clear that $\bE_{x\sim\mu|_X}[\tilde s(x)f_j(x)]$ and $\bE_{x\sim\mu|_X}[\tilde s(x)\ptimes b(x)]$ lie in the interval $[-\rho,\rho]$ for every $j\in \{-t - 1,\ldots,t + 1\}$ and $b\in B$, so the inequality above holds trivially if $\rho \le (\varepsilon_1 + 2\eta_2)/2$. We thus assume $\rho > (\varepsilon_1 + 2\eta_2)/2$ without loss of generality. Also, since $f_{t+1}$ and $f_{-t - 1}$ are the constant $1$ and constant $-1$ functions, \eqref{eq:dcm-exist} holds trivially if its right-hand-side is negative or zero. We thus assume that the right-hand-side of \eqref{eq:dcm-exist} is positive.

By \Cref{claim:B-transform}, there exists $j\in \{-t,\ldots,t\}$ such that %
\begin{equation}
\label{eq:comp-main-1}
{\sup}_{b'\in B_{\eta_2}\sps{2\eta_2j}}\bE_{x\sim\mu|_X}[\tilde s(x)\ptimes b'(x)] \ge {\sup}_{b\in B}\bE_{x\sim\mu|_X}[\tilde s(x)\ptimes b(x)] - 2\eta_2.
\end{equation}
We fix the $j$ that satisfies the inequality above. It now suffices to prove that with probability at least $1-\delta/2$,
\[
\bE_{x\sim\mu|_X}[\tilde s(x)f_j(x)] \ge {\sup}_{b'\in B_{\eta_2}\sps{2\eta_2j}}\bE_{x\sim\mu|_X}[\tilde s(x)\ptimes b'(x)] - \varepsilon_1.
\]
This is the same inequality as \eqref{eq:dcm-analysis-goal-1} except that we replace $(f,B,\varepsilon,\delta)$ with $(f_j,B_{\eta_2}\sps{2\eta_2j},\varepsilon_1,\delta/2)$, and it follows from the same argument as in the proof of \Cref{lm:dcm}.
\end{proof}

\begin{proof}[Proof of \Cref{lm:dcm-real}]
Suppose the input data points to \Cref{alg:dcm-real} are drawn i.i.d.\ from a distribution $\mu$ satisfying $\Pr_{(x,y)\sim\mu}[s(x) = y] = 1$ for some $s\in S$. Let us consider the models $f_{-t - 1},\ldots,f_{t + 1}$ at \Cref{line:dcm-real-return}.
By our assumptions that $t = O(1/\eta_2)$ and $n\sps 2\ge C\varepsilon_2^{-2}\log(1/(\eta_2\delta))$ for a sufficiently large absolute constant $C>0$, for every $f\in \{f_{-t - 1},\ldots,f_{t + 1}\}$, by the Chernoff bound, with probability at least $1-\delta/(4t + 6)$, it holds that
\begin{equation}
\label{eq:dcm-1}
|\bE_{(x,y)\sim\mu}[yf(x)] - Q_f| \le \varepsilon_2/2.
\end{equation}
Combining this with \Cref{lm:dcm-exist} using the union bound, with probability at least $1-\delta$,
inequality \eqref{eq:dcm-1} holds simultaneously for all $f\in \{f_{-t - 1},\ldots,f_{t + 1}\}$, and there exists $f'\in \{f_{-t - 1},\ldots,f_{t + 1}\}$ such that 
\begin{equation}
\label{eq:dcm-2}
\bE_{(x,y)\sim\mu}[yf'(x)] \ge {\sup}_{b\in B}\bE_{(x,y)\sim\mu}[y\ptimes b(x)] - \varepsilon_1 - 2\eta_1 - 2\eta_2.
\end{equation}
Therefore, with probability at least $1-\delta$, the output model $f^*$ of \Cref{alg:dcm-real} satisfies
\begin{align*}
\bE_{(x,y)\sim\mu}[yf^*(x)] & \ge Q_{f^*} - \varepsilon_2/2 \tag{by \eqref{eq:dcm-1}}\\
& \ge Q_{f'}  - \varepsilon_2/2\tag{by definition of $f^*$}\\
& \ge \bE_{(x,y)\sim\mu}[yf'(x)] - \varepsilon_2\tag{by \eqref{eq:dcm-1}}\\
& \ge {\sup}_{b\in B}\bE_{(x,y)\sim\mu}[y\ptimes b(x)] - \varepsilon_1 - \varepsilon_2 - 2\eta_1 - 2\eta_2 \tag{by \eqref{eq:dcm-2}}.
\end{align*}
This proves that \Cref{alg:dcm-real} belongs to $\dcm_{n}(S,B,\varepsilon_1 + \varepsilon_2 + 2\eta_1 + 2\eta_2,\delta)$, as desired.
\end{proof}
\subsection{General Case of Correlation Maximization}
We are now ready to consider the general case of correlation maximization without the deterministic-labels assumption and prove \Cref{thm:cm} stated at the beginning of \Cref{sec:cm}.

In the general case of correlation maximization, an input data point $(x,y)$ may no longer satisfy $y = s(x)$ for the source hypothesis $s$, and thus the data point itself may not be informative enough for our learner to decide the rejection probability in the rejection sampling procedure we use in \Cref{alg:dcm,alg:dcm-real}. Our idea of solving this problem is to change the label $y$ in every data point and enumerate over all possibilities of the changes for all input data points $(x,y)$. We make sure that in one of the possibilities we have $y \approx s(x)$ for every data point and thus our algorithm in the deterministic-labels setting gives us a good model. We compute a model for each possibility and do a final testing step to choose an approximately best one.
This idea is inspired by the \emph{non-uniform covering} technique used by \citet{hopkins2022realizable}.

In order to allow efficient enumeration over $y\in [-1,1]$, we need a discretized version of the continuous interval $[-1,1]$. To that end, for $\eta_1\in (0,1)$,
we choose $Y\subseteq[-1,1]$ such that
\begin{enumerate}
\item $0\in Y$;
\item for every $u\in [-1,1]$, there exists $y\in Y$ such that $|u - y| \le \eta_1$;
\item $|Y| \le O(1/\eta_1)$.
\end{enumerate}
It is clear that there exists $Y\subseteq[-1,1]$ satisfying all the above properties (for example, take $Y$ to be the set of all integer multiples of $\eta_1$ in $[-1,1]$). 
We design a learner in \Cref{alg:cm} using this definition of $Y$, and we prove the following lemma showing that \Cref{alg:cm} solves correlation maximization with a desired sample complexity:
\begin{lemma}
\label{lm:cm}
Let $S,B\subseteq([-1,1]\cup\{*\})^X$ be real-valued hypothesis classes.
Suppose the parameters of \Cref{alg:cm} satisfy $\varepsilon_1,\eta_1,\eta_2, \delta\in (0,1/2)$,
\begin{align*}
n\sps 1 & \ge \sup_{\theta\in \bR}\sup_{\varepsilon'\ge \varepsilon_1}\scomp(S_{\eta_1}\sps 0,B_{\eta_2}\sps \theta,\varepsilon'/4,\delta/4), \\
n\sps 1 & \ge C(\varepsilon_1 + \eta_2)^{-1}\log(1/\delta),\\
n\sps 2 & \ge C\varepsilon_2^{-2}(n\sps 1\log(1/\eta_1) + \log (1/\eta_2\delta))
\end{align*}
for $\varepsilon_2\in (0,1)$ and a sufficiently large absolute constant $C > 0$. 
Suppose in \Cref{alg:cm} we choose $t$ to be the maximum integer satisfying $(2t + 1)\eta_2 < 1$ and choose $Y\subseteq[-1,1]$ as above.
Then \Cref{alg:cm} belongs to $\cm_n(S,B,\varepsilon_1 + \varepsilon_2 + 2\eta_1 + 2\eta_2,\delta)$ where $n = n\sps 1 + n\sps 2$.
\end{lemma}
\begin{algorithm}
\caption{Correlation maximization for $(S,B)$}
\label{alg:cm}
\SetKwInOut{Parameters}{Parameters}
\SetKwInOut{Input}{Input}
\SetKwInOut{Output}{Output}
\Parameters{$S,B\subseteq([-1,1]\cup\{*\})^X$, $Y\subseteq[-1,1]$, $n,n\sps 1,n\sps 2\in \bZ_{> 0}$ satisfying $n = n\sps 1 + n\sps 2$, $t\in \bZ_{\ge 0},\eta_1,\eta_2,\varepsilon_1,\delta \in \bR_{\ge 0}$.}
\Input{data points $(x_1,y_1),\ldots,(x_n,y_n)\in X\times[-1,1]$.}
\Output{model $f^*:X\rightarrow\{-1,1\}$.}
Partition the input data points into two datasets: $\Psi\sps 1 = \left(\left(x_i\sps 1, y_i\sps 1\right)\right)_{i=1}^{n\sps 1}$ and $\Psi\sps 2 = \left(\left(x_i\sps 2, y_i\sps 2\right)\right)_{i=1}^{n\sps 2}$\;
\For{$\by := (\hat y_1,\ldots,\hat y_{n\sps 1})\in Y^{n\sps 1}$\label{line:cm-outer-for-1}}
{
Initialize $\Psi$ to be an empty dataset\;
\For{$i = 1,\ldots,n\sps 1$\label{line:cm-for-1}}{
With probability $|\hat y_i\sps 1|$, add the data point $(x_i\sps 1,\sign(\hat y_i\sps 1))$ into $\Psi$\;
}\label{line:cm-for-2}
Let $n'$ be the number of data points in $\Psi$\label{line:cm-gamma}\;
\For{$j = -t,\ldots,t$}
{
Invoke learner in $\comp_{n'}(S_{\eta_1}\sps{0},B_{\eta_2}\sps {2\eta_2j},\varepsilon_1 n\sps 1/(4n'),\delta/4)$ on $\Psi$ to obtain $f_{\by,j}:X\rightarrow\{-1,1\}$\;
}
}\label{line:cm-outer-for-2}
Define $f^+, f^-:X\to \{-1,1\}$ to be constant functions: $f^+(x) = 1$ and $f^-(x) = -1$ for every $x\in X$\;
\Return $f^*$ that maximizes $Q_f := \frac{1}{n\sps 2}\sum_{i=1}^{n\sps 2}y_i\sps 2 f(x_i\sps 2)$ over $f\in \{f^+,f^-\}\cup \{f_{\by,j}: \by\in Y^{n\sps 1},j\in \{-t,\ldots,t\}\}$\label{line:cm-return}\;
\end{algorithm}
We first prove \Cref{thm:cm} using \Cref{lm:cm} and then prove \Cref{lm:cm} after that.
\begin{proof}[Proof of \Cref{thm:cm}]
We choose $\varepsilon_1 = \varepsilon_2 = \beta/2$ in \Cref{lm:cm}. Using the same argument as in the proof of \Cref{thm:dcm}, we have
\[
\sup_{\theta\in \bR}\sup_{\varepsilon'\ge \varepsilon_1}\scomp(S_{\eta_1}\sps 0,B_{\eta_2}\sps \theta,\varepsilon'/4,\delta/4) = O\left(\frac{m}{\beta^2}\log^2_+\left(\frac{m}{\beta}\right)  + \frac{1}{\beta^2}\log\left(\frac 1{\delta}\right)\right).
\]
Therefore, there exist
\begin{align*}
n\sps 1 & = O\left(\frac{m}{\beta^2}\log^2_+\left(\frac{m}{\beta}\right)  + \frac{1}{\beta^2}\log\left(\frac 1{\delta}\right)\right) \textnormal{ and }\\
n\sps 2 & = O\left(\frac{m}{\beta^4}\log^2_+\left(\frac{m}{\beta}\right)\log\left(\frac 1{\eta_1}\right)  + \frac{1}{\beta^4}\log\left(\frac 1{\eta_1}\right)\log\left(\frac 1{\delta}\right) + \frac 1{\beta^2}\log\left(\frac 1{\eta_2}\right)\right)
\end{align*}
that satisfy the requirement of \Cref{lm:cm}, which implies that $\cm_n(S,B,\beta + 2\eta_1 + 2\eta_2,\delta)\ne \emptyset$ for
\[
n = n\sps 1 + n\sps 2 = O\left(\frac{m}{\beta^4}\log^2_+\left(\frac{m}{\beta}\right)\log\left(\frac 1{\eta_1}\right)  + \frac{1}{\beta^4}\log\left(\frac 1{\eta_1}\right)\log\left(\frac 1{\delta}\right) + \frac 1{\beta^2}\log\left(\frac 1{\eta_2}\right)\right).
\]
The fact that $m \le \fat_{\varepsilon/5}(S,B)$ when $\eta_1 = \eta_2 =\varepsilon/5$ follows from \Cref{claim:fat-alt}.
\end{proof}
To prove \Cref{lm:cm}, we apply ideas in our previous subsections with some small changes to the definition of the hypothesis $\tilde s$ and the distributions $\nu$ and $\mu'$.
By the definition of $Y$, there exists $\tau:[-1,1]\rightarrow Y$ such that
\begin{enumerate}
\item $\tau(u) = 0$ for every $u\in [-\eta_1,\eta_1]$, and
\item $|\tau (u) - u| \le \eta_1$ for every $u\in [-1,1]$.
\end{enumerate}
We define $\tilde s:X\rightarrow[-1,1]\cup\{*\}$ such that
\[
\tilde s(x) = \begin{cases}
\tau(s(x)), & \textnormal{if }s(x)\ne *;\\
*, & \textnormal{otherwise}.
\end{cases}
\]
For a distribution $\mu$ over $X\times[-1,1]$ satisfying $\Pr_{x\sim\mu|_X}[s(x) \ne *] = 1$ and $\bE_{(x,y)\sim\mu}[y|x] = s(x)$ for some $s\in S$,
we define $\nu$ to be the joint distribution over $(x,\tilde s(x),u)\sim X\times[-1,1]\times\{-1,1\}$ where $x\sim\mu|_X$ and $\Pr[u = 1|x] = |\tilde s(x)|$, and we define $\mu'$ to be the conditional distribution of $(x,\sign(y))$ given $u = 1$ for $(x,y,u)\sim\nu$. The following analogue of \Cref{claim:dcm-new-points} follows directly from the description of \Cref{alg:cm}:
\begin{claim}
\label{claim:cm-new-points}
Let the distributions $\mu,\nu,\mu'$ be defined as above.
Assume that the input data points $(x_i,y_i)_{i=1}^n$ to \Cref{alg:cm} are generated i.i.d.\ from $\mu$.
Let us focus on the single iteration of the outer \for loop (Lines~\ref{line:cm-outer-for-1}-\ref{line:cm-outer-for-2}) where $\hat y_i\sps 1 = \tilde s(x_i\sps 1)$ for every $i = 1,\ldots,n\sps 1$.
For every $i = 1,\ldots, n\sps 1$, let $u_i$ denote the indicator for the event that a new data point is added to $\Psi$ in the $i$-th iteration of the inner \for loop at Lines~\ref{line:cm-for-1}-\ref{line:cm-for-2}. 
Then $(x_i\sps 1, \hat y_i\sps 1, u_i)_{i=1}^{n\sps 1}$ are distributed independently according to $\nu$.
Also, when conditioned on $n'$, the data points in $\Psi$ at \Cref{line:cm-gamma} are distributed independently according to $\mu'$.
\end{claim}
\begin{lemma}
\label{lm:cm-exist}
In the setting of \Cref{lm:cm}, assume that the input data points to \Cref{alg:cm} are generated i.i.d.\ from a distribution $\mu$ over $X\times[-1,1]$ satisfying $\Pr_{x\sim\mu|_X}[s(x) \ne *] = 1$ and $\bE_{(x,y)\sim\mu}[y|x] = s(x)$ for some $s\in S$.
With probability at least $1-\delta/2$, at \Cref{line:cm-return}, there exists $f\in \{f^+,f^-\}\cup \{f_{\by,j}: \by\in Y^{n\sps 1},j\in \{-t,\ldots,t\}\}$ such that
\begin{equation}
\label{eq:cm-exist}
\bE_{(x,y)\sim\mu}[yf(x)]\ge {\sup}_{b\in B}\bE_{(x,y)\sim\mu}[y\ptimes b(x)] - \varepsilon_1 - 2\eta_1 - 2\eta_2.
\end{equation}
\end{lemma}
\begin{proof}
We first show that $\bE_{(x,y)\sim\mu}[y\ptimes b(x)|x] \le s(x)\ptimes b(x)$. When $b(x)\ne *$, the inequality becomes an equality because $\bE_{(x,y)\sim\mu}[y|x] = s(x)$. When $b(x) = *$, the inequality is equivalent to $\bE_{(x,y)\sim\mu}\big[-|y|\big|x\big]\le -\big|\bE_{(x,y)\sim\mu}[y|x]\big|$, which holds by Jensen's inequality.

Now we know that $\bE_{(x,y)\sim\mu}[y\ptimes b(x)] \le \bE_{x\sim\mu|_X}[s(x)\ptimes b(x)]$. Therefore, a sufficient condition for \eqref{eq:cm-exist} is
\begin{equation}
\label{eq:cm-exist-1}
\bE_{x\sim\mu|_X}[s(x)f(x)] \ge {\sup}_{b\in B}\bE_{x\sim\mu|_X}[s(x)\ptimes b(x)] - \varepsilon_1 - 2\eta_1 - 2\eta_2.
\end{equation}

We fix $\by = (\hat y_1,\ldots,\hat y_{n\sps 1})\in Y^{n\sps 1}$ such that $\hat y_i = \tilde s(x_i\sps 1)$ for every $i = 1,\ldots, n\sps 1$. It suffices to show that with probability at least $1-\delta/2$, there exists $f\in \{f^+,f^-\}\cup\{f_{\by,j}:j\in \{-t,\ldots,t\}\}$ such that \eqref{eq:cm-exist-1} holds. This is very similar to inequality \eqref{eq:dcm-exist-1} in the proof of \Cref{lm:dcm-exist}, and we can essentially apply the same proof here. The only difference is that we are using slightly different definitions for $\tilde s(x),\nu,\mu'$ here, but all the properties of $\tilde s,\nu,\mu'$ needed in the proof of \Cref{lm:dcm-exist} still hold. In particular, we still have $|\tilde s(x) - s(x)| \le \eta_1$ for every $x\in X$ whenever $s(x) \ne *$, and we can use \Cref{claim:cm-new-points} in place of \Cref{claim:dcm-new-points}. It is also straightforward to show that \Cref{claim:dcm-rejection-feasible} and \Cref{claim:rejection} still hold with our new definitions of $\tilde s, \nu,\mu'$ after replacing $\eta$ with $\eta_1$. All other components in the proof of \Cref{lm:dcm-exist} can be applied here without change.
\end{proof}

\begin{proof}[Proof of \Cref{lm:cm}]
Suppose the input data points to \Cref{alg:cm} are drawn i.i.d.\ from a distribution $\mu$ satisfying $\Pr_{x\sim\mu|_X}[s(x) \ne *] = 1$ and $\bE_{(x,y)\sim\mu}[y|x] = s(x)$ for some $s\in S$. Let us consider the models in $F_{\mathsf{candidate}} := \{f^+,f^-\}\cup \{f_{\by,j}: \by\in Y^{n\sps 1},j\in \{-t,\ldots,t\}\}$ at \Cref{line:cm-return}.
By the fact that $|Y| \le O(1/\eta_1), t \le O(1/\eta_2)$ and our assumption $n\sps 2\ge C\varepsilon^{-2}(n\sps 1\log(1/\eta_1) + \log (1/\eta_2\delta))$ for a sufficiently large absolute constant $C > 0$, 
for every $f\in F_{\mathsf{candidate}}$, by the Chernoff bound, with probability at least $1 - \delta/ (2|F_{\mathsf{candidate}}|)$, it holds that
\begin{equation}
\label{eq:cm-1}
|Q_{f} - \bE_{(x,y)\sim\mu}[yf(x)]|\le \varepsilon_2/2.
\end{equation}
Combining this with \Cref{lm:cm-exist} using the union bound, with probability at least $1-\delta$, inequality \eqref{eq:cm-1} holds simultaneously for all $f\in F_{\mathsf{candidate}}$, and there exists $f'\in F_{\mathsf{candidate}}$ such that
\begin{equation}
\label{eq:cm-2}
\bE_{(x,y)\sim\mu}[yf'(x)]\ge {\sup}_{b\in B}\bE_{(x,y)\sim\mu}[y\ptimes b(x)] - \varepsilon_1 - 2\eta_1 - 2\eta_2.
\end{equation}
Note that \eqref{eq:cm-1} and \eqref{eq:cm-2} are similar to \eqref{eq:dcm-1} and \eqref{eq:dcm-2}, respectively, and the proof is completed using the same argument as in the proof of \Cref{lm:dcm-real}.
\end{proof}
\section{Sample Complexity of Realizable Multiaccuracy and Multicalibration}
\label{sec:ma-mc}
In this section, we give a sample complexity characterization for realizable multiaccuracy and multicalibration in the distribution-free setting. These tasks have been studied by \citep{hu2022metric} for total hypothesis classes. Here we generalize their definitions to partial hypothesis classes.

Given a distribution $\mu$ over $X\times[-1,1]$ and a model $f:X\to [-1,1]$, we first generalize the definition of $\mae_{\mu,B}(f)$ and $\mce_{\mu,B}(f)$ in \eqref{eq:mae-total} and \eqref{eq:mce-total} to partial hypothesis classes $B\subseteq[-1,1]\cup\{*\}$. It is not enough to directly use the generalized product $\diamondsuit$ as in \Cref{sec:cm}. For example, suppose we define $\mae_{\mu,B}(f)$ to be
\[
{\sup}_{b\in B}|\bE_{(x,y)\sim\mu}[(f(x) - y)\ptimes b(x)]|.
\]
Then $\mae_{\mu,B}(f)$ equals to the $\ell_1$ error $\bE_{(x,y)\sim\mu}[|f(x) - y|]$ even when $B$ only contains a single hypothesis $b$ which assigns every individual $x\in X$ the undefined label $b(x) = *$, making it challenging to achieve a low $\mae$ even when $B$ has fat-shattering dimension zero. To avoid this issue, we note that for any $u\in \bR$, the absolute value $|u|$ can be equivalently written as $\sup_{\sigma\in \{-1,1\}}u\sigma$, leading us to the following definitions:
\begin{align}
\mae_{\mu,B}(f) & := \sup_{b\in B}\sup_{\sigma\in \{-1,1\}}\bE\Big[\Big((f(x) - y)\sigma\Big)\ptimes b(x)\Big], \quad \text{and} \label{eq:ma-def}\\
\mce_{\mu,B}(f) & := \sup_{b\in B}\sum_{v\in V}\sup_{\sigma\in \{-1,1\}}\bE\Big[\Big((f(x) - y)\one(f(x) = v)\sigma\Big)\ptimes b(x)\Big]. \label{eq:mc-def}
\end{align}
In the definition of $\mce$, we use $V$ to denote the range of $f$ which we assume to be countable. The supremum over $\sigma\in \{-1,1\}$ is inside the sum over $v\in V$, so $\sigma$ is allowed to depend on $v$. 

We can now define realizable multiaccuracy and multicalibration for partial hypothesis classes:

\begin{definition}[Realizable Multiaccuracy ($\ma$)]
\label{def:ma}
Given two hypothesis classes $S,B\subseteq([-1,1]\cup\{*\})^X$, an error bound $\varepsilon \ge 0$, a failure probability bound $\delta \ge 0$, and a nonnegative integer $n$, we define $\ma_n(S,B,\varepsilon,\delta)$ to be $\learn_n(Z,F,\dstr,(F_\mu)_{\mu\in \dstr},\delta)$ where $Z = X\times[-1,1]$, $F = [-1,1]^X$, $\dstr$ consists of all distributions $\mu$ over $X\times [-1,1]$ satisfying \eqref{eq:cm-assumption}, and $F_\mu$ consists of all models $f:X\rightarrow[-1,1]$ such that 
\begin{equation}
\label{eq:ma-goal}
\mae_{\mu,B}(f) \le \varepsilon.
\end{equation}
\end{definition}

\begin{definition}[Realizable Multicalibration ($\mc$)]
\label{def:mc}
We define $\mc_n(S,B,\varepsilon,\delta)$ in the same way as we define $\ma_n(S,B,\varepsilon,\delta)$ in \Cref{def:ma} except that we replace \eqref{eq:ma-goal} with
\[
\mce_{\mu,B}(f) \le \varepsilon.
\]
\end{definition}

We prove the following upper bound (\Cref{thm:ma-mc,thm:ma-mc-binary}) and lower bound (\Cref{thm:ma-mc-lower}) on the sample complexity of realizable multiaccuracy and multicalibration in \Cref{sec:ma-mc-upper,sec:ma-mc-lower}, respectively.
\begin{theorem}
\label{thm:ma-mc}
Let $S,B\subseteq([-1,1]\cup\{*\})^X$ be real-valued hypothesis classes. 
For $\beta,\eta_1,\eta_2,\delta\in (0,1/2)$, defining $m:= \sup_{r:X\to \bR}\sup_{\theta\in\bR}\vc(S_{\eta_1}\sps r, B_{\eta_2}\sps \theta)$, we have
\begin{align*}
& \sma(S,B,\beta + 2\eta_1 + 4\eta_2,\delta)\\
\le {} & \smc(S,B,\beta + 2\eta_1 + 4\eta_2,\delta)\\
\le {} & O\left(
\frac{m}{\beta^6}\log^2_+\left(\frac{m}{\beta}\right)\log\left(\frac 1{\eta_1}\right) + \frac 1{\beta^6}\log\left(\frac 1{\eta_1}\right)\log\left(\frac 1{\beta\delta}\right) + \frac{1}{\beta^4}\log\left(\frac 1{\eta_2}\right)   \right).
\end{align*}
For $\varepsilon\in (0,1/2)$, 
choosing $\beta = \eta_1 = \eta_2 = \varepsilon/7$, we have $m\le \fat_{\varepsilon/7}(S,B)$ and
\[
\sma(S,B,\varepsilon,\delta) \le \smc(S,B,\varepsilon,\delta) \le O \left( 
	\frac{m}{\varepsilon^6}\log^2_+\left(
		\frac{m}{\varepsilon}
	\right)
	\log\left(\frac 1\varepsilon\right) + \frac{1}{\varepsilon^6}\log\left(\frac 1\varepsilon\right)\log\left(\frac 1{\varepsilon\delta}\right) 
\right).
\]
\end{theorem}
\begin{theorem}
\label{thm:ma-mc-binary}
In the setting of \Cref{thm:ma-mc}, assume in addition that $S$ is binary, i.e., $S\subseteq\{-1,1,*\}^X$ and define $m:= \sup_{\theta\in \bR}\vc(S,B_{\eta_2}\sps \theta)$. Then,
\begin{align*}
\sma(S,B,\beta + 4\eta_2,\delta) & \le O\left(\frac{m}{\beta^4}\log^2_+\left(\frac{m}{\beta}\right)  + \frac{1}{\beta^4}\log\left(\frac 1{\eta_2\beta\delta}\right)\right),\\
\smc(S,B,\beta + 4\eta_2,\delta) & \le O\left(\frac{m}{\beta^4}\log^2_+\left(\frac{m}{\beta}\right)  + \frac{1}{\beta^4}\log\left(\frac 1{\eta_2\beta\delta}\right) + \frac 1{\beta^5}\right).
\end{align*}
\end{theorem}
\begin{theorem}
\label{thm:ma-mc-lower}
Let $S,B\subseteq([-1,1]\cup\{*\})^X$ be real-valued hypothesis classes. For $\eta_1,\eta_2\in \bR_{> 0}$ and $\delta\in (0,1)$, defining $m:=\sup_{r_1,r_2:X\to \bR}\vc(S_{\eta_1}\sps {r_1}, B_{\eta_2}\sps {r_2})$, we have
\[
\smc(S,B,\eta_1\eta_2/3,\delta)  \ge \sma(S,B,\eta_1\eta_2/3,\delta) \ge \log(1 - \delta) + \Omega(m).
\]
For any $\varepsilon\in (0,1/2)$, choosing $\eta_1 = \eta_2  = \sqrt{3\varepsilon}$, we have $m = \fat_{\sqrt{3\varepsilon}}(S,B)$ and
\[
\smc(S,B,\varepsilon,\delta) \ge \sma(S,B,\varepsilon,\delta) \ge \log(1 - \delta) + \Omega(m).
\]
Moreover, the constant $3$ in the theorem can be replaced by any absolute constant $c > 2$.
\end{theorem}
\subsection{Upper Bound}
\label{sec:ma-mc-upper}
To prove our sample complexity upper bound (\Cref{thm:ma-mc,thm:ma-mc-binary}) for realizable multiaccuracy and multicalibration, we use ideas from \citep{hebert2018multicalibration} where a \emph{weak agnostic learner} for a hypothesis class $B$ is used to achieve multiaccuracy and multicalibration w.r.t.\ $B$. 
In our setting with an additional source class $S$, we use learners that solve \emph{weak correlation maximization} for multiple choices of $(S',B')$ in place of the weak agnostic learner:
\begin{definition}[Weak correlation maximization ($\wcm$)]
\label{def:wcm}
Given two real-valued hypothesis classes $S,B\subseteq([-1,1]\cup\{*\})^X$, parameters $\alpha,\gamma \ge0$, a failure probability bound $\delta \ge 0$, and a nonnegative integer $n$, we define $\wcm_n(S,B,\alpha,\gamma,\delta)$ to be $\learn_n(Z,F,\dstr,(F_\mu)_{\mu\in \dstr},\delta)$ with $Z,F,\dstr,F_\mu$ chosen as follows. We choose $Z = X\times[-1,1]$ and $F = \{-1,1\}^X$. The distribution class $\dstr$ consists of all distributions $\mu$ over $X\times [-1,1]$ satisfying the following properties:
\begin{align}
&\bullet\; \text{there exists $s\in S$ such that ${\Pr}_{x\sim\mu|_X}[s(x)\ne *] = 1$ and $\bE_{(x,y)\sim\mu}[y|x] = s(x)$};\label{eq:wcm-assumption}\\
&\bullet\; {\sup}_{b\in B}\bE_{(x,y)\sim\mu}[y\ptimes b(x)] \ge \alpha.\notag
\end{align}
The admissible set $F_\mu$ consists of all models $f:X\rightarrow \{-1,1\}$ satisfying
\[
\bE_{(x,y)\sim\mu}[yf(x)] \ge \gamma.
\]
\end{definition}
\begin{definition}[Weak deterministic-label correlation maximization ($\wdcm$)]
\label{def:wdcm}
We define \\$\wdcm_n(S,B,\alpha,\gamma,\delta)$ in the same way as we define $\wcm_n(S,B,\alpha,\gamma,\delta)$ in \Cref{def:wcm} except that we replace \eqref{eq:wcm-assumption} with the stronger assumption
\[
\text{there exists $s\in S$ such that ${\Pr}_{(x,y)\sim\mu}[s(x) = y] = 1$}.
\]
\end{definition}
In comparison, we sometimes refer to the learning task we study in \Cref{sec:cm} as \emph{strong} correlation maximization. Our learners in \Cref{sec:cm} can also be used to solve weak correlation maximization:
\begin{claim}
\label{claim:wcm-cm}
Let $S,B\subseteq([-1,1]\cup\{*\})^X$ be hypothesis classes. Assume $\alpha,\delta\in\bR_{\ge 0},\gamma\in [0,\alpha]$ and $n\in \bZ_{\ge 0}$, then
$\cm_n(S,B,\alpha-\gamma,\delta) \subseteq \wcm_n(S,B,\alpha,\gamma,\delta)$. The same relationship holds when we replace $\cm$ and $\wcm$ with $\dcm$ and $\wdcm$.
\end{claim}
The claim follows directly from the definitions of $\cm$ (\Cref{def:cm}), $\dcm$ (\Cref{def:dcm}), $\wcm$ (\Cref{def:wcm}) and $\wdcm$ (\Cref{def:wdcm}).

We design a learner (\Cref{alg:ma-mc}) for multiaccuracy and multicalibration that invokes learners for weak correlation maximization. Instead of minimizing $\mae$ and $\mce$ directly, \Cref{alg:ma-mc} minimizes the \emph{$\blambda$-multicalibration error} which we define as follows. 
For a positive integer $k$, let us consider a partition of $[-1,1]$ into intervals $\blambda = (\Lambda_i)_{i = 1,\ldots,k}$. 
Following the definition of $\mce$ in \eqref{eq:mc-def},
we define the \emph{$\blambda$-multicalibration error} of a model $f:X\rightarrow[-1,1]$ w.r.t.\ a distribution $\mu$ over $X\times[-1,1]$ and a hypothesis class $B\subseteq([-1,1]\cup\{*\})^X$ to be
\begin{align}
\mce\sps \blambda_{\mu,B}(f) & := \sup_{b\in B}\sum_{i=1}^k\sup_{\sigma\in \{-1,1\}}\bE_{(x,y)\sim\mu}\Big[\Big((f(x) - y)\one(f(x)\in \Lambda_i)\sigma\Big)\ptimes b(x)\Big]\label{eq:mcee-1}\\
& = \sup_{b\in B}\sup_{\bsigma\in \{-1,1\}^k}\sum_{i=1}^k\bE_{(x,y)\sim\mu}\Big[\Big((f(x) - y)\one(f(x)\in \Lambda_i)\sigma_i\Big)\ptimes b(x)\Big],\notag
\end{align}
where the inner supremum in the last line is over $\bsigma := (\sigma_1,\ldots,\sigma_k)\in \{-1,1\}^k$.
It is clear that $\mae_{\mu,B}(f) = \mce_{\mu,B}\sps \blambda(f)$ when $k = 1$ and $\Lambda_1 = [-1,1]$. The following claim allows us to also relate $\mce$ to $\mce\sps \blambda$:
\begin{claim}
\label{claim:mce-mcee}
Let $\blambda = (\Lambda_i)_{i = 1,\ldots,k}$ be a partition of the interval $[-1,1]$ where $\Lambda_1 = [-1, -1 + 2/k]$ and $\Lambda_i = (-1 + (2i - 2)/k, -1 + 2i/k]$ for $i = 2,\ldots,k$. Given a model $f:X\to[-1,1]$, we define a model $f':X\to [-1,1]$ such that $f'(x) = -1 + (2i - 1)/k$ whenever $f(x)\in \Lambda_i$. Then,
\[
\mce_{\mu,B}(f') \le \mce_{\mu,B}\sps \blambda (f) + 1/k.
\]
\end{claim}
\begin{proof}
By the definition of $\mce$ in \eqref{eq:mc-def},
\begin{equation}
\label{eq:mce-mcee-1}
\mce_{\mu,B}(f') = \sup_{b\in B}\sum_{i=1}^k\sup_{\sigma\in \{-1,1\}}\bE_{(x,y)\sim\mu}\Big[\Big((f'(x) - y)\one(f'(x) = -1 + (2i - 1)/k)\sigma\Big)\ptimes b(x)\Big].
\end{equation}
For every $b\in B,i\in \{1,\ldots,k\},\sigma\in\{-1,1\}$ and $(x,y)\in X\times [-1,1]$,
\begin{align*}
& \Big((f'(x) - y)\one(f'(x) = -1 + (2i - 1)/k)\sigma\Big)\ptimes b(x)\\
= {} & \Big((f'(x) - y)\one(f(x)\in \Lambda_i)\sigma\Big)\ptimes b(x)\\
\le {} & \Big((f(x) - y)\one(f(x)\in \Lambda_i)\sigma\Big)\ptimes b(x) + |f(x) - f'(x)|\one(f(x)\in \Lambda_i)\\
\le {} & \Big((f(x) - y)\one(f(x)\in \Lambda_i)\sigma\Big)\ptimes b(x) + \frac 1k\one(f(x)\in\Lambda_i).
\end{align*}
Plugging this into \eqref{eq:mce-mcee-1},
\begin{align*}
\mce_{\mu,B}(f') & \le \sup_{b\in B}\sum_{i=1}^k\sup_{\sigma\in \{-1,1\}}\bE_{(x,y)\sim\mu}\Big[\Big((f(x) - y)\one(f(x)\in \Lambda_i)\sigma\Big)\ptimes b(x)\Big] + \frac 1k\\
& = \mce_{\mu,B}\sps \blambda(f) + 1/k.\qedhere
\end{align*}
\end{proof}
With $\mae$ and $\mce$ both related to $\mce\sps \blambda$, we can focus on showing that \Cref{alg:ma-mc} achieves a low $\mce\sps \blambda$. The learners in \citep{hebert2018multicalibration} also aim for a low error that is similar to $\mce\sps \blambda$, but the sum over $i = 1,\ldots,k$ in \eqref{eq:mcee-1} is replaced by a supremum and thus a factor of $k$ would be lost in \Cref{claim:mce-mcee}. When multiaccuracy is our end goal, there is no difference because we choose $k = 1$. For multicalibration, we choose to define $\mce\sps \blambda$ as in \eqref{eq:mcee-1} to tradeoff time efficiency for sample efficiency: we use \eqref{eq:mcee-1} to achieve a better sample complexity upper bound in \Cref{thm:ma-mc} (in terms of the dependency on $\beta$ and $\varepsilon$), and as a consequence our learner (\Cref{alg:ma-mc}) has running time exponential in $k$ because of an enumeration procedure over $\bsigma\in\{-1,1\}^k$ (we choose $k = \Theta(1/\beta)$ when proving \Cref{thm:ma-mc}).

To analyze \Cref{alg:ma-mc}, we first reformulate the definition of $\mce\sps \blambda$ in a more convenient way for the analysis.
We fix a partition $\blambda = (\Lambda_i)_{i=1,\ldots,k}$ of the interval $[-1,1]$.
Given $\bsigma := (\sigma_1,\ldots,\sigma_k)\in \{-1,1\}^k$ and $u\in [-1,1]$, we define $\chi_\bsigma:[-1,1]\rightarrow\{-1,1\}$ such that $\chi_\bsigma(u) = \sigma_j$ when $u\in \Lambda_j$.
This allows us to rewrite $\mce\sps \blambda$ as follows:
\begin{align*}
\mce_{\mu,B}\sps \blambda(f) & = \sup_{b\in B}\sup_{\bsigma\in \{-1,1\}^k}\bE\Big[\Big((f(x) - y)\chi_\bsigma(f(x))\Big)\ptimes b(x)\Big]\\
& = \sup_{b\in B}\sup_{\bsigma\in \{-1,1\}^k}\bE\Big[(f(x) - y)\ptimes \Big(\chi_\bsigma(f(x))b(x)\Big)\Big],
\end{align*}
where we use the convention that $\chi_\bsigma(f(x))b(x) = *$ whenever $b(x) = *$. 

Let $B_{\bsigma,f}\subseteq([-1,1]\cup\{*\})^X$ denote the class of all hypotheses $h:X\to[-1,1]\cup\{*\}$ such that there exists $b\in B$ satisfying $h(x) = \chi_{\bsigma}(f(x))b(x)$ for every $x\in X$.
We can now simplify the definition of $\mce\sps \blambda$ further:
\begin{equation}
\label{eq:mcee-simplified}
\mce_{\mu,B}\sps \blambda(f) = {\sup}_{\bsigma\in\{-1,1\}^k}{\sup}_{b\in B_{\bsigma,f}} \bE\Big[(f(x) - y)\ptimes b(x)\Big].
\end{equation}

Given a hypothesis class $S\subseteq([-1,1]\cup\{*\})^X$ and a total function $f:X\rightarrow \bR$, we define $(S-f)/2$ to be the class consisting of all hypotheses $h:X\to [-1,1]\cup\{*\}$ such that there exists $s\in S$ satisfying
\[
h(x) = \begin{cases}
*, & \text{if } s(x) = *,\\
(s(x) - f(x))/2, & \text{otherwise},
\end{cases}
\quad \text{for every } x\in X.
\]
\Cref{alg:ma-mc} invokes learners $L$ solving weak correlation maximization for $((S - f)/2, B_{\bsigma,f})$ for various $f:X\to[-1,1]$. To bound the number of data points needed by $L$, we prove the following claim controlling the mutual VC dimension of binary hypothesis classes generated from $(S - f)/2$ and $B_{\bsigma,f}$.
\begin{restatable}{claim}{claimmamcvcbound}
\label{claim:ma-mc-vc-bound}
Define $\tilde S := ((S - f)/2)$ and $\tilde B := B_{\bsigma,f}$ as above. Then for $\eta_1,\eta_2\in\bR_{\ge 0}$ and $r_1:X\to \bR$,
\[
{\sup}_{\theta\in \bR}\vc(\tilde S_{\eta_1}\sps {r_1},\tilde B_{\eta_2}\sps {\theta}) \le 2{\sup}_{\theta'\in\bR}\vc(S_{2\eta_1}\sps {2r_1 + f},B_{\eta_2}\sps {\theta'}).
\]
\end{restatable}
We defer the relatively straightforward proof of the claim to \Cref{sec:proof-ma-mc-vc-bound}. 

Our learner (\Cref{alg:ma-mc}) uses the following definition: for any $u\in \bR$, we define $\proj_{[-1,1]}(u)$ to be the projection of $u$ into the interval $[-1,1]$, i.e., $\proj_{[-1,1]}(u) = \max\{-1, \min\{1, u\}\}$.
The following lemma shows that \Cref{alg:ma-mc} indeed achieves a low $\mce\sps \blambda$ with large probability when taking sufficiently many input data points:
\begin{lemma}
\label{lm:ma-mc}
Let $S,B\subseteq([-1,1]\cup\{*\})^X$ be real-valued hypothesis classes. 
Suppose the parameters of \Cref{alg:ma-mc} satisfy $\alpha,\gamma\in \bR_{>0},\delta\in (0,1/2), W > 4/\gamma^2$, 
\begin{align}
 n\sps 1 & \ge \sup_{f:X\rightarrow[-1,1]}\sup_{\bsigma\in \{-1,1\}^k}\swcm((S - f)/2, B_{\bsigma,f},\alpha/2, \gamma/2, \delta/(2W)),\label{eq:ma-mc-assumption-n-1}\\
 n\sps 2 & \ge C\gamma^{-2}(k + \log(W/\delta))\notag
\end{align}
for a sufficiently absolute large absolute constant $C>0$. 
Also, suppose the input data points to \Cref{alg:ma-mc} are drawn i.i.d.\ from a distribution $\mu$ over $X\times[-1,1]$ satisfying ${\Pr}_{x\sim\mu|_X}[s(x)\ne *] = 1$ and $\bE_{(x,y)\sim\mu}[y|x] = s(x)$ for some $s\in S$.
Then with probability at least $1-\delta$, the output model $f$ of \Cref{alg:ma-mc} satisfies
\[
\mce_{\mu,B}\sps \blambda(f) \le \alpha.
\]
If we additionally assume that $S$ is binary, i.e., $S\subseteq\{-1,1,*\}^X$, then we can replace $\wcm$ with $\wdcm$ in \eqref{eq:ma-mc-assumption-n-1} and in \Cref{line:ma-mc-invoke} of \Cref{alg:ma-mc}, and the lemma still holds.
\end{lemma}
\LinesNumbered
\begin{algorithm}
\caption{Multiaccuracy/multicalibration for $(S,B)$}
\label{alg:ma-mc}
\SetKwInOut{Parameters}{Parameters}
\SetKwInOut{Input}{Input}
\SetKwInOut{Output}{Output}
\Parameters{$S,B\subseteq([-1,1]\cup\{*\})^X$, $n,n\sps 1,n\sps 2, W\in \bZ_{> 0}$ satisfying $n = W(n\sps 1 + n\sps 2)$, $\alpha,\gamma,\delta \in \bR_{\ge 0}$, a partition $\blambda = (\Lambda_1,\ldots,\Lambda_k)$ of $[-1,1]$.}
\Input{data points $(x_1,y_1),\ldots,(x_n,y_n)\in X\times[-1,1]$.}
\Output{model $f:X\rightarrow [-1,1]$.}
Partition the input data points into $2W$ datasets: $\Psi\sps {j,1} = \left(\left(x_i\sps {j,1}, y_i\sps {j,1}\right)\right)_{i=1}^{n\sps 1}$ and $\Psi \sps {j,2} = \left(\left(x_i\sps {j,2}, y_i\sps {j,2}\right)\right)_{i=1}^{n\sps 2}$ for $j = 1,\ldots,W$\;
Initialize $f:X\rightarrow[-1,1]$ to be the constant zero function: $f(x) = 0$ for every $x\in X$\;
\For{$j = 1,\ldots,W$\label{line:ma-mc-for-1}}
{
Define $\tilde y_i = (y_i\sps{j,1} - f(x_i\sps{j,1}))/2$ for $i = 1,\ldots, n\sps 1$ and define $\Psi' = ((x_i\sps{j,1},\tilde y_i))_{i=1}^{n\sps 1}$\label{line:ma-mc-1}\;
\For{$\bsigma\in \{-1,1\}^k$}
{
Invoke learner $L\in \wcm_{n\sps 1}((S - f)/2,B_{\bsigma,f},\alpha/2,\gamma/2, \delta/(2W))$ on $\Psi'$ to obtain $f_\bsigma$\label{line:ma-mc-invoke}\;
}
Choose $f'$ from $\{f_\bsigma:\bsigma\in \{-1,1\}^k\}$ that maximizes $Q_{f'}:=\frac{1}{n\sps 2}\sum_{i=1}^{n\sps 2}(y_i\sps {j,2}  - f(x_i\sps{j,2}))f'(x_i\sps {j,2})$\label{line:ma-mc-4}\;
\eIf{$Q_{f'} \ge 3\gamma/4$\label{line:ma-mc-if}}
{Update $f(x)$ to $\proj_{[-1,1]}(f(x) + \gamma f'(x)/2)$ for every $x\in X$\label{line:ma-mc-2}\;}
{\Break\label{line:ma-mc-3}\;}
}\label{line:ma-mc-for-2}
\Return $f$\label{line:ma-mc-return}\;
\end{algorithm}
We first prove \Cref{thm:ma-mc,thm:ma-mc-binary} using \Cref{lm:ma-mc} and then prove \Cref{lm:ma-mc} afterwards.
\begin{proof}[Proof of \Cref{thm:ma-mc}]
We choose $\gamma = \beta/3,k = \lceil \frac{3}{\beta}\rceil, \alpha = 2\beta/3 + 2\eta_1 + 4\eta_2$ and $W = \lfloor 4/\gamma^2\rfloor + 1 = O(1/\beta^2)$ in \Cref{lm:ma-mc}. 
We also choose $\blambda$ as in \Cref{claim:mce-mcee}.
Define $\eta_1' = \eta_1/2$. For $\bsigma\in \{-1,1\}^k$ and $f:X\to [-1,1]$, define $\tilde S = ((S - f)/2)$ and $\tilde B = B_{\bsigma,f}$. By \Cref{claim:ma-mc-vc-bound}, for every $\theta\in \bR$, $\vc(\tilde S_{\eta_1'}\sps 0, \tilde B_{\eta_2}\sps \theta)\le 2m$.
Therefore,
\begin{align}
& \swcm((S - f)/2, B_{\bsigma,f},\alpha/2, \gamma/2, \delta/(2W))\nonumber \\
\le {} & \scm((S - f)/2, B_{\bsigma,f},(\alpha - \gamma)/2, \delta/(2W)) \tag{by \Cref{claim:wcm-cm}}\\
= {} & \scm((S - f)/2, B_{\bsigma,f},\beta/6 + 2\eta_1' + 2\eta_2, \delta/(2W))\nonumber \\
\le {} & O\left(
\frac{m}{\beta^4}\log^2_+\left(\frac{m}{\beta}\right)\log\left(\frac 1{\eta_1}\right) + \frac 1{\beta^4}\log\left(\frac 1{\eta_1}\right)\log\left(\frac 1{\beta\delta}\right) + \frac{1}{\beta^2}\log\left(\frac 1{\eta_2}\right) \right), \label{eq:ma-mc-wcm-1}
\end{align}
where the last inequality is by \Cref{thm:cm}.
This means that the requirement of \Cref{lm:ma-mc} can be satisfied by
\begin{align*}
n\sps 1 & \le O\left(
\frac{m}{\beta^4}\log^2_+\left(\frac{m}{\beta}\right)\log\left(\frac 1{\eta_1}\right) + \frac 1{\beta^4}\log\left(\frac 1{\eta_1}\right)\log\left(\frac 1{\beta\delta}\right) + \frac{1}{\beta^2}\log\left(\frac 1{\eta_2}\right) \right), \text{ and}\\
n\sps 2 & \le O\left(\frac 1{\beta^3} + \frac 1{\beta^2}\log\left(\frac 1{\beta\delta}\right)\right).
\end{align*}
By \Cref{claim:mce-mcee}, the guarantee $\mce_{\mu,B}\sps \blambda(f)\le \alpha$ of \Cref{lm:ma-mc} implies that the output model $f$ of \Cref{alg:ma-mc} can be easily transformed to $f'$ satisfying $\mae_{\mu,B}(f')\le \mce_{\mu,B}(f') \le \alpha + 1/k \le \beta + 2\eta_1 + 4\eta_2$. 
Since \Cref{alg:ma-mc} takes $n = W(n\sps 1 + n\sps 2)$ data points,
we have
\begin{align*}
& \smc(S,B,\beta + 2\eta_1 + 4\eta_2,\delta)\\
\le {} & W(n\sps 1 + n\sps 2)\\
\le {} & O\left(
\frac{m}{\beta^6}\log^2_+\left(\frac{m}{\beta}\right)\log\left(\frac 1{\eta_1}\right) + \frac 1{\beta^6}\log\left(\frac 1{\eta_1}\right)\log\left(\frac 1{\beta\delta}\right) + \frac{1}{\beta^4}\log\left(\frac 1{\eta_2}\right)   \right).
\qedhere
\end{align*}
\end{proof}
\begin{proof}[Proof of \Cref{thm:ma-mc-binary}]
We first prove the upper bound on $\smc$. We set $\eta_1 = 0$ and define $\gamma,k,\alpha,W$ in the same way as in the proof of \Cref{thm:ma-mc}. We apply \Cref{thm:dcm-real} to get
\begin{align}
& \swdcm((S - f)/2, B_{\bsigma,f},\alpha/2, \gamma/2, \delta/(2W)) \nonumber \\
\le{} & O\left(\frac{m}{\beta^2}\log^2_+\left(\frac{m}{\beta}\right)  + \frac{1}{\beta^2}\log\left(\frac 1{\eta_2\beta\delta}\right)\right).\label{eq:ma-mc-binary-1}
\end{align}
By our assumption that $S$ is binary, \Cref{lm:ma-mc} allows us to replace $\wcm$ with $\wdcm$ in \eqref{eq:ma-mc-assumption-n-1}. We can thus use \eqref{eq:ma-mc-binary-1} in place of \eqref{eq:ma-mc-wcm-1} and follow the rest of the proof of \Cref{thm:ma-mc}. We omit the details. The upper bound on $\sma$ can be proved similarly, except that we set $k = 1$ instead.
\end{proof}

Before we prove \Cref{lm:ma-mc}, we define a few events based on the execution of \Cref{alg:ma-mc}. Given a positive integer $j$, we use $E_j\sps 1$ to denote the bad event in which both of the following occur:
\begin{enumerate}
\item Before \Cref{line:ma-mc-1} is executed in the $j$-th iteration of the outer \for loop (Lines~\ref{line:ma-mc-for-1}-\ref{line:ma-mc-for-2}), there exists $b\in \bigcup_{\bsigma\in\{-1,1\}^k}B_{\bsigma,f}$ such that
\begin{equation}
\label{eq:bad-1}
\bE_{(x,y)\sim\mu}[(y - f(x))\ptimes b(x)] > \alpha.
\end{equation}
\item When \Cref{line:ma-mc-4} is executed in the $j$-th iteration of the outer \for loop, for all $f'\in \{f_\bsigma:\bsigma\in \{-1,1\}^k\}$ it holds that
\begin{equation}
\label{eq:bad-2}
\bE_{(x,y)\sim\mu}[(y - f(x))f'(x)] < \gamma.
\end{equation}
\end{enumerate}
We use $E_j\sps 2$ to denote the bad event that before Line~\ref{line:ma-mc-4} is executed in the $j$-th iteration of the outer \for loop, there exists $f'\in \{f_\bsigma:\bsigma\in \{-1,1\}^k\}$ such that
\[
|Q_{f'} - \bE_{(x,y)\sim\mu}[(y - f(x))f'(x)]| > \gamma/4.
\]
We have the following claim showing that each bad event only happens with small probability:
\begin{claim}
\label{claim:ma-mc-bad}
In the setting of \Cref{lm:ma-mc},
for every positive integer $j$, $\max\{\Pr[E_j\sps 1],\Pr[E_j\sps 2]\} \le \delta/(2W)$.
\end{claim}
\begin{proof}
Let us focus on the $j$-th iteration of the outer \for loop (Lines~\ref{line:ma-mc-for-1}-\ref{line:ma-mc-for-2}).
Consider the model $f$ immediately before \Cref{line:ma-mc-1} and define $\mu'$ to be the distribution of $(x,(y - f(x))/2)$ for $(x,y)\sim\mu$.
Let us condition on the event that
\eqref{eq:bad-1} holds for some $b\in B_{\bsigma,f}$ with $\bsigma\in \{-1,1\}^k$. By the definition of $\mu'$, \eqref{eq:bad-1} implies
\[
\bE_{(x,y)\sim\mu'}[y\ptimes b(x)] > \alpha/2.
\]
It is clear that after Line~\ref{line:ma-mc-1}, every data point in $\Psi'$ distributes i.i.d.\ from $\mu'$ which satisfies $\bE_{(x,y)\sim\mu'}[y|x] = (s(x) - f(x))/2$, and when $S$ is binary we additionally have $\Pr_{(x,y)\sim\mu'}[(s(x) - f(x))/2 = y] = 1$.
By our assumption \eqref{eq:ma-mc-assumption-n-1}, we have $\wcm_{n\sps 1}((S - f)/2,B_{\bsigma,f},\alpha/2,\gamma/2, \delta/(2W))\ne \emptyset$ at \Cref{line:ma-mc-invoke} (or $\wdcm_{n\sps 1}((S - f)/2,B_{\bsigma,f},\alpha/2,\gamma/2, \delta/(2W))\ne \emptyset$ if $S$  is binary).
By the guarantee of the learner $L\in \wcm_{n\sps 1}((S - f)/2,B_{\bsigma,f},\alpha/2,\gamma/2, \delta/(2W))$ (or $L\in \wdcm_{n\sps 1}((S - f)/2,B_{\bsigma,f},\alpha/2,\gamma/2,\allowbreak \delta/(2W))$ if $S$ is binary), with probability at least $1-\delta/(2W)$, the following holds before \Cref{line:ma-mc-4}:
\[
\bE_{(x,y)\sim\mu'}[yf_\bsigma(x)] \ge \gamma/2,
\]
which implies that \eqref{eq:bad-2} does not hold for $f' = f_\bsigma$. This proves $\Pr[E_j\sps 1]\le \delta/(2W)$.

The claim $\Pr[E_j\sps 2] \le \delta/(2W)$ follows from the Chernoff bound, the union bound, and our assumption that $n\sps 2\ge C\gamma^{-2}(k + \log(W/\delta))$ for a sufficiently large absolute constant $C>0$.
\end{proof}
Now we prove the following lemma showing that \Cref{alg:ma-mc} makes progress towards a good model in each iteration when the bad events do not happen.
\begin{lemma}
\label{lm:ma-mc-good}
In the setting of \Cref{lm:ma-mc},
for a positive integer $j$, assume that neither bad event $E_j\sps 1$ or $E_j\sps 2$ happens. Then at least one of the following three good events $G_j\sps 1,G_j\sps 2,G_j\sps 3$ happens:
\begin{enumerate}
\item $G_j\sps 1$: the outer \for loop (Lines~\ref{line:ma-mc-for-1}-\ref{line:ma-mc-for-2}) has fewer than $j$ iterations;
\item $G_j\sps 2$: $\bE_{(x,y)\sim\mu}[(f(x) - y)^2]$ decreases by at least $\gamma^2/4$ at Line~\ref{line:ma-mc-2} in the $j$-th iteration of the outer \for loop;
\item $G_j\sps 3$: Line~\ref{line:ma-mc-3} is executed in the $j$-th iteration (which must be the last iteration) of the outer \for loop and \Cref{alg:ma-mc} outputs a model $f$ satisfying
\begin{equation}
\label{eq:ma-mc-good}
\bE_{(x,y)\sim\mu}[(y - f(x))\ptimes b(x)] \le \alpha \text{ for every }b\in {\bigcup}_{\bsigma\in \{-1,1\}^k}B_{\bsigma,f}.
\end{equation}
\end{enumerate}
\end{lemma}
\begin{proof}
We first show that if the outer \for loop has at least $j$ iterations and Line~\ref{line:ma-mc-2} is executed in the $j$-th iteration, then $\bE_{(x,y)\sim\mu}[(f(x) - y)^2]$ decreases by at least $\gamma^2/4$ at Line~\ref{line:ma-mc-2}. Indeed, we have $\bE_{(x,y)\sim\mu}[(y - f(x))f'(x)] \ge Q_{f'} - \gamma/4 \ge \gamma/2$ before Line~\ref{line:ma-mc-2}, where the first inequality holds because $E_j\sps 2$ does not happen by our assumption, and the second inequality holds because Line~\ref{line:ma-mc-2} is executed only when the \bif condition at \Cref{line:ma-mc-if} is satisfied. Therefore,
\begin{align*}
& \bE_{(x,y)\sim\mu}[(f(x) - y)^2] - \bE_{(x,y)\sim\mu}\Big[\Big(\proj_{[-1,1]}(f(x) + \gamma f'(x)/2) - y\Big)^2\Big]\\
\ge {} & \bE_{(x,y)\sim\mu}[(f(x) - y)^2] - \bE_{(x,y)\sim\mu}\Big[\Big((f(x) + \gamma f'(x)/2) - y\Big)^2\Big]\\
= {} & \gamma\bE_{(x,y)\sim\mu}[(y - f(x))f'(x)]  - (\gamma/2)^2\bE_{(x,y)\sim\mu}[(f'(x))^2]\\
\ge {} & \gamma^2/2 - \gamma^2/4\\
= {} & \gamma^2/4.
\end{align*}

It remains to show that if inequality \eqref{eq:ma-mc-good} is not satisfied by the model $f$ before \Cref{line:ma-mc-1} in the $j$-th iteration, then \Cref{line:ma-mc-2} is executed. Indeed, since event $E_j\sps 1$ does not happen by our assumption, before \Cref{line:ma-mc-4}, there exists $f'\in \{f_\bsigma:\bsigma\in \{-1,1\}^k\}$ such that
\[
\bE_{(x,y)\sim\mu}[(y - f(x))f'(x)] \ge \gamma.
\]
This implies that $Q_{f'} \ge 3\gamma/4$ since $E_j\sps 2$ does not happen. Therefore, the \bif condition at \Cref{line:ma-mc-if} is satisfied and \Cref{line:ma-mc-2} is executed, as desired.
\end{proof}
\begin{proof}[Proof of \Cref{lm:ma-mc}]
For any positive integer $j$, by \Cref{claim:ma-mc-bad} and \Cref{lm:ma-mc-good}, with probability at least $1-\delta/W$, at least one of the three good events in \Cref{lm:ma-mc-good} happens.
By the union bound, with probability at least $1-\delta$, for every $j = 1,\ldots,W$, at least one of the three good events happens.
Note that the event $G_j\sps 2$ cannot happen for $W > 4/\gamma^2$ different $j$'s because $\bE_{(x,y)\sim\mu}[(f(x) - y)^2]$ is initially at most $1$ and always nonnegative. 
Also, for every $j = 1,\ldots,W$, the event $G_j\sps 1$ cannot happen unless the event $G_{j'}\sps 3$ happens for a $j'< j$.
Therefore, with probability at least $1-\delta$, $G_j\sps 3$ must happen for the last iteration $j$ of the outer \for loop, in which case the model $f$ returned at Line~\ref{line:ma-mc-return} satisfies $\mce_{\mu,B}\sps \blambda(f)\le \alpha$ by \eqref{eq:mcee-simplified}.
\end{proof}
\subsection{Lower Bound}
\label{sec:ma-mc-lower}
Now we prove our sample complexity lower bound (\Cref{thm:ma-mc-lower}) for realizable multiaccuracy and muticalibration. In \citep[Lemma 8]{hu2022metric}, the authors show a lower bound using a standard packing argument in the distribution-specific setting with total hypothesis classes. We state their result below and then transfer their result to our distribution-free setting with partial hypothesis classes.

We first define realizable multiaccuracy and multicalibration in the distribution-specific setting:
\begin{definition}[Distribution-specific realizable multiaccuracy ($\ma\sps {\mu_X}$)/multicalibration ($\mc\sps{\mu_X}$)]
\label{def:ma-mc-specific}
Given two total hypothesis classes $S,B\subseteq [-1,1]^X$, a distribution $\mu_X$ over $X$, an error bound $\varepsilon \ge 0$, a failure probability bound $\delta \ge 0$, and a nonnegative integer $n$, we define $\ma_n\sps{\mu_X}(S,B,\varepsilon,\delta)$ in the same way as we define $\ma_n(S,B,\varepsilon,\delta)$ in \Cref{def:ma} except that we additionally require distributions $\mu\in \dstr$ to satisfy $\mu|_X = \mu_X$. Similarly, we define $\mc_n\sps{\mu_X}(S,B,\varepsilon,\delta)$ by adding an additional distribution assumption to $\mc_n(S,B,\varepsilon,\delta)$ (\Cref{def:mc}).
\end{definition}

The lower bound of \citet{hu2022metric} is in terms of packing and covering numbers defined for pairs of total hypothesis classes. Let $S,B\subseteq[-1,1]^X$ be total hypothesis classes and let $\mu_X$ be a distribution over $X$.  For $\varepsilon \ge 0$, the packing number $M_{\mu_X,B}(S,\varepsilon)$ is defined to be the maximum size of a subset $S'\subseteq S$ satisfying $|\bE_{x\sim\mu_X}[(s_1(x) - s_2(x))b(x)]|> \varepsilon$ for every distinct $s_1,s_2\in S'$.
The covering number $N_{\mu_X,B}(S,\varepsilon)$ is defined to be the minimum size of a subset $S'\subseteq S$ such that for every $s\in S$ there exists $s'\in S'$ satisfying $|\bE_{x\sim\mu_X}[(s(x) - s'(x))b(x)]|\le \varepsilon$. The following claim is a standard relationship between packing and covering numbers \citep[see e.g.\ the proof of][Lemma 8]{hu2022metric}:
\begin{claim}
\label{claim:packing-covering}
Let $\mu_X$ be a distribution over $X$.
For any $\varepsilon\ge 0$ and $S,B\subseteq[-1,1]^X$,
\[
N_{\mu_X,B}(S,\varepsilon) \le M_{\mu_X,B}(S,\varepsilon).
\]
\end{claim}
\citet{hu2022metric} prove the following lower bound on the sample complexity of distribution-specific realizable multiaccuracy:
\begin{theorem}[{\citep[Lemma 8]{hu2022metric}}]
\label{thm:ma-lower-specific}
Let $S,B\subseteq[-1,1]^X$ be total hypothesis classes and let $\mu_X$ be a distribution over $X$. For $\varepsilon \in \bR_{>0}$ and $\delta\in (0,1)$,
\[
\sma\sps{\mu_X}(S,B,\varepsilon,\delta) \ge \log((1-\delta)M_{\mu_X,B}(S,2\varepsilon)).
\]
\end{theorem}
The version in
\citet{hu2022metric} is in terms of the covering number $N_{\mu_X,B}$ rather than the packing number $M_{\mu_X,B}$ in \Cref{thm:ma-lower-specific}, but their proof works in both cases. Below we prove \Cref{thm:ma-mc-lower} using \Cref{thm:ma-lower-specific}.

\begin{proof}[Proof of \Cref{thm:ma-mc-lower}]
For concreteness, we prove the theorem with the absolute constant $c$ being $3$. The proof is still valid if we replace every constant $3$ in it with an arbitrary absolute constant $c > 2$.

Let $X'\subseteq X$ be a finite subset shattered by both $S_{\eta_1}\sps {r_1}$ and $B_{\eta_2}\sps {r_2}$. Since we define $m$ to be $\sup_{r_1,r_2:X\to \bR}\vc(S_{\eta_1}\sps {r_1}, B_{\eta_2}\sps {r_2})$, we can choose the size of $X'$ to be $m$ when $m$ is finite, or to be arbitrarily large when $m$ is infinite. Let $\unif_{X'}$ be the uniform distribution over $X'$. Define $\tilde S$ to be the class of all total hypotheses $\tilde s:X'\to [-1,1]$ such that there exists $s\in S$ satisfying $\tilde s(x) = s(x)$ for every $x\in X'$, and we define $\tilde B$ similarly. It is clear that
\[
\smc(S,B,\eta_1\eta_2/3,\delta)  \ge \sma(S,B,\eta_1\eta_2/3,\delta) \ge \sma\sps{\unif_{X'}}(\tilde S,\tilde B,\eta_1\eta_2/3,\delta),
\]
so it suffices to show that
\begin{equation}
\label{eq:ma-lower-goal}
\sma\sps{\unif_{X'}}(\tilde S,\tilde B,\eta_1\eta_2/3,\delta) \ge \log(1-\delta) + \Omega(|X'|).
\end{equation}

By \Cref{lm:gv}, there exists $U\subseteq\{-1,1\}^{X'}$ such that $|U|\ge 2^{\Omega(|X'|)}$ and for every distinct $u_1,u_2\in U$,
\begin{equation}
\label{eq:ma-lower-3}
{\Pr}_{x\sim\unif_{X'}}[u_1(x)\ne u_2(x)] \ge 1/3.
\end{equation}
Since $X'$ is shattered by $S_{\eta_1}\sps{r_1}$, $X'$ must also be shattered by $\tilde S_{\eta_1}\sps{r_1}$, and thus for every $u\in U$, there exists $s_u\in \tilde S$ such that
\[
u(x)(s_u(x) - r_1(x)) > \eta_1 \textnormal{ for every }x\in X'.
\]
This implies that if $u_1,u_2\in U$ satisfy $u_1(x)\ne u_2(x)$ for some $x\in X'$, then $|s_{u_1}(x) - s_{u_2}(x)| > 2\eta_1$. Combining this with \eqref{eq:ma-lower-3}, for every distinct $u_1,u_2\in U$,
\[
\bE_{x\sim\unif_{X'}}|s_{u_1}(x) - s_{u_2}(x)| > 2\eta_1/3.
\]
Defining $S' := \{s_u:u\in U\} \subseteq \tilde S$, we have $|S'| = |U| \ge 2^{\Omega(|X'|)}$, and for every distinct $s_1,s_2\in S'$, 
\begin{equation}
\label{eq:ma-lower-2}
\bE_{x\sim\unif_{X'}}|s_1(x) - s_2(x)| > 2\eta_1/3.
\end{equation}

Since $X'$ is shattered by $B_{\eta_2}\sps{r_2}$, it must also be shattered by $\tilde B_{\eta_2}\sps {r_2}$. Thus for every distinct $s_1,s_2\in S'$, there exist $b_1,b_2\in \tilde B$ such that for every $x\in X'$,
\begin{align*}
&\sign(s_{1}(x) - s_{2}(x))(b_1(x) - r_2(x))  > \eta_2,\\
-&\sign(s_{1}(x) - s_{2}(x))(b_2(x) - r_2(x))  > \eta_2.
\end{align*}
This implies that for every $x\in X'$,
\[
\sign(s_{1}(x) - s_{2}(x))(b_1(x) - b_2(x)) > 2\eta_2.
\]
Combining this with \eqref{eq:ma-lower-2}, we have
\begin{align*}
& |\bE_{x\sim\unif_{X'}}[(s_{1}(x) - s_{2}(x))b_1(x)]| + |\bE_{x\sim\unif_{X'}}[(s_{1}(x) - s_{2}(x))b_2(x)]|\\
\ge {} & \bE_{x\sim\unif_{X'}}[(s_{1}(x) - s_{2}(x))(b_1(x) - b_2(x))]\\
= {} & \bE_{x\sim\unif_{X'}}[|s_{1}(x) - s_{2}(x)|\sign(s_1(x) - s_1(x))(b_1(x) - b_2(x))] \\
> {} & 4\eta_1\eta_2/3,
\end{align*}
and thus
\[
{\sup}_{b\in \tilde B}|\bE_{x\sim\unif_{X'}}[(s_{1}(x) - s_{2}(x))b(x)]| > 2\eta_1\eta_2/3.
\]
Since the above holds for every distinct $s_1,s_2\in S'$ with $S'\subseteq \tilde S$ satisfying $|S'| \ge 2^{\Omega(|X'|)}$, we have $M_{\unif_{X'},\tilde B}(\tilde S,2\eta_1\eta_2/3) \ge 2^{\Omega(|X'|)}$. Our goal \eqref{eq:ma-lower-goal} then follows from \Cref{thm:ma-lower-specific}.
\end{proof}
\begin{remark}
\label{remark:covering}
It is a classic result that the covering number of a binary hypothesis class $S\subseteq \{-1,1\}^X$ can be upper bounded in terms of its VC dimension. Formally, fixing $B = [-1,1]^X$, for $\varepsilon\in (0,1/2)$ and any distribution $\mu_X$ over $X$, the following inequality holds (\citep[see e.g.][Theorem 8.3.18]{MR3837109}):
\[
\log N_{\mu_X,B}(S,\varepsilon) \le O(\vc(S)\log(1/\varepsilon)).
\]
Similar results have been proved for real-valued hypothesis classes $S$ as well using the fat-shattering dimension \citep{MR1328428,MR1629694,MR2042042}.
In these results, the hypothesis class $B$ is fixed to be $[-1,1]^X$, and thus the covering number is w.r.t.\ the $\ell_1$ metric:\footnote{Some previous results also consider other metrics such as the $\ell_\infty$ metric and the $\ell_2$ metric.}
\[
\sup_{b\in B}|\bE_{x\sim\mu_X}[(s_1(x) - s_2(x))b(x)]| = \bE_{x\sim\mu_X}|s_1(x) - s_2(x)|.
\]

Combining \Cref{thm:ma-lower-specific} and \Cref{thm:ma-mc}, we can generalize these existing covering number upper bounds to hold for arbitrary $B \subseteq [-1, 1]^X$, rather than just $B = [-1, 1]^X$.
We state these generalizations below as upper bounds on the packing number $M_{\mu_X,B}(S,\varepsilon)$, and these bounds also hold for the covering number $N_{\mu_X,B}(S,\varepsilon)$ by \Cref{claim:packing-covering}.

For any $\eta_1,\eta_2,\beta\in (0,1/2)$, defining 
$
m:= \sup_{r:X\to \bR}\sup_{\theta\in\bR}\vc(S_{\eta_1}\sps r, B_{\eta_2}\sps \theta),
$
we have
\begin{align*}
& \log M_{\mu_X,B}(S,2\beta + 4\eta_1 + 8\eta_2) 
\\
\le {} & O(\sma(S,B,\beta + 2\eta_1 + 4\eta_2,1/2) + 1)\tag{by \Cref{thm:ma-lower-specific}}\\
\le {} & O\left(
\frac{m}{\beta^6}\log^2_+\left(\frac{m}{\beta}\right)\log\left(\frac 1{\eta_1}\right) + \frac 1{\beta^6}\log\left(\frac 1{\eta_1}\right)\log\left(\frac 1{\beta}\right) + \frac{1}{\beta^4}\log\left(\frac 1{\eta_2}\right)   \right).\tag{by \Cref{thm:ma-mc}}
\end{align*}
In particular, for $\varepsilon\in (0,1/2)$, choosing $\beta = \eta_1 = \eta_2 = \varepsilon/14$, we have 
$m \le \fat_{\varepsilon/14}(S,B)$ and
\begin{align*}
\log M_{\mu_X,B}(S,\varepsilon) \le O \left( 
	\frac{m}{\varepsilon^6}\log^2_+\left(
		\frac{m}{\varepsilon}
	\right)
	\log\left(\frac 1\varepsilon\right) + \frac{1}{\varepsilon^6}\log^2\left(\frac 1\varepsilon\right)
\right).
\end{align*}
When $S$ is binary, we define $m:= \sup_{\theta\in \bR}\vc(S,B_{\eta_2}\sps\theta)$ and get
\begin{align*}
\log M_{\mu_X,B}(S,2\beta  + 8\eta_2) 
& \le  O(\sma(S,B,\beta + 4\eta_2,1/2) + 1)\tag{by \Cref{thm:ma-lower-specific}}\\
& \le  O\left(\frac{m}{\beta^4}\log^2_+\left(\frac{m}{\beta}\right)  + \frac{1}{\beta^4}\log\left(\frac 1{\eta_2\beta\delta}\right)\right).\tag{by \Cref{thm:ma-mc-binary}}
\end{align*}
For $\varepsilon\in (0,1/2)$, choosing $\beta = \eta_2 = \varepsilon/10$ in the inequality above, we have $m \le \fat_{\varepsilon/10}(S,B)$ and
\[
\log M_{\mu_X,B}(S,\varepsilon) 
 \le O\left(\frac{m}{\varepsilon^4}\log^2_+\left(\frac{m}{\varepsilon}\right)  + \frac{1}{\beta^4}\log\left(\frac 1{\varepsilon\delta}\right)\right).
\]
When $S$ and $B$ are both binary, the above inequality holds with $m = \vc(S,B)$. 

The above covering/packing number upper bounds hold for any pair of classes $S,B\subseteq[-1,1]^X$, but they do not imply a uniform convergence bound for the multiaccuracy error: \citet[Section 6.2]{hu2022metric} give an example showing that the sample complexity of \emph{agnostic} multiaccuracy cannot in general be upper bounded in terms of the mutual fat-shattering dimension.
\end{remark}
\section{Boosting}
\label{sec:boosting}
So far, we have ignored computational efficiency when designing learners for comparative learning tasks. In this section, we consider running time in addition to sample complexity and present an 
efficient \emph{boosting} algorithm that solves comparative learning given oracle access to a \emph{weak} comparative learner. 

Below we formally define \emph{weak comparative learning} ($\wcomp$) in a similar fashion
to the \emph{weak agnostic learning} task studied in \citep{MR2582918,DBLP:conf/innovations/Feldman10}. In comparison, we sometimes refer to the task $\comp$ in \Cref{def:comp} as \emph{strong} comparative learning.

\begin{definition}[Weak comparative learning ($\wcomp$)]
\label{def:wcomp}
Given two binary hypothesis classes $S,B\subseteq\{-1,1,*\}^X$, parameters $\alpha,\gamma \ge 0$, and a nonnegative integer $n$, we define $\wcomp_n(S,B,\alpha,\allowbreak \gamma,\delta)$ to be $\learn_n(Z,F,\dstr,(F_\mu)_{\mu\in \dstr},\delta)$ with $Z,F,\dstr,F_\mu$ chosen as follows. We choose $Z = X\times\{-1,1\}$ and $F = \{-1,1\}^X$. The distribution class $\dstr$ consists of all distributions $\mu$ over $X\times \{-1,1\}$ such that $\Pr_{(x,y)\sim\mu}[s(x) = y] = 1$ for some $s\in S$ and
\[
{\inf}_{b\in B}{\Pr}_{(x,y)\sim\mu}[b(x)\ne y] \le 1/2 - \alpha.
\]
The admissible set $F_\mu$ consists of all models $f:X\rightarrow\{-1,1\}$ such that 
\[
{\Pr}_{(x,y)\sim\mu}[f(x)\ne y] \le 1/2 - \gamma.
\]
\end{definition}

When $S$ and $B$ are both total, we can use \citep[Theorem 3.5]{DBLP:conf/innovations/Feldman10} to get an efficient boosting algorithm that solves strong comparative learning ($\comp$) using an oracle for weak comparative learning ($\wcomp$).
In this section, we generalize this result to \emph{partial} and \emph{real-valued} hypothesis classes $S$ and $B$. That is, we focus on the more general tasks: weak and strong deterministic-label correlation maximization ($\dcm$ and $\wdcm$).
By \eqref{eq:error-correlation}, $\comp$ (\Cref{def:comp}) is a special case of $\dcm$ (\Cref{def:dcm}) where $S$ and $B$ are binary, and similarly $\wcomp$ (\Cref{def:wcomp}) is a special case of $\wdcm$ (\Cref{def:wdcm}). 

Assuming oracle access to a learner solving weak deterministic-label correlation maximization ($\wdcm$) for a pair of hypothesis classes $S,B\subseteq ([-1,1]\cup\{*\})^X$, 
we show an efficient boosting algorithm that solves strong deterministic-label correlation maximization ($\dcm$) for the same classes $(S,B)$ in \Cref{thm:boosting} below.

Because we require our boosting algorithm to be efficient, we cannot expect it to output a model $f:X\to \{-1,1\}$ explicitly because the size of the domain $X$ may be large or even infinite. Instead, the algorithm outputs a succinct description of the model $f$ from which the value $f(x)$ can be computed efficiently given any $x\in X$. Formally, we define the \emph{evaluation time} of a description of a model $f$ to be the worst-case time needed to compute $f(x)$ given $x\in X$ using the model's description (for example, if $f$ is described as a circuit, then the evaluation time corresponds to the circuit size).

\begin{theorem}
\label{thm:boosting}
Let $S,B\subseteq([-1,1]\cup\{*\})^X$ be partial hypothesis classes. Suppose $\alpha,\gamma,\varepsilon,\delta_1,\delta_2,\delta_3\in (0,1/2)$ and $n_0\in \bZ_{\ge 0}$.
Suppose $L$ is a learner in $\wdcm_{n_0}(S,B,\alpha,\gamma,\delta_1)$ that always represents its output model using a description with evaluation time at most $T_{\mathsf{eval}}$. Then there exist $W,W',n\in\bZ_{\ge 0}$ and a learner $L'$ in $\dcm_n(S,B,\alpha + \varepsilon,W'(\delta_1 + \delta_2) + W\delta_3)$ that invokes $L$ at most $W'$ times and has additional running time $O(W'T_{\mathsf{eval}}n)$ where
\begin{align*}
W' & = O(\gamma^{-2}\alpha^{-1}\log(1/\alpha)),\\
W & = W' + O(\varepsilon^{-2}),\\
n & = O(W'(n_0/\alpha + \alpha^{-2}\gamma^{-2}\log(1/\delta_2)) + W\varepsilon^{-2}\log(1/\delta_3) ).
\end{align*}
Also, $L'$ always represents its output model using a description with evaluation time $O(W'T_{\mathsf{eval}})$.
\end{theorem}

We prove \Cref{thm:boosting} in the rest of the section.
The main idea comes from the observation that our learners in \Cref{sec:ma-mc-upper} for realizable multiaccuracy and multicalibration only require \emph{weak} learners for correlation maximization. 
The idea of using multiaccuracy or multicalibration for boosting dates back at least to \citet{DBLP:conf/innovations/Feldman10} and the idea has recently been further explored by \citet{DBLP:conf/innovations/GopalanKRSW22,gopalan2022loss}. Although \citet{DBLP:conf/innovations/Feldman10} lacked the terminology of ``multiaccuacy'' and ``multicalibration,'' the author showed that any model achieving a low $\mae$ w.r.t.\ a total binary hypothesis class $B$ and achieving a low \emph{sign-calibration error} must also achieve the goal of (strong) agnostic learning for $B$ when composed with the $\sign$ function. Here, for a distribution $\mu$ over $X\times \bR$ and a model $f:X\to \bR$, we define the sign-calibration error to be:
\[
\sce_\mu(f) := |\bE_{(x,y)\sim\mu}[(y - f(x))\sign(f(x))]|.
\]
The sign-calibration error is upper bounded by the \emph{overall calibration error} $\ce$ we define later in \eqref{eq:ce} in \Cref{sec:compr}, and $\ce$ is a special case of $\mce$.
Below we generalize the result by \citet{DBLP:conf/innovations/Feldman10} to partial and real-valued hypotheses:

\begin{lemma}
\label{lm:ma-cm}
Let $\mu$ be a distribution over $X\times \bR$. Let $f_1:X\rightarrow\bR$ and $b: X \rightarrow [-1,1]\cup \{*\}$ be total/partial functions. Assume
\begin{equation}
\label{eq:ma-cal-1}
\bE_{(x,y)\sim\mu}[(y - f_1(x))\ptimes b(x)] \le \alpha,
\end{equation}
and 
\begin{equation}
\label{eq:ma-cal-2}
\bE_{(x,y)\sim\mu}[(f_1(x) - y)\sign(f_1(x))] \le \varepsilon.
\end{equation}
Then,
\[
\bE_{(x,y)\sim\mu}[y\,\sign (f_1(x))] \ge \bE_{(x,y)\sim\mu}[y\ptimes b(x)] - \alpha - \varepsilon.
\]
\end{lemma}
In \Cref{lm:ma-cm}, \eqref{eq:ma-cal-2} is a weaker assumption than $f_1$ having a low $\sce$ because \eqref{eq:ma-cal-2} does not take the absolute value of its left-hand-side. 
Similarly, for a hypothesis class $B\subseteq([-1,1]\cup\{*\})^X$, the requirement that \eqref{eq:ma-cal-1} holds for every $b\in B$ is weaker than the requirement that $\mae_{\mu,B}(f)\le \alpha$.
If $f_1$ satisfies \eqref{eq:ma-cal-1} for every $b\in B$ and $f_1$ additionally satisfies \eqref{eq:ma-cal-2}, then the conclusion of \Cref{lm:ma-cm} holds for every $b\in B$, implying that the composition $\sign\circ f_1$ satisfies the goal of strong correlation maximization:
\[
\bE_{(x,y)\sim\mu}[y\,\sign (f_1(x))] \ge {\sup}_{b\in B}\bE_{(x,y)\sim\mu}[y\ptimes b(x)] - \alpha - \varepsilon.
\]
\begin{proof}[Proof of \Cref{lm:ma-cm}]
By \eqref{eq:ma-cal-2},
\begin{equation}
\label{eq:ma-cal-3}
\bE_{(x,y)\sim\mu}[y\,\sign (f_1(x))] \ge \bE_{x\sim\mu|_X}[f_1(x)\sign(f_1(x))] - \varepsilon = \bE_{x\sim\mu|_X}[|f_1(x)|] - \varepsilon.
\end{equation}
For every $x\in X$, it is easy to check that the following inequality holds regardless of whether $b(x) = *$:
\[
|f_1(x)| \ge y \ptimes b(x) - (y - f_1(x))\ptimes b(x).
\]
Plugging this into \eqref{eq:ma-cal-3} and using \eqref{eq:ma-cal-1},
\begin{align*}
\bE_{(x,y)\sim\mu}[y\,\sign (f_1(x))] & \ge \bE_{(x,y)\sim\mu}[y\, \ptimes b(x)] - \bE_{(x,y)\sim\mu}[(y - f_1(x))\ptimes b(x)] - \varepsilon \\
& \ge \bE_{(x,y)\sim\mu}[y\, \ptimes b(x)] - \alpha - \varepsilon.
\qedhere
\end{align*}
\end{proof}
\begin{remark}
\label{remark:ma-cal}
It is clear from its proof that \Cref{lm:ma-cm} still holds if we replace $\sign(f_1(x))$ by some $f_2(x)$ as long as $f_2(x) = \sign(f_1(x))$ whenever $f_1(x)\ne 0$.
\end{remark}

Using \Cref{lm:ma-cm}, the goal of our boosting algorithm becomes to achieve a low $\mae$ and a low $\sce$ given oracle access to a learner solving $\wdcm$.
We achieve this following the same idea in our learner (\Cref{alg:ma-mc}) for $\ma$/$\mc$ in \Cref{sec:ma-mc-upper}.
A challenge is that \Cref{alg:ma-mc} invokes learners in $\wcm_n((S - f)/2, B_{\bsigma,f},\varepsilon',\delta')$ for various choices of $f,\bsigma$, but we only have oracle access to a learner in $\wcm_n(S, B,\varepsilon',\delta')$ for a single pair $(S,B)$.
The main challenge is the difference between $(S - f)/2$ and $S$.
When $S$ is a binary hypothesis class, 
the challenge can be solved using the rejection sampling procedure we used in \Cref{sec:cm}, but for a general real-valued $S$, we need an additional adjustment: instead of searching for a model achieving a low $\mae$ and a low $\sce$, we search for a model achieving these low errors \emph{after a projection transformation that depends on the source hypothesis $s$}.

Specifically,
for real numbers $y,u\in\bR$, define $\p(y,u)\in\bR$ to be the projection of $u$ into the interval $[0,y]$ or $[y,0]$ (depending on whether $y$ or $0$ is larger) as follows:
\[
\p (y,u) = 
\begin{cases}
\min\{0,y\}, & \textnormal{if } u < \min\{0,y\};\\
\max\{0,y\}, & \textnormal{if } u > \max\{0,y\};\\
u, & \textnormal{otherwise}.
\end{cases}
\]
Let $\mu$ be a distribution over $X\times[-1,1]$ such that $\Pr_{(x,y)\sim\mu}[s(x) = y] = 1$ for some $s\in S$.
For a model $f:X\rightarrow[-1,1]$, we define a \emph{projected model} $f_1:X\rightarrow [-1,1]$ such that $f_1(x) = \p(s(x),f(x))$. In our boosting algorithm, we search for a model $f$ that would make $f_1$ achieve a low $\mae$ and a low $\sce$.
Then by \Cref{lm:ma-cm}, the model $\sign\circ f_1$ would be the desired output. Although the definition of $f_1$ depends on the unknown source hypothesis $s\in S$, we can still (effectively) output $\sign\circ f_1$ because for every $x\in X$ satisfying $f_1(x)\ne 0$, it holds that $\sign(f_1(x)) = \sign(f(x))$, and thus \Cref{lm:ma-cm} still holds with $\sign(f_1(x))$ replaced by $\sign(f(x))$ (see \Cref{remark:ma-cal}). Therefore, our boosting algorithm outputs $\sign\circ f$ as a surrogate for $\sign\circ f_1$. 

We present our boosting algorithm in \Cref{alg:boosting} and analyze it in the following lemma, of which \Cref{thm:boosting} is a direct corollary (after replacing both $\delta_2$ and $\delta_4$ in \Cref{lm:boosting} with $\delta_2/2$):
\begin{lemma}
\label{lm:boosting}
Let $S,B\subseteq([-1,1]\cup\{*\})^X$ be partial hypothesis classes. For $\alpha,\gamma,\delta_1\in (0,1/2)$ and $n_0\in\bZ_{\ge 0}$, let $L$ be a learner in $\wdcm_{n_0}(S,B,\alpha,\gamma,\delta_1)$ that always represents its output model $f$ using a description with evaluation time at most $T_{\mathsf{eval}}$. Assume that the parameters of \Cref{alg:boosting} satisfy $\varepsilon\in (0,1/2), W' > C\alpha^{-1}\gamma^{-2}\log(1/\alpha),W > W' + 4/\varepsilon^2$, $n\sps 1\ge 2n_0/\alpha, n\sps 1 \ge C\alpha^{-1}\log(1/\delta_4),n\sps 2\ge C\alpha^{-2}\gamma^{-2}\log(1/\delta_2), n\sps 3 \ge C\varepsilon^{-2}\log(1/\delta_3)$ for $\delta_2,\delta_3,\delta_4\in (0,1/2)$ and a sufficiently large absolute constant $C > 0$. Then given oracle access to $L$, \Cref{alg:boosting} belongs to $\dcm_n(S,B,\alpha + \varepsilon,W'(\delta_1 + \delta_2 + \delta_4) + W\delta_3)$ where $n = W'(n\sps 1 + n\sps 2) + Wn\sps 3$. Moreover, \Cref{alg:boosting} invokes $L$ at most $W'$ times and has additional running time $O(W'T_{\mathsf{eval}}n)$, and its output model is represented using a description  with evaluation time $O(W'T_{\mathsf{eval}})$.
\end{lemma}
\begin{algorithm}
\caption{Boosting via MA + sign calibration}
\label{alg:boosting}
\SetKwInOut{Parameters}{Parameters}
\SetKwInOut{Input}{Input}
\SetKwInOut{Output}{Output}
\Parameters{$S,B\subseteq([-1,1]\cup\{*\})^X$, $\alpha,\gamma,\varepsilon \in \bR_{\ge 0}$, $n,n\sps 1,n\sps 2, n\sps 3, W,W'\in \bZ_{> 0}$ satisfying $n = W'(n\sps 1 + n\sps 2) + W n\sps 3$, oracle access to learner $L\in \wdcm_{n_0}(S,B,\alpha,\gamma,\delta_1)$ with $n_0\in \bZ_{\ge 0}$ and $\delta_1\in \bR_{\ge 0}$.}
\Input{data points $(x_1,y_1),\ldots,(x_n,y_n)\in X\times[-1,1]$.}
\Output{model $f:X\rightarrow\{-1,1\}$.}
Partition the data points into $2W' + W$ datasets: $\Psi\sps {j',1} = \left(\left(x_i\sps {j',1}, y_i\sps {j',1}\right)\right)_{i=1}^{n\sps 1}$, $\Psi\sps {j',2} = \left(\left(x_i\sps {j',2}, y_i\sps {j',2}\right)\right)_{i=1}^{n\sps 2}$ for $j' = 1,\ldots,W'$, and $\Psi\sps {j,3} = \left(\left(x_i\sps {j,3}, y_i\sps {j,3}\right)\right)_{i=1}^{n\sps 3}$ for $j = 1,\ldots,W$\;
$(j,j')\gets (1,1)$\;
Initialize $f:X\rightarrow[-1,1]$ to be the constant zero function: $f(x) = 0$ for every $x\in X$\;
\While{$j\le W$ and $j' \le W'$}
{
\eIf{$Q:=\frac{1}{n\sps 3}\sum_{i=1}^{n\sps 3}\left(y_i\sps {j,3} - \p\left(y_i\sps{j,3},f(x_i\sps{j,3})\right)\right) \sign(f(x_i\sps {j,3})) < -3\varepsilon/4$\label{line:boosting-test-3}}
{
Update $f(x)$ to $\proj_{[-1,1]}(f(x) - \varepsilon\, \sign(f(x))/2)$ for every $x\in X$\label{line:boosting-progress-1}\;
}
{
Initialize $\Psi$ to be the empty dataset\;
\For{$i = 1,\ldots,n\sps 1$\label{line:boosting-rejection-1}}
{
With probability $\frac{|y_i\sps {j',1} - \p(y_i\sps{j',1},f(x_i\sps{j',1}))|}{|y_i\sps{j',1}|}$, add the data point $(x\sps {j',1}_i,y\sps {j',1}_i)$ to $\Psi$\tcc*{We use the convention that $\frac{|y_i\sps {j',1} - \p(y_i\sps{j',1},f(x_i\sps{j',1}))|}{|y_i\sps{j',1}|} = 0$ if $y_i\sps{j',1} = 0$.}
}\label{line:boosting-rejection-2}
Let $n'$ be the number of data points in $\Psi$\label{line:boosting-n'}\;
\If{$n' < n\sps 1\alpha/2$\label{line:boosting-if-1}}{\Break\label{line:boosting-break-1}\;}
Invoke $L$ on the first $n_0$ data points in $\Psi$ to get $f'$\label{line:boosting-invoke}\;
$Q_{f'}\gets\frac{1}{n\sps 2}\sum_{i=1}^{n\sps 2}\left(y_i\sps {j',2} - \p\left(y_i\sps{j',2},f(x_i\sps{j',2})\right)\right) f'(x_i\sps {j',2})$\label{line:boosting-test-2}\;
\eIf{$Q_{f'} \ge 4\gamma n'/(9n\sps 1)$\label{line:boosting-if-2}}
{Update $f(x)$ to $\proj_{[-1,1]}(f(x) + Q_{f'} f'(x)/2)$ for every $x\in X$\label{line:boosting-progress-2}\;
}
{\Break\label{line:boosting-break-2}\;}
$j'\gets j' + 1$\;
}
$j\gets j + 1$\;
}
\Return $\sign\circ f$\label{line:boosting-return}\;
\end{algorithm}
To prove \Cref{lm:boosting}, we need to define a potential function for the model $f$ updated throughout the algorithm.
For $y,u\in \bR$, define $\varphi(y,u):= \int_y^u(\p (y,t) - y)\diff t$. In other words, 
\begin{align*}
\text{ for } y\ge 0,\quad \varphi(y,u) = \begin{cases}
\frac 12(y - u)^2, & \text{if } u\in [0, y],\\
0, &\text{if } u\in (y, +\infty),\\
\frac 12y^2 - yu, & \text{if }u\in (-\infty, 0),
\end{cases} \\
\text{ and for } y < 0,\quad  
\varphi(y,u) = \begin{cases}
\frac 12(y - u)^2, & \text{if } u\in [y, 0],\\
0, &\text{if } u\in (-\infty, y),\\
\frac 12y^2 - yu, & \text{if }u\in (0, +\infty).
\end{cases} 
\end{align*}
\begin{claim}
\label{claim:smooth}
For every $y,u,u'\in \bR$,
\[
\varphi(y,u') \le \varphi(y,u) + (\p(y,u) - y)(u' - u) + \frac 12(u' - u)^2.
\]
\end{claim}
\begin{proof}
By our definition $\varphi(y,u) := \int_y^u(\p (y,t) - y)\diff t$,
\begin{align*}
\varphi(y,u') - \varphi(y,u) & = \int_u^{u'}(\p(y,t) - y)\diff t\\
& = (\p(y,u) - y)(u' - u) + \int_u^{u'}(\p(y,t) - \p(y,u))\diff t\\
& \le (\p(y,u) - y)(u' - u) + \int_u^{u'}(t - u)\diff t\\
& \le (\p(y,u) - y)(u' - u) +  \frac 12 (u' - u)^2.\qedhere
\end{align*}
\end{proof}

Now we fix a distribution $\mu$ over $X\times[-1,1]$ such that $\Pr_{(x,y)\sim\mu}[s(x) = y] = 1$ for some $s\in S$.
For every $f:X\to [-1,1]$, we define its potential function to be 
\[
\Phi(f) = \bE_{(x,y)\sim\mu}[\varphi(y,f(x))].
\]
We also define $\rho(f) = \bE_{(x,y)\sim\mu}[|y - \p(y,f(x))|/|y|]$. Note that when $y = 0$, it holds that $y - \p(y,f(x)) = 0$, in which case we use the convention that $|y - \p(y,f(x))|/|y| = 0$.

\begin{claim}
\label{claim:Q>=P}
For $f:X\to [-1,1]$, let $\rho(f)$ and $\Phi(f)$ be defined as above. Then,
$\rho(f) \ge (2/3)\Phi(f)$.
\end{claim}
\begin{proof}
It suffices to prove that $|y - \p(y,u)|\ge (2/3) \varphi(y,u)$ for every $y,u\in [-1,1]$. We prove this only for $y \ge 0$ because the other case $y < 0$ can be handled similarly. If $u \ge y$, we have $\varphi(y,u) = 0$ and the inequality holds trivially. If $u\in [0,y]$, the inequality is equivalent to $(y - u)\ge (2/3)(1/2)(y - u)^2$, which holds because $0 \le y - u\le 3$. If $u \le 0$, the inequality is equivalent to $y \ge (2/3)(y^2/2 - yu)$, which can be verified easily by $u\ge -1$ and $y \ge y^2$. 
\end{proof}

We also define a few bad events based on the execution of \Cref{alg:boosting}. 
For positive integers $j$ and $j'$, we define bad events $E_{j'}\sps 1, E_{j'}\sps 2, E_j\sps 3, E_{j'}\sps 4$ as follows (we use different subscripts $j$ and $j'$ for different events in correspondence to the $j$ and $j'$ used in \Cref{alg:boosting}).
We use $E_{j'}\sps 1$ to denote the bad event that both of the following occur:
\begin{enumerate}
\item Immediately before \Cref{line:boosting-invoke} is executed for the $j'$-th time, there exists $b\in B$ such that 
\begin{equation}
\label{eq:boosting-bad-1}
\bE_{(x,y)\sim\mu}[(y - \p(y,f(x)))\ptimes b(x)]> \alpha.
\end{equation}
\item Immediately after \Cref{line:boosting-invoke} is executed for the $j'$-th time, it holds that
\begin{equation}
\label{eq:boosting-bad-2}
\bE_{(x,y)\sim\mu}[(y - \p(y,f(x)))f'(x)] < \gamma \rho(f).
\end{equation}
\end{enumerate}
We use $E_{j'}\sps 2$ to denote the bad event that immediately after \Cref{line:boosting-test-2} is executed for the $j'$-th time, it holds that
\[
|Q_{f'} - \bE_{(x,y)\sim\mu}[(y - \p(y,f(x)))f'(x)]| > \alpha\gamma/9.
\]
We use $E_j\sps 3$ to denote the bad event that when \Cref{line:boosting-test-3} is executed for the $j$-th time, it holds that
\[
|Q - \bE_{(x,y)\sim\mu}[(y - \p(y,f(x)))\sign(f(x))]| > \varepsilon/4.
\]
We use $E_{j'}\sps 4$ to denote the bad event that when \Cref{line:boosting-n'} is executed for the $j'$-th time, $\rho (f)\ge \alpha$ but either $n' < n\sps 1\rho(f)/2$ or $n' > 2n\sps 1\rho(f)$.
\begin{lemma}
\label{lm:boosting-bad}
In the setting of \Cref{lm:boosting}, assume that the input data points to \Cref{alg:dcm-real} are generated i.i.d.\ from a distribution $\mu$ satisfying $\Pr_{(x,y)\sim\mu}[s(x) = y] = 1$ for some $s\in S$. For every $j\in \{1,\ldots,W\}$ and every $j'\in \{1,\ldots,W'\}$, let the events $E_{j'}\sps 1,E_{j'}\sps 2, E_j\sps 3,E_{j'}\sps 4$ be defined as above. Then, $\Pr[E_j\sps 1] \le \delta_1, \Pr[E_j\sps 2] \le \delta_2, \Pr[E_j\sps 3] \le \delta_3, \Pr[E_j\sps 4] \le \delta_4$.
\end{lemma}
\begin{proof}
The claimed upper bounds on $\Pr[E_{j'}\sps 4], \Pr[E_{j'}\sps 2]$ and $\Pr[E_j\sps 3]$ follow from the (multiplicative and additive) Chernoff bound and our assumptions that
\[
n\sps 1 \ge C\alpha^{-1}\log(1/\delta_4),n\sps 2\ge C\alpha^{-2}\gamma^{-2}\log(1/\delta_2), n\sps 3 \ge C\varepsilon^{-2}\log(1/\delta_3)
\]
for a sufficiently large absolute constant $C > 0$.
It remains to prove that $\Pr[E_j\sps 1] \le \delta_1$. Let $\nu$ denote the distribution of $(x,y,u)$, where $(x,y)\sim\mu$ and $\Pr[u = 1|x,y] = |y - \p(y,f(x))|/|y|$. Let $\mu'$ denote the conditional distribution of $(x,y)$ given $u = 1$ where $(x,y,u)\sim\nu$. 
Note that $\Pr_{(x,y,u)\sim\nu}[u = 1] = \bE_{(x,y)\sim\mu}[|y - \pi(y,f(x))|/|y|] = \rho(f)$.
Assuming $\rho(f) > 0$, it is easy to verify that the distribution $\mu'$ satisfies $\Pr_{(x,y)\sim\mu'}[s(x) = y] = 1$,  and for every $b:X\to[-1,1]\cup\{*\}$,
\begin{align}
\bE_{(x,y)\sim\mu'}[y\ptimes b(x)] & = \frac{1}{\Pr_{(x,y,u)\sim\nu}[u = 1]}\bE_{(x,y)\sim\mu}[\Pr[u = 1|x,y](y\ptimes b(x))]\notag \\
& = \frac 1{\rho(f)}\bE_{(x,y)\sim\mu}[(y - \p(y,f(x)))\ptimes b(x)].\label{eq:boosting-new-dist}
\end{align}
Immediately before \Cref{line:boosting-invoke} is executed for the $j'$-th time,
the data points in $\Psi$ distribute i.i.d.\ from $\mu'$ by the rejection sampling procedure at Lines~\ref{line:boosting-rejection-1}-\ref{line:boosting-rejection-2}. Also, $\Psi$ contains at least $n_0$ data points because $n' \ge n\sps 1\alpha/2 \ge n_0$, where the first inequality holds because the \bif condition at \Cref{line:boosting-if-1} does not hold, and the second inequality holds by our assumption.
Therefore, if \eqref{eq:boosting-bad-1} holds, we have $\rho(f) \ge \bE_{(x,y)\sim\mu}[|y - \p(y,f(x))|] \ge \alpha$ and $\bE_{(x,y)\sim\mu'}[y\ptimes b(x)] > \alpha/\rho(f) \ge \alpha$, and thus by the guarantee of $L\in \wdcm_{n_0}(S,B,\alpha,\gamma,\delta_1)$, with probability at least $1 - \delta_1$,
\[
\bE_{(x,y)\sim\mu}[(y - \p(y,f(x)))f'(x)] = \rho(f) \cdot\bE_{(x,y)\sim\mu'}[yf'(x)] \ge \gamma \rho(f),
\]
violating \eqref{eq:boosting-bad-2}. This implies that $\Pr[E_{j'}\sps 1] \le \delta_1$.
\end{proof}
Since $j$ increases by $1$ in each iteration of the \while loop in \Cref{alg:boosting}, every iteration of the \while loop can be identified by the value of the pair $(j,j')$ at the beginning of the iteration. We thus refer to a specific iteration as the iteration corresponding to $(j,j')$ for $j\in \{1,\ldots,W\}$ and $j'\in \{1,\ldots,W'\}$.
\begin{lemma}
\label{lm:boosting-good}
In the setting of \Cref{lm:boosting-bad}, if none of $E_{j'}\sps 1,E_{j'}\sps 2, E_j\sps 3, E_{j'}\sps 4$ happens for a pair of positive integers $j\in \{1,\ldots,W\}$ and $j'\in \{1,\ldots,W'\}$, then in the iteration of the \while loop corresponding to $(j,j')$ (if exists), one of the following good events happens:
\begin{enumerate}
\item $G\sps 1_{j,j'}$: $\Phi(f)$ decreases by at least $\varepsilon^2/8$ at \Cref{line:boosting-progress-1};
\item $G\sps 2_{j,j'}$: $\Phi(f)$ decreases by at least $\max\{\gamma^2\alpha^2/162, \gamma^2 \rho(f)^2/162\}$ at \Cref{line:boosting-progress-2};
\item $G\sps 3_{j,j'}$: \Cref{line:boosting-break-1} or \Cref{line:boosting-break-2} is executed and the model $f$ at \Cref{line:boosting-return} satisfies
\begin{equation}
\label{eq:boosting-good-1}
\bE_{(x,y)\sim\mu}[(y - \p(y,f(x)))\ptimes b(x)]\le \alpha \text{ for every }b\in B
\end{equation}
and
\begin{equation}
\label{eq:boosting-good-2}
\bE_{(x,y)\sim\mu}[(y - \p(y,f(x)))\sign(f(x))]\ge -\varepsilon.
\end{equation}
\end{enumerate}
\end{lemma}

\begin{proof}
By our assumption, all of the complement events $\neg E_{j'}\sps 1,\neg E_{j'}\sps 2, \neg E_j\sps 3, \neg E_{j'}\sps 4$ happen.
We focus on the iteration of the \while loop corresponding to $(j,j')$ and
first show that if \Cref{line:boosting-progress-1} is executed, then $\Phi(f)$ decreases by at least $\varepsilon^2/8$. Indeed, by $\neg E_j\sps 3$, the model $f$ immediately before \Cref{line:boosting-progress-1} satisfies
\begin{equation}
\label{eq:boosting-good-proof-1}
\bE_{(x,y)\sim\mu}[(y - \p(y,f(x)))\sign(f(x))] < -\varepsilon/2.
\end{equation}
Now we have
\begin{align*}
& \bE_{(x,y)\sim\mu}\Big[\varphi\Big(y,\proj_{[-1,1]}\Big(f(x) - \varepsilon\,\sign(f(x))/2\Big)\Big)\Big] - \bE_{(x,y)\sim\mu}[\varphi(y,f(x))]\\
\le {} & \bE_{(x,y)\sim\mu}[\varphi(y,f(x) - \varepsilon\,\sign(f(x))/2)] - \bE_{(x,y)\sim\mu}[\varphi(y,f(x))]\\
\le {} & \bE_{(x,y)\sim\mu}\Big[\Big(\p(y,f(x)) - y\Big)\Big(- \varepsilon\,\sign(f(x))/2\Big)\Big] + \frac 12 \bE_{(x,y)\sim\mu}[(\varepsilon\,\sign(f(x))/2)^2]\tag{by \Cref{claim:smooth}}\\
\le {} & -\varepsilon^2/4 + \varepsilon^2/8 \tag{by \eqref{eq:boosting-good-proof-1}}\\
\le {} & -\varepsilon^2/8.
\end{align*}
This implies that $\Phi(f)$ decreases by at least $\varepsilon^2/8$ at \Cref{line:boosting-progress-1}.

Now we show that if \Cref{line:boosting-progress-2} is executed then $\Phi(f)$ decreases by at least $\max\{\alpha^2\gamma^2/162, \allowbreak \alpha^2\rho(f)^2/162\}$. Indeed, since the \bif condition at \Cref{line:boosting-if-1} is not satisfied, by $\neg E_j\sps 4$, we have $n'\ge \max\{\alpha/2,\rho(f)/2\}n\sps 1$. Therefore, by $\neg E_j\sps 2$, the functions $f$ and $f'$ immediately before \Cref{line:boosting-progress-2} satisfy
\begin{align}
& \bE_{(x,y)\sim\mu}[(y - \p(y,f(x)))f'(x)] > Q_{f'} - \alpha\gamma/9,\quad  \text{and}\label{eq:boosting-proof-2}\\
& Q_{f'} \ge 4n'\gamma/9n\sps 1 \ge \max\{2\alpha\gamma/9,2\rho(f)\gamma/9\}.\label{eq:boosting-proof-3}
\end{align}
Now we have
\begin{align*}
& \bE_{(x,y)\sim\mu}\Big[\varphi\Big(y,\proj_{[-1,1]}\Big(f(x) + Q_{f'}f'(x)/2\Big)\Big)\Big] - \bE_{(x,y)\sim\mu}[\varphi(y,f(x))]\\
\le {} & \bE_{(x,y)\sim\mu}\Big[\varphi\Big(y,f(x) + Q_{f'}f'(x)/2\Big)\Big] - \bE_{(x,y)\sim\mu}[\varphi(y,f(x))]\\
\le {} & \bE_{(x,y)\sim\mu}\Big[\Big(\p(y,f(x)) - y\Big)Q_{f'}f'(x)/2\Big] + \frac 12 \bE_{(x,y)\sim\mu}[(Q_{f'}f'(x)/2)^2]\tag{by \Cref{claim:smooth}}\\
\le {} &  - (Q_{f'}/2)(Q_{f'} - \alpha\gamma/9) + Q_{f'}^2/8 \tag{by \eqref{eq:boosting-proof-2}}\\
= {} & -(Q_{f'}/2)(3Q_{f'}/4  - \alpha\gamma/9)\\
\le {} & \max\{\alpha^2\gamma^2/162, \alpha^2\rho(f)^2/162\}.\tag{by \eqref{eq:boosting-proof-3}} 
\end{align*}
This implies that $\Phi(f)$ decreases by at least $\max\{\alpha^2\gamma^2/162, \alpha^2\rho(f)^2/162\}$ at \Cref{line:boosting-progress-2}.

It remains to show that if either \eqref{eq:boosting-good-1} or \eqref{eq:boosting-good-2} is violated, then either \Cref{line:boosting-progress-1} or \Cref{line:boosting-progress-2} is executed. Indeed, if \eqref{eq:boosting-good-2} is violated, then the \bif condition at \Cref{line:boosting-test-3} is satisfied by $\neg E_j\sps 3$, which implies that \Cref{line:boosting-progress-1} is executed. If \eqref{eq:boosting-good-1} is violated, then $\rho(f) \ge \bE_{(x,y)\sim\mu}|y - \p(y,f(x))| > \alpha$. By $\neg E_{j'}\sps 4$, we have $n' \ge n\rho(f)/2 \ge n\alpha/2$, which means that the \bif condition at \Cref{line:boosting-if-1} is not satisfied. By $\neg E_{j'}\sps 1$ and $\neg E_{j'}\sps 2$, we have
\[
Q_{f'} \ge \bE_{(x,y)\sim\mu}[(y - \p(y,f(x)))f'(x)] - \gamma\alpha/9 \ge \gamma \rho(f) - \gamma\alpha/9 \ge 8\gamma \rho(f)/9 \ge 4\gamma n'  / (9n\sps 1),
\]
which implies that the \bif condition at \Cref{line:boosting-if-2} is satisfied, and thus \Cref{line:boosting-progress-2} is executed.
The last inequality holds because $n'\le 2n\sps 1\rho(f)$ by $\neg E_{j'}\sps 4$. 
\end{proof}
\begin{proof}[Proof of \Cref{lm:boosting}]
By \Cref{lm:boosting-bad} and the union bound, none of the bad events $E_{j'}\sps 1,E_{j'}\sps 2$ and $E_j\sps 3$ happen for $j = 1,\ldots, W$ and $j' = 1,\ldots,W'$ with probability at least $1 - W'(\delta_1 + \delta_2 + \delta_4) - W\delta_3$. This means that all three good events in \Cref{lm:boosting-good} happen in every iteration of the \while loop. 
Note that $G\sps 1_{j,j'}$ cannot happen for more than $4/\varepsilon^2$ iterations because $\Phi(f)$ is initially at most $1/2$ and always nonnegative. 
Also, $G\sps 2_{j,j'}$ cannot happen for more than $C\alpha^{-1}\gamma^{-2}\log(1/\alpha)$ iterations for a sufficiently large absolute constant $C > 0$. This is because in each iteration where $G_{j,j'}\sps 2$ happens, $\Phi(f)$ decreases by at least $\gamma^2\rho(f)^2/162 \ge \gamma^2 \Phi(f)^2/400$ by \Cref{claim:Q>=P}. Consequently, if $\Phi(f)\ge \alpha$, $\Phi(f)$ decreases at least by a factor of $(1 - \gamma^2\alpha/400)$ in each such iteration. Therefore, after $O(\alpha^{-1}\gamma^{-2}\log(1/\alpha))$ such iterations, we have $\Phi(f)\le \alpha$. After that, $\Phi(f)$ decreases by at least $\alpha^2\gamma^2/162$ in each iteration where $G_{j,j'}\sps 2$ happens, so there can be at most $O(\alpha^{-1}\gamma^{-2})$ additional such iterations. In total, $G_{j,j'}\sps 2$ can happen in at most $O(\alpha^{-1}\gamma^{-2}O(1/\alpha)) + O(\alpha^{-1}\gamma^{-2}) = O(\alpha^{-1}\gamma^{-2}O(1/\alpha))$ iterations.

Therefore, by our choice of $W'> C\alpha^{-1}\gamma^{-2}\log(1/\alpha)$ and $W > W' + 4/\varepsilon$, with probability at least $1 - W'(\delta_1 + \delta_2 + \delta_4) - W\delta_3$, $G\sps 3_{j,j'}$ must happen in one of the iterations of the \while loop. This means that the model $f$ before \Cref{alg:boosting} returns satisfies \eqref{eq:boosting-good-1} and \eqref{eq:boosting-good-2}. 
In other words, if we define $f_1(x) = \pi(s(x),f(x))$ for every $x\in X$, then
\begin{align}
\bE_{(x,y)\sim\mu}[(y - f_1(x))\ptimes b(x)] & \le \alpha \text{ for every }b\in B, \text{ and}\label{eq:boosting-final-1}\\
\bE_{(x,y)\sim\mu}[(f_1(x) - y)\sign(f(x))] & \le \varepsilon.\label{eq:boosting-final-2}
\end{align}
Note that $\sign(f_1(x)) = \sign(f(x))$ whenever $f_1(x)\ne 0$.
By \Cref{lm:ma-cm} and \Cref{remark:ma-cal}, inequalities \eqref{eq:boosting-final-1} and \eqref{eq:boosting-final-2} imply
\[
\bE_{(x,y)\sim\mu}[y\,\sign(f(x))] \ge {\sup}_{b\in B}\bE_{(x,y)\sim\mu}[y\ptimes b(x)] - \alpha - \varepsilon.
\]
This proves that \Cref{alg:boosting} belongs to $\dcm_n(S,B,\alpha + \varepsilon,W'(\delta_1 + \delta_2 + \delta_4) + W\delta_3)$.

It remains to bound the running time of \Cref{alg:boosting} and the evaluation time of the output model $f$. We show that throughout the algorithm we can maintain a succinct description of $f$ with evaluation time at most $O(W'T_{\mathsf{eval}})$. To achieve this, we maintain a list initialized to be empty. Whenever we update $f$ at \Cref{line:boosting-progress-2}, we append the succinct description of $f'$ (obtained from $L$) and the value $Q_{f'}$ to the list. Whenever we update $f$ at \Cref{line:boosting-progress-1}, we append a special symbol $\bot$ to the list. We compress adjacent symbols $\bot$ using a single $\bot$ accompanied by an integer indicating the number of repetitions. The list contains all the information needed to describe $f$ throughout the algorithm, and it is easy to check that, given $x\in X$, we can always evaluate $f(x)$ using time linear in the product of $T_{\mathsf{eval}}$ and the length of the list. The length of the list is $O(W')$ because \Cref{line:boosting-progress-2} is executed for at most $W'$ times, so we can always evaluate $f(x)$ in time $O(W'T_{\mathsf{eval}})$. Since \Cref{alg:boosting} makes $O(n)$ such evaluations, the running time of \Cref{alg:boosting} (in addition to the $\le W'$ oracle calls) is $O(W'T_{\mathsf{eval}}n)$.
\end{proof}
\section{Comparative Regression via Omnipredictors}
\label{sec:compr}
We formally define the \emph{comparative regression} task where the learning objective is to minimize a general loss function. 

\begin{definition}[Comparative Regression ($\compr$)]
\label{def:compr}
Given a partial hypothesis class $S\subseteq([-1,1]\cup\{*\})^X$, a total hypothesis class $B\subseteq [-1,1]^X$, a loss function $\ell:\{-1,1\}\times[-1,1]\to \bR$, an error bound $\varepsilon \ge 0$, a failure probability bound $\delta \ge 0$, and a nonnegative integer $n$, we define $\compr_n(S,B,\ell,\varepsilon,\delta)$ to be $\learn_n(Z,F,\dstr,(F_\mu)_{\mu\in \dstr},\delta)$ with $Z,F,\dstr,F_\mu$ chosen as follows. We choose $Z = X\times\{-1,1\}$ and $F = [-1,1]^X$. The distribution class $\dstr$ consists of all distributions $\mu$ over $X\times \{-1,1\}$ satisfying the following property:
\begin{equation}
\label{eq:compr-assumption}
\text{there exists } s\in S \text{ such that }
{\Pr}_{x\sim\mu|_X}[s(x) \ne *] = 1 \text{ and } \bE_{(x,y)\sim\mu}[y|x] = s(x).
\end{equation}
The admissible set $F_\mu$ consists of all models $f:X\rightarrow[-1,1]$ such that 
\begin{equation}
\label{eq:compr-goal}
\bE_{(x,y)\sim\mu}[\ell(y,f(x))] \le {\inf}_{b\in B}\bE_{(x,y)\sim\mu}\bE[\ell(y,b(x))] + \varepsilon.
\end{equation}
\end{definition}

In the definition above, we assume that the benchmark class $B$ is total so that we do not need to define $\ell(y,b(x))$ when $b(x) = *$. We assume that the ranges of the model $f$ and any benchmark $b\in B$ are bounded between $-1$ and $1$, but any bounded range can be reduced to this setting by a scaling. We also assume that the label $y$ in a data point $(x,y)\sim\mu$ is binary: $y\in \{-1,1\}$, and thus \eqref{eq:compr-assumption} implies that the conditional distribution of $y$ given $x$ is $\ber^*(s(x))$. 
We focus on this binary-label setting because it is the main setting of the \emph{omnipredictor} result in \citep{DBLP:conf/innovations/GopalanKRSW22}, which our results are based on. 
There are certainly other natural and interesting settings of comparative regression (e.g.\ the deterministic-label setting where $y\in [-1,1]$ and $\Pr_{(x,y)\sim\mu}[s(x) = y] = 1$ for some $s\in S$). We leave further study of these settings for future work. %

In this section, we prove the following sample complexity upper bound for comparative regression in terms of the mutual fat-shattering dimension:
\begin{restatable}{theorem}{thmcompr}
\label{thm:compr}
Let $S\subseteq([-1,1]\cup\{*\})^X$ be a partial hypothesis class and $B\subseteq[-1,1]^X$ be a total hypothesis class. Let $\ell:\{-1,1\}\times[-1,1]\rightarrow\bR$ be a loss function such that $\ell(y,\cdot)$ is convex and $\kappa$-Lipschitz for any $y\in \{-1,1\}$. For $\beta,\eta_1,\eta_2,\delta\in (0,1/2)$, defining $m:= \sup_{r:X\to \bR}\sup_{\theta\in\bR}\vc(S_{\eta_1}\sps r, B_{\eta_2}\sps \theta)$,
\begin{align*}
& \scompr(S,B,\ell,\kappa(\beta + 2\eta_1 + 4\eta_2),\delta)\\
\le {} & O\left(
\frac{m}{\beta^6}\log^2_+\left(\frac{m}{\beta}\right)\log\left(\frac 1{\eta_1}\right) + \frac 1{\beta^6}\log\left(\frac 1{\eta_1}\right)\log\left(\frac 1{\beta\delta}\right) + \frac{1}{\beta^4}\log\left(\frac 1{\eta_2}\right)   \right).
\end{align*}
\end{restatable}
We prove \Cref{thm:compr} using the \emph{omnipredictor} result of \citet{DBLP:conf/innovations/GopalanKRSW22}, which shows that any model with a low $\mce$ w.r.t.\ $B$ and a low \emph{overall calibration} error can be easily transformed to a model that achieves a low loss compared to the best benchmark in $B$. 
Here, the overall calibration error of a model $f$ is defined as follows:
\begin{equation}
\label{eq:ce}
\ce_\mu(f) := \sum_{v\in V}|\bE_{(x,y)\sim\mu}[(y - f(x))\one(f(x) = v)]|,
\end{equation}
where $V$ is the range of $f$ which we require to be countable. The name ``overall calibration error'' comes from the fact that $\ce_\mu(f) = \mce_{\mu,B}(f)$ when $B$ only contains a single hypothesis $b$ such that $b(x) = 1$ for every $x\in X$.
\begin{restatable}[Omnipredictor \citep{DBLP:conf/innovations/GopalanKRSW22}]{theorem}{thmomni}
\label{thm:omni}
Let $\ell:\{-1,1\}\times[-1,1]\rightarrow\bR$ be a loss function such that $\ell(y,\cdot)$ is convex and $\kappa$-Lipschitz for any $y\in \{-1,1\}$.
Define $\tau:[-1,1]\rightarrow [-1,1]$ such that
\[
\tau(u)\in {\arg\min}_{q\in [-1,1]}\bE_{y\sim \ber^*(u)}[\ell(y,q)].
\]
Let $\mu$ be a distribution over $X\times\{-1,1\}$ and $B\subseteq [-1,1]^X$ be a total hypothesis class. Let $f:X\rightarrow[-1,1]$ be a model satisfying $\mce_{\mu,B}(f) \le \alpha$ and $\ce_{\mu}(f) \le \varepsilon$.
Then,
\[
\bE_{(x,y)\sim\mu}[\ell(y,\tau(f(x)))] \le \inf_{b\in B}\bE_{(x,y)\sim\mu}[\ell(y,b(x))] + (\alpha + 3\varepsilon)\kappa.
\]
\end{restatable}
\citet{DBLP:conf/innovations/GopalanKRSW22} only proved \Cref{thm:omni} in the special case where $\varepsilon = 0$. Since achieving $\ce_\mu(f) = 0$ is in general impossible with finitely many data points, it is important to prove \Cref{thm:omni} for a general $\varepsilon > 0$.
We include a proof of \Cref{thm:omni} in \Cref{sec:proof-omni}.

We prove \Cref{thm:compr} by combining \Cref{thm:omni} with our sample complexity upper bound for realizable multicalibration in \Cref{sec:ma-mc}. A challenge here is that in addition to achieving $\mce_{\mu,B}(f)\le \alpha$, \Cref{thm:omni} also requires us to achieve $\ce_\mu(f)\le\varepsilon$. This is similar to the situation in boosting (\Cref{sec:boosting}) where we need to simultaneously achieve a low $\mae$ and a low $\sce$. We include a more detailed proof of \Cref{thm:compr} in \Cref{sec:proof-compr}.

\section{Comparative Online Learning}
\label{sec:online}
In this section, we study comparative learning in the online setting. We show that the connections we make in \Cref{sec:comp} between comparative learning and learning partial hypotheses can be extended to the online setting, allowing us to show regret bounds for \emph{comparative online learning} in \Cref{thm:online}.

Specifically, in the online setting, the data points $(x,y)\in X\times\{-1,1\}$ are not given to the learner all at once. Instead, they come one-by-one and the learner sequentially makes predictions about the label of every individual $x$ before the true label $y$ is revealed. Additionally, the data points are \emph{not} assumed to be drawn i.i.d.\ from some distribution. Formally, a (possibly inefficient and randomized) online learner $L$ does the following on a stream of data points $(x_1,y_1),\ldots,(x_n,y_n)$: for every $i = 1,\ldots,n$, given $i-1$ labeled data points $(x_1,y_1),\ldots,(x_{i-1},y_{i-1})\in X\times\{-1,1\}$ and an extra unlabeled data point $x_i\in X$, the learner outputs a prediction $\hat y_i\in \{-1,1\}$. The performance of the learner $L$ is measured by
\[
\mistake(L;(x_i,y_i)_{i=1}^n):=\frac 1n\sum_{i=1}^n\Pr[\hat y_i \ne y_i],
\]
where the probability is over the internal randomness in $A$.
We also measure the performance of a hypothesis $h:X\to\{-1,1,*\}$ by
\[
\mistake(h;(x_i,y_i)_{i=1}^n):=\frac 1n\sum_{i=1}^n\one(h(x_i) \ne y_i).
\]

Researchers have studied online learning for partial binary hypothesis classes $H \subseteq \{-1, 1, *\}^X$ in the realizable and agnostic settings:
\begin{definition}[Realizable online learning]
Given a hypothesis class $H\subseteq\{-1,1,*\}^X$, a regret bound $\varepsilon \ge 0$, and a positive integer $n$, we use $\reao_n(H,\varepsilon)$ to denote the set of all online learners $L$ such that for every $h\in H$ and every sequence of data points $(x_1,y_1),\ldots,(x_n,y_n)\in X\times\{-1,1\}$ satisfying $y_i = h(x_i)$ for every $i = 1,\ldots,n$, it holds that
\[
\mistake(L;(x_i,y_i)_{i=1}^n)\le \varepsilon.
\]
\end{definition}

\begin{definition}[Agnostic online learning]
Given a hypothesis class $H\subseteq\{-1,1,*\}^X$, a regret bound $\varepsilon \ge 0$, and a positive integer $n$, we use $\agno_n(H,\varepsilon)$ to denote the set of all online learners $L$ such that for any sequence of data points $(x_1,y_1),\ldots,(x_n,y_n)\in X\times\{-1,1\}$, it holds that
\[
\mistake(L;(x_i,y_i)_{i=1}^n)\le \inf_{h\in H}\mistake(h;(x_i,y_i)_{i=1}^n) + \varepsilon.
\]
\end{definition}

We combine the realizable and agnostic settings to define comparative online learning:
\begin{definition}[Comparative online learning]
\label{def:compo}
Given hypothesis classes $S,B\subseteq\{-1,1,*\}^X$, a regret bound $\varepsilon \ge 0$, and a positive integer $n$, we use $\compo_n(S,B,\varepsilon)$ to denote the set of all online learners $L$ such that for every $s\in S$ and every sequence of data points $(x_1,y_1),\ldots,(x_n,y_n)\in X\times\{-1,1\}$ satisfying $y_i = s(x_i)$ for every $i = 1,\ldots,n$, it holds that
\[
\mistake(L;(x_i,y_i)_{i=1}^n)\le \inf_{b\in B}\mistake(b;(x_i,y_i)_{i=1}^n) + \varepsilon.
\]
\end{definition}

Analogous to the question of sample complexity, a basic question in online learning is to understand the optimal \emph{regret}, i.e., the minimum $\varepsilon$ for which there exists a learner that solves the tasks above given a sequence of $n$ data points. Given a total binary hypothesis class $H$, the optimal regret in realizable and agnostic online learning has been characterized by \citet{littlestone1988learning} and \citet{ben2009agnostic} using the \emph{Littlestone dimension}, and this characterization has been extended to partial hypothesis classes by \citet{MR4399723}. The Littlestone dimension of a partial hypothesis class $H\subseteq\{-1,1,*\}^X$ is defined as follows. Given $m\in \bZ_{\ge 0}$, suppose we associate an individual $x_\zeta\in X$ to every binary string $\zeta \in \cup_{i=0}^{m-1}\{-1,1\}^i$. There are $2^m - 1$ such strings $\zeta$ in total, so $(x_\zeta)_{\zeta\in \cup_{i=0}^{m-1}\{-1,1\}^i}\in X^{2^m - 1}$. We say $(x_\zeta)_{\zeta\in \cup_{i=0}^{m-1}\{-1,1\}^i}$ is \emph{shattered} by $H$ if for every $\xi = (\xi_1,\ldots,\xi_m)\in \{-1,1\}^m$, there exists $h\in H$ such that $h(x_{\xi_{<i}}) = \xi_i$ for every $i = 1,\ldots,m$, where $\xi_{<i}\in \{-1,1\}^{i-1}$ is the prefix of $\xi$ of length $i - 1$. The Littlestone dimension of $H$ is defined to be
\[
\ldim(H):=\sup\{m\in\bZ_{\ge 0}:\text{there exists $(x_\zeta)_{\zeta\in \cup_{i=0}^{m-1}\{-1,1\}^i}\in X^{2^m - 1}$ that is shattered by $H$}\}.
\]
\begin{theorem}[\citep{MR4399723}]
\label{thm:online-rea}
For every $H\subseteq\{-1,1,*\}^X$ and $n\in\bZ_{>0}$,
define $m:= \ldim(H)$ and 
\[
\varepsilon^*:=\inf\{\varepsilon \in\bR_{\ge 0}:\reao_n(H,\varepsilon)\ne \emptyset\}.
\]Then
\[
\min\left\{\frac 12,\frac {m}{2n}\right\} \le \varepsilon^* \le \frac {m}n.
\]
\end{theorem}
\begin{theorem}[\citep{MR4399723}]
\label{thm:online-agn}
For every $H\subseteq\{-1,1,*\}^X$ and $n\in\bZ_{>0}$,
define $m:= \ldim(H)$ and 
\[
\varepsilon^*:=\inf\{\varepsilon \in\bR_{\ge 0}:\agno_n(H,\varepsilon)\ne \emptyset\}.
\]
Then
\[
\Omega\left(\min\left\{1,\sqrt{\frac{m}{n}}\right\}\right) \le \varepsilon^* \le O\left(\sqrt{\frac mn\log\frac {2m + n}m}\right).
\]
\end{theorem}

To give a regret characterization for comparative online learning, we define the \emph{mutual Littlestone dimension} for a pair of hypothesis classes $S$ and $B$ to be
\begin{align*}
& \ldim(S,B):=\\
& \sup\{m\in\bZ_{\ge 0}:\text{there exists $(x_\zeta)_{\zeta\in \cup_{i=0}^{m-1}\{-1,1\}^i}\in X^{2^m - 1}$ that is shattered by both $S$ and $B$}\}.
\end{align*}
Similarly to \Cref{claim:comp-vc} for the mutual VC dimension, the mutual Littlestone dimension of $(S,B)$ is equal to the Littlestone dimension of the agreement hypothesis class $\bA_{S,B}$ (defined in \Cref{sec:comp}):
\begin{claim}
\label{claim:ld}
Let $S,B\subseteq\{-1,1,*\}^X$ be partial binary hypothesis classes. We have
$\ldim(S,B) = \ldim(\bA_{S,B})$.
\end{claim}
We omit the proof of the claim because the proof of \Cref{claim:comp-vc} can be applied here with only minor changes.
Our main result in this section is the following regret characterization for comparative online learning.
\begin{theorem}
\label{thm:online}
For every $S,B\subseteq\{-1,1,*\}^X$ and $n\in\bZ_{>0}$,
define $m:= \ldim(S,B)$ and 
\[
\varepsilon^*:=\inf\{\varepsilon \in\bR_{\ge 0}:\compo_n(S,B,\varepsilon)\ne \emptyset\}.
\]
Then
\[
\min\left\{\frac 12,\frac{m}{2n}\right\} \le \varepsilon^* \le O\left(\sqrt{\frac mn\log\frac {2m + n}m}\right).
\]
\end{theorem}
Our proof of \Cref{thm:online} uses the same strategy as in our proof of \Cref{thm:comp}. We reduce $\compo$ for $(S,B)$ to $\agno$ for $\bA_{S,B}$, and conversely reduce $\reao$ for $\bA_{S,B}$ to $\compo$ for $(S,B)$. \Cref{thm:online} is a direct corollary of \Cref{thm:online-rea}, \Cref{thm:online-agn}, \Cref{claim:ld} and the following two lemmas:
\begin{lemma}
\label{lm:online-upper-reduction}
Let $S,B\subseteq \{-1,1,*\}^X$ be binary hypothesis classes. For any $\varepsilon\ge 0$ and $n\in \bZ_{> 0}$, we have
$\agno_n(\bA_{S,B},\varepsilon) \subseteq \compo_n(S,B,\varepsilon)$. In other words, any learner solving agnostic online learning for $\bA_{S,B}$ also solves comparative online learning for $(S,B)$ with the same parameter $\varepsilon$.
\end{lemma}
\begin{lemma}
\label{lm:online-lower-reduction}
Let $S,B\subseteq \{-1,1,*\}^X$ be binary hypothesis classes. For any $\varepsilon\ge 0$ and $n\in \bZ_{> 0}$, we have
$\compo_n(S,B,\varepsilon)\subseteq \reao_n(\bA_{S,B},\varepsilon)$. In other words, any learner solving comparative online learning for $(S,B)$ also solves realizable online learning for $\bA_{S,B}$ with the same parameter $\varepsilon$.
\end{lemma}
\begin{proof}[Proof of \Cref{lm:online-upper-reduction}]
Let $L$ be a learner in $\agno_n(\bA_{S,B},\varepsilon)$. Consider an arbitrary $s\in S$ and any sequence of data points $(x_1,y_1),\ldots,(x_n,y_n)\in X\times\{-1,1\}$ satisfying $y_i = s(x_i)$ for every $i = 1,\ldots,n$.
By the guarantee of $L\in \agno_n(\bA_{S,B},\varepsilon)$, we have
\begin{align*}
\mistake(L;(x_i,y_i)_{i=1}^n) & \le \inf_{h\in \bA_{S,B}}\mistake(h;(x_i,y_i)_{i=1}^n) + \varepsilon\\
& = \inf_{s'\in S,b\in B}\mistake(\ba_{s',b};(x_i,y_i)_{i=1}^n) + \varepsilon\\
& \le \inf_{b\in B}\mistake(\ba_{s,b};(x_i,y_i)_{i=1}^n) + \varepsilon\\
& = \inf_{b\in B}\frac 1n\sum_{i=1}^n \one(\ba_{s,b}(x_i)\ne y_i) + \varepsilon\\
& = \inf_{b\in B}\frac 1n\sum_{i=1}^n \one(s(x_i)\ne y_i \text{ or } b(x_i)\ne y_i) + \varepsilon\tag{by \Cref{claim:agree}}\\
& = \inf_{b\in B}\frac 1n\sum_{i=1}^n \one(b(x_i)\ne y_i) + \varepsilon \tag{by $s(x_i) = y_i$}\\
& = \inf_{b\in B}\mistake(b;(x_i,y_i)_{i=1}^n) + \varepsilon.
\end{align*}
This proves that $L \in \compo_n(S,B,\varepsilon)$, as desired.
\end{proof}
\begin{proof}[Proof of \Cref{lm:online-lower-reduction}]
Let $L$ be a learner in $\compo_n(S,B,\varepsilon)$. Consider an arbitrary $h\in \bA_{S,B}$ and any sequence of data points $(x_1,y_1),\ldots,(x_n,y_n)\in X\times\{-1,1\}$ satisfying $y_i = h(x_i)$ for every $i = 1,\ldots,n$.

By the definition of $\bA_{S,B}$, there exists $s\in S$ and $b\in B$ such that $h = \ba_{s,b}$. For every $i = 1,\ldots,n$, by \Cref{claim:agree}, our assumption $y_i = h(x_i)$ implies that $s(x_i) = b(x_i) = y$. By the guarantee of $L\in \compo_n(S,B,\varepsilon)$, we have
\[
\mistake(L;(x_i,y_i)_{i=1}^n)  \le\inf_{b'\in B}\mistake(b';(x_i,y_i)_{i=1}^n) + \varepsilon = \varepsilon,
\]
where the last equation holds because $b(x_i)= y_i$ for every $i = 1,\ldots,n$ and thus $\inf_{b'\in B}\mistake(b';\allowbreak (x_i,y_i)_{i=1}^n) = 0$. The inequality above implies that $L\in \reao_n(\bA_{S,B},\varepsilon)$, as desired.
\end{proof}

\appendix
\section{A Learning Task with More than Two Hypothesis Classes}
\label{sec:agree}

In \Cref{sec:comp}, we show a connection between comparative learning for a pair of hypothesis classes $S,B\subseteq\{-1,1,*\}^X$ and realizable/agnostic learning for the agreement class $\bA_{S,B}$. In this section, we extend this connection to a learning task involving an arbitrary number of hypothesis classes and give a sample complexity characterization for this task:

\begin{definition}[Agreement learning]
\label{def:agree}
Let $(H_i)_{i\in I}$ be a collection of binary hypothesis classes where $H_i\subseteq \{-1,1,*\}^X$ for every index $i\in I$. Given an error bound $\varepsilon$, a failure probability bound $\delta$, and a nonnegative integer $n$, we define $\agree_n((H_i))_{i\in I},\varepsilon,\delta)$ to be $\learn_n(Z,F,\dstr,(F_\mu)_{\mu\in \dstr})$ where $Z = X\times\{-1,1\}$, $F = \{-1,1\}^X$, $\dstr$ consists of all distributions $\mu$ over $X\times\{-1,1\}$, and $F_\mu$ consists of all models $f:X\rightarrow\{-1,1\}$ satisfying
\[
{\Pr}_{(x,y)\sim\mu}[f(x) = y] \ge {\sup}_{(h_i)_{i\in I}\in\prod_{i\in I}H_i} {\Pr}_{(x,y)\sim\mu}[\forall i\in I, y = h_i(x)] - \varepsilon,
\]
where the supremum is taken over a collection of hypotheses $(h_i)_{i\in I}$ satisfying $h_i\in H_i$ for every index $i\in I$ (which we denote by $(h_i)_{i\in I}\in\prod_{i\in I}H_i$).
\end{definition}

We characterize the sample complexity of agreement learning in terms of the mutual VC dimension of the collection of hypothesis classes, where we extend the definition of mutual VC dimension to hold for more than two hypothesis classes in the natural way:
\[\vc((H_i)_{i \in I}) := \{|X'| : X' \subseteq X, X' \text{ is shattered by $H_i$ for all }i \in I\}.\]

\begin{theorem}
\label{thm:agree}
Let $(H_i)_{i\in I}$ be a collection of binary hypothesis classes where $H_i\subseteq \{-1,1,*\}^X$ for every index $i\in I$. For any $\varepsilon,\delta \in (0,1/4)$,
the sample complexity of agreement learning satisfies the following upper bound:
\[
\sagree((H_i))_{i\in I},\varepsilon,\delta) = O\left(\frac{\vc((H_i)_{i \in I})}{\varepsilon^2}\log_{+}^2\left(\frac{\vc((H_i)_{i \in I})}{\varepsilon}\right) + \frac 1{\varepsilon^2}\log\left(\frac 1\delta\right)\right),\label{eq:agree-upper}
\]
When $\vc((H_i)_{i\in I}) \ge 1$, we have the following lower bound:
\[
\sagree((H_i))_{i\in I},\varepsilon,\delta) = \Omega\left(\frac{\vc((H_i)_{i\in I})}{\varepsilon^2} + \frac 1{\varepsilon^2}\log\left(\frac 1\delta\right)\right).\label{eq:agree-lower}
\]
\end{theorem}

To prove \Cref{thm:agree}, we show that agreement learning for $(H_i)_{i\in I}$ is equivalent to agnostic learning for a single partial hypothesis class $\bA_{(H_i)_{i\in I}}\subseteq\{-1,1,*\}^X$ we define as follows.
For every $(h_i)_{i\in I}\in \prod_{i\in I}H_i$, we define an \emph{agreement} hypothesis $\ba_{(h_i)_{i\in I}}:X\rightarrow \{-1,1,*\}$ by
\[
\ba_{(h_i)_{i\in I}}(x) = \begin{cases}
0,& \textnormal{if } h_i(x) = 0 \textnormal{ for every } i\in I;\\
1,& \textnormal{if } h_i(x) = 1 \textnormal{ for every } i\in I;\\
*, & \textnormal{otherwise.}
\end{cases}
\]
We define the \emph{agreement class} $\bA_{(H_i)_{i\in I}} := \{\ba_{(h_i)_{i\in I}}: (h_i)_{i\in I}\in \prod_{i\in I}H_i\}\subseteq\{-1,1,*\}^X$. The following claim follows immediately from the definition of $\ba_{(h_i)_{i\in I}}$:
\begin{claim}
\label{claim:agn-agree}
Let $(h_i)_{i\in I}$ be a collection of hypotheses where $h_i:X\rightarrow\{-1,1,*\}$ for every index $i\in I$.
For every $(x,y)\in X\times\{-1,1\}$, it holds that
\[
y = \ba_{(h_i)_{i\in I}}(x)\Longleftrightarrow \forall i\in I, y = h_i(x).
\]
\end{claim}
The following claim can be proved similarly to \Cref{claim:comp-vc}:
\begin{claim}
\label{claim:agree-vc}
Let $(H_i)_{i\in I}$ be a collection of hypothesis classes where $H_i\subseteq \{-1,1,*\}^X$ for every index $i\in I$.
Then $\vc(\bA_{(H_i)_{i\in I}}) = \vc((H_i)_{i\in I})$.
\end{claim}

The following lemma shows that agreement learning for $(H_i)_{i\in I}$ is equivalent to agnostic learning for $\bA_{(H_i)_{i\in I}}$.
\begin{lemma}
\label{lm:agn-agree}
$\agn_n(\bA_{(H_i)_{i\in I}},\varepsilon,\delta)= \agree_n((H_i)_{i\in I},\varepsilon,\delta)$, i.e., any learner solving agnostic learning for $\bA_{(H_i)_{i\in I}}$ also solves agreement learning for $(H_i)_{i\in I}$ with the same parameters $\varepsilon$ and $\delta$, and vice versa.
\end{lemma}

\begin{proof}[Proof of \Cref{lm:agn-agree}]
The learning tasks $\agn_n(\bA_{(H_i)_{i\in I}},\varepsilon,\delta)$ and $\agree_n((H_i)_{i\in I},\varepsilon,\delta)$ are defined in \Cref{def:agn} and \Cref{def:agree}, respectively. Both tasks are defined to be $\learn_n(Z,F,\dstr,\allowbreak (F_\mu)_{\mu\in \dstr})$ where $Z = X\times\{-1,1\}$, $F = \{-1,1\}^X$, and $\dstr$ consists of all distributions $\mu$ over $X\times \{-1,1\}$. The only potential difference is in the choice of $F_\mu$. In the definition of $\agn_n(\bA_{(H_i)_{i\in I}},\varepsilon,\delta)$, $F_\mu$ consists of all models $f:X\rightarrow\{-1,1\}$ satisfying
\[
{\Pr}_{(x,y)\sim\mu}[f(x)\ne y]\le {\inf}_{h\in \bA_{(H_i)_{i\in I}}}{\Pr}_{(x,y)\sim\mu}[h(x)\ne y] + \varepsilon,
\]
or equivalently,
\[
{\Pr}_{(x,y)\sim\mu}[f(x) = y]\ge {\sup}_{h\in \bA_{(H_i)_{i\in I}}}{\Pr}_{(x,y)\sim\mu}[h(x) = y] - \varepsilon.
\]
By the definition of $\bA_{(H_i)_{i\in I}}$, the inequality above is equivalent to 
\begin{equation}
\label{eq:agn-agree-1}
{\Pr}_{(x,y)\sim\mu}[f(x) = y]\ge {\sup}_{(h_i)_{i\in I}\in \prod_{i\in I}H_i}{\Pr}_{(x,y)\sim\mu}[\ba_{(h_i)_{i\in I}}(x) = y] - \varepsilon.
\end{equation}
In the definition of $\agree_n((H_i)_{i\in I},\varepsilon,\delta)$, $F_\mu$ consists of all models $f:X\rightarrow\{-1,1\}$ satisfying
\begin{equation}
\label{eq:agn-agree-2}
{\Pr}_{(x,y)\sim\mu}[f(x) = y] \ge {\sup}_{(h_i)_{i\in I}\in\prod_{i\in I}H_i} {\Pr}_{(x,y)\sim\mu}[\forall i\in I, h_i(x) = y] - \varepsilon.
\end{equation}
By \Cref{claim:agn-agree}, inequalities \eqref{eq:agn-agree-1} and \eqref{eq:agn-agree-2} are equivalent, which implies that the choices of $F_\mu$ in the definitions of $\agn_n(\bA_{(H_i)_{i\in I}},\varepsilon,\delta)$ and $\agree_n((H_i)_{i\in I},\varepsilon,\delta)$ are the same, completing the proof.
\end{proof}
\begin{proof}[Proof of \Cref{thm:agree}]
By \Cref{lm:agn-agree}, $\sagree ((H_i)_{i\in I},\varepsilon,\delta)) = \sagn(\bA_{(H_i)_{i\in I}}, \varepsilon,\delta)$. The theorem then follows from \Cref{claim:agree-vc} and the sample complexity bounds for $\agn$ in \citep[Theorem 41]{MR4399723}.
\end{proof}

\section{Helper Lemmas}
The following lemma is a much weaker version of the Gilbert-Varshamov bound in coding theory \citep{gilbert1952comparison,varshamov1957estimate}. %
\begin{lemma}
\label{lm:gv}
Let $X$ be a non-empty finite set with $|X| = n$. For every $\varepsilon\in (0,1/2)$, there exists $F\subseteq\{-1,1\}^X$ with $|F|\ge 2^{\Omega(\varepsilon^2n)}$ such that for every distinct $f,f'\in F$,
\begin{equation}
\label{eq:gv}
\frac 1n\sum_{x\in X}\one(f(x)\ne f'(x)) \ge 1/2 - \varepsilon.
\end{equation}
\end{lemma}
\begin{proof}
We fix an absolute constant $C > 0$ whose value we determine later.
It is trivial that there exist $f,f'\in\{-1,1\}^X$ such that \eqref{eq:gv} holds (for example, choose $f(x) = 1$ and $f'(x) = -1$ for every $x\in X$). Therefore, 
if $\varepsilon^2n \le C$, we can choose $F = \{f,f'\}$ with $|F| = 2 \ge 2^{\varepsilon^2n/C}$ to prove the lemma. We thus assume that $\varepsilon^2n > C$.

For a positive integer $N$, let $f_1,\ldots,f_N$ be drawn uniformly and independently from $\{-1,1\}^X$. For every pair $i,j\in \bZ$ satisfying $1 \le i < j \le N$, by the Chernoff bound, with probability at least $1 - 2^{-c\varepsilon^2n}$ for an absolute constant $c > 0$,
\begin{equation}
\label{eq:gv-2}
\frac 1n\sum_{x\in X}\one(f_i(x)\ne f_j(x)) \ge 1/2 - \varepsilon.
\end{equation}
Now we set $C$ to be larger than $4/c$ and choose $N = \lfloor 2^{c\varepsilon^2n/2}\rfloor - 1$. By our assumption, $c\varepsilon^2n > cC > 4$, so $N = 2^{\Omega(\varepsilon^2n)}$.
By the union bound, with probability at least $1 - 2^{-c\varepsilon^2n}N^2 > 0$, \eqref{eq:gv-2} holds for every pair $i,j\in \bZ$ satisfying $1\le i < j \le N$, in which case choosing $F = \{f_1,\ldots,f_N\}$ proves the lemma.
\end{proof}
\begin{lemma}
\label{lm:hard-task}
For $m\in \bZ_{>0}$, let $X$ be a set with $|X| = m$ and let $\mu_X$ be the uniform distribution over $X$.
Let $L$ be a learner that takes $n$ data points $(x_1,y_1),\ldots,(x_n,y_n)\in X\times\{-1,1\}$ as input and outputs a model $f:X\to[-1,1]$.
Assume $\varepsilon\in [0,1]$.
Suppose that for every $h:X\to\{-1,1\}$, if the input data points $(x_1,y_1),\ldots,(x_n,y_n)$ are drawn i.i.d.\ such that $x_i\sim\mu_X$ and $y_i = h(x_i)$ for every $i = 1,\ldots,n$, then with probability more than $1/2$, the output model $f$ satisfies
\[
\bE_{x\sim\mu_X}|f(x)- h(x)| \le \varepsilon.
\]
Then, 
\[
n \ge (1 - \varepsilon)m.
\]
\end{lemma}
\begin{proof}
Consider the process of first drawing $h$ uniformly at random from $\{-1,1\}^X$, and draw $n$ data points $(x_1,y_1),\ldots,(x_n,y_n)\in X\times\{-1,1\}$ i.i.d.\ such that $x_i\sim\mu_X$ and $y_i = s(x_i)$ for every $i = 1,\ldots,n$. 
We then use these data points as input to $L$ and obtain its output model $f:X\to[-1,1]$.
By our assumption, with probability more than $1/2$,
\begin{equation}
\label{eq:hard-1}
\bE_{x\sim\mu_X}|f(x)- h(x)| \le \varepsilon.
\end{equation}
Now we give a lower bound on $\bE_{x\sim\mu_X}[|f(x)- h(x)|]$:
\begin{align*}
\bE_{x\sim\mu_X}[|f(x)- h(x)|] & \ge  \frac 1m\sum_{x\in X\backslash\{x_1,\ldots,x_n\}}|f(x) - h(x)|\\
& = \frac 1m\sum_{x\in X\backslash\{x_1,\ldots,x_n\}}(1 - f(x)h(x))\tag{by $h(x)\in \{-1,1\}$ and $f(x)\in [-1,1]$}\\
& \ge \frac{m - n}{m} - \frac 1m\sum_{x\in X\backslash \{x_1,\ldots,x_n\}} f(x)h(x).
\end{align*}
Conditioned on the input data points $(x_1,y_1),\ldots,(x_n,y_n)$, the label $h(x)$ of every individual $x\in X\backslash\{x_1,\ldots,x_n\}$ distributes independently and uniformly from $\{-1,1\}$, and the model $f$ is independent from all these labels $h(x)$ for $x\in X\backslash\{x_1,\ldots,x_n\}$. Therefore, the distribution of $\sum_{x\in X\backslash\{x_1,\ldots,x_n\}} f(x)h(x)$ is symmetric around zero. With probability at least $1/2$, we have
\[
\sum_{x\in X\backslash\{x_1,\ldots,x_n\}} f(x)h(x) \le 0,
\]
in which case
\begin{equation}
\label{eq:hard-2}
\bE_{x\sim\mu_X}|f(x)- h(x)| \ge \frac{m-n}{m}.
\end{equation}
Therefore, with nonzero probability, \eqref{eq:hard-1} and \eqref{eq:hard-2} both hold. This implies that $\frac{m - n}m\le \varepsilon$ and thus $n \ge (1 - \varepsilon)m$, as desired.
\end{proof}
\section{Non-duality Examples}
\label{sec:non-duality}
In this section, we show examples where sample complexity duality does not hold for several learning tasks we consider in this paper, i.e., swapping the roles of the two hypothesis classes in these tasks can significantly change the sample complexity. All examples in this section only require total hypotheses.

Most learning tasks throughout the paper are defined in the distribution-free setting, but
many of our examples in this section are with regard to the distribution-specific variants of these tasks.
For example, given a distribution $\mu_X$ over $X$, we define the distribution-specific comparative learning task $\comp\sps {\mu_X}_n(S,B,\varepsilon,\delta)$ in the same way as we define $\comp_n(S,B,\varepsilon,\delta)$ in \Cref{def:comp} except that the distributions $\mu\in \dstr$ are additionally constrained to satisfy $\mu|_X = \mu_X$. We omit the distribution-specific definitions for all other learning tasks.
\subsection{Distribution-Specific Comparative Learning}

For $m\in\bZ_{>0}$, consider a set $X$ with $|X| = 4m$ and let $\mu_X$ be the uniform distribution over $X$. Let $S$ be the class consisting of all hypotheses $s:X\rightarrow\{-1,1\}$ satisfying $|\{x\in X:s(x) = 1\}| = m$, and let $B$ be the class consisting of all hypotheses $b:X\rightarrow\{-1,1\}$ satisfying $|\{x\in X:b(x) = 1\}| = 2m$.
\begin{lemma}
For any $\varepsilon,\delta\in \bR_{\ge 0}$, we have
\begin{equation}
\label{eq:non-duality-comp-1}
\scomp\sps{\mu_X}(S,B,\varepsilon,\delta) = 0.
\end{equation}
However, for any $\varepsilon\in (0,1/4)$ and $\delta\in (0,1/2)$,
\begin{equation}
\label{eq:non-duality-comp-2}
\scomp\sps{\mu_X}(B,S,\varepsilon,\delta) \ge (1/2 - 2\varepsilon)m.
\end{equation}
\end{lemma}
\begin{proof}
To prove \eqref{eq:non-duality-comp-1}, it suffices to show that the learner $L$ which always outputs the constant function $f:X\to\{-1,1\}$ with $f(x) = -1$ for every $x\in X$ belongs to $\comp_0\sps{\mu_X}(S,B,\varepsilon,\delta)$. Let $\mu$ be a distribution over $X\times\{-1,1\}$ satisfying $\mu|_X = \mu_X$ and $\Pr_{(x,y)\sim\mu}[s(x) = y] = 1$ for some $s\in S$. We have
\[
{\Pr}_{(x,y)\sim\mu}[f(x)\ne y] = {\Pr}_{x\sim\mu_X}[f(x)\ne s(x)] = {\Pr}_{x\sim\mu_X}[s(x) = 1] = 1/4,
\]
and for every $b\in B$,
\[
{\Pr}_{(x,y)\sim\mu}[b(x)\ne y] = {\Pr}_{x\sim\mu_X}[b(x)\ne s(x)] \ge {\Pr}_{x\sim\mu_X}[b(x) = 1] - \Pr_{x\sim\mu_X}[s(x) = 1] = 1/2 - 1/4 = 1/4.
\]
Therefore,
\[
{\Pr}_{(x,y)\sim\mu}[f(x)\ne y] \le {\inf}_{b\in B}{\Pr}_{(x,y)\sim\mu}[b(x)\ne y],
\]
and thus $L \in \comp_0\sps{\mu_X}(S,B,\varepsilon,\delta)$.

Now we prove \eqref{eq:non-duality-comp-2}. 
Let $L$ be a learner in $\comp_n\sps{\mu_X}(B,S,\varepsilon,\delta)$ for some $n\in \bZ_{\ge 0}$. It suffices to show that $n \ge (1/2 - 2\varepsilon)m$.
Without loss of generality, assume $X = X_1\cup X_2$ where $X_1 = \{1,\ldots,2m\}$ and $X_2 = \{-1,\ldots,-2m\}$. Consider an arbitrary hypothesis $h:X_1\rightarrow\{-1,1\}$ and suppose we get data points $(x_1,y_1),\ldots,(x_n,y_n)$ such that every $x_i$ is drawn i.i.d.\ from the uniform distribution $\unif_{X_1}$ over $X_1$ and $y_i = h(x_i)$. Define new data points $(x_i',y_i')$ to be $(x_i,y_i)$ with probability $1/2$ and $(-x_i,-y_i)$ with probability $1/2$. 
Then, the new data points are distributed i.i.d.\ according to a distribution $\mu$ over $X\times\{-1,1\}$ satisfying $\mu|_X = \mu_X$ and $\Pr_{(x,y)\sim\mu}[b(x) = y] = 1$, 
where $b:X\rightarrow\{-1,1\}$ satisfies $b(x) = h(x)$ if $x\in X_1$ and $b(x) = -h(-x)$ if $x\in X_2$. It is clear that $b\in B$. 
Thus if we feed the new data points into the learner $L$, with probability at least $1 - \delta >1/2$, we get a model $f:X\rightarrow\{-1,1\}$ satisfying 
\[
{\Pr}_{(x,y)\sim\mu}[f(x)\ne y] \le {\inf}_{s\in S}{\Pr}_{(x,y)\sim\mu}[s(x)\ne y] + \varepsilon,
\]
or equivalently,
\[
{\Pr}_{x\sim\mu_X}[f(x)\ne b(x)] \le {\inf}_{s\in S}{\Pr}_{x\sim\mu_X}[s(x)\ne b(x)] + \varepsilon = \frac 14 + \varepsilon.
\]
The inequality above implies $\bE_{x\sim\mu_X}[|f(x) - b(x)|]\le 1/2 + 2\varepsilon$. Therefore,
\begin{align*}
1/2 + 2\varepsilon & \ge \frac{1}{4m}\sum_{x\in X_1}|f(x) - b(x)| + \frac{1}{4m}\sum_{x\in X_2}|f(x) - b(x)|\\
& = \frac{1}{4m}\sum_{x\in X_1}|f(x) - h(x)| + \frac{1}{4m}\sum_{x\in X_2}|f(x) + h(-x)|\\
& = \frac{1}{4m}\sum_{x\in X_1}|f(x) - h(x)| + \frac{1}{4m}\sum_{x\in X_1}|f(-x) + h(x)|\\
& = \frac{1}{4m}\sum_{x\in X_1}|(f(x) - f(-x)) - 2h(x)|\\
& = \frac 1{2m}\sum_{x\in X_1}|\frac{f(x) - f(-x)}{2} - h(x)|.
\end{align*}
This means that the function $g:X_1\to [-1,1]$ defined by $g(x) = \frac{f(x) - f(-x)}{2}$ for every $x\in X_1$ satisfies the following with probability more than $1/2$:
\[
\bE_{x\sim\unif_{X_1}}|g(x) - h(x)| \le 1/2 + 2\varepsilon.
\]
By \Cref{lm:hard-task}, we have $n \ge (1/2 - 2\varepsilon)m$, as desired.
\end{proof}
\subsection{Correlation Maximization}
\label{sec:non-duality-cm}
For $m\in \bZ_{>0}$, consider a non-empty finite set $X$ with $|X| = 2m$. Define $S = \{0,1\}^X$ and $B = \{-1,1\}^X$.
\begin{lemma}
For every $\varepsilon,\delta\in\bR_{\ge 0}$, we have
\begin{equation}
\label{eq:cm-non-duality-1}
\scm(S,B,\varepsilon,\delta) = 0.
\end{equation}
However, even for the uniform distribution $\mu_X$ over $X$, for every $\varepsilon,\delta\in (0,1/2)$,
\begin{equation}
\label{eq:cm-non-duality-2}
\sdcm\sps{\mu_X}(B,S,\varepsilon,\delta) \ge (1/2 - \varepsilon)m.
\end{equation}
\end{lemma}
\begin{proof}
To prove \eqref{eq:cm-non-duality-1}, it suffices to show that the learner $L$ which always outputs the constant function $f:X\to\{-1,1\}$ with $f(x) = 1$ for every $x\in X$ belongs to $\cm_0(S,B,\varepsilon,\delta)$. Let $\mu$ be a distribution over $X\times\{-1,1\}$ satisfying $\Pr_{(x,y)\sim\mu}[s(x) = y] = 1$ for some $s\in S$. Since $s\in S = \{0,1\}^X$, we have $s(x) \ge 0$ for every $x\in X$, and thus $\Pr_{(x,y)\sim\mu}[y \ge 0] = 1$. Therefore,
\[
\bE_{(x,y)\sim\mu}[yf(x)] = \bE_{(x,y)\sim\mu}[y] = \bE_{(x,y)\sim\mu}|y| \ge {\sup}_{b\in B}\bE_{(x,y)\sim\mu}[yb(x)].
\]
This proves that $L \in \cm_0(S,B,\varepsilon,\delta)$.

Now we prove \eqref{eq:cm-non-duality-2}.
Let $L$ be a learner in $\dcm_n\sps{\mu_X}(B,S,\varepsilon,\delta)$ for some $n\in \bZ_{\ge 0}$. It suffices to show that $n \ge (1/2 - \varepsilon)m$.
Without loss of generality, assume $X = X_1\cup X_2$ where $X_1 = \{1,\ldots,m\}$ and $X_2 = \{-1,\ldots,-m\}$. For $h:X_1\to \{-1,1\}$, let $(x_1,y_1),\ldots,(x_n,y_n)$ be data points where every $x_i$ is drawn independently from the uniform distribution $\unif_{X_1}$ over $X_1$ and $y_i = h(x_i)$. For every $i = 1,\ldots,n$, define a new data point $(x_i',y_i')$ such that $(x_i',y_i') = (x_i,y_i)$ with probability $1/2$ and $(x_i',y_i') = (-x_i,-y_i)$ with the remaining probability $1/2$. It is clear that the new data points are distributed i.i.d.\ from a distribution $\mu$ over $X\times \{-1,1\}$ satisfying $\mu|_X = \mu_X$ and $\Pr_{(x,y)\sim\mu}[b(x) = y] = 1$, where $b:X\to \{-1,1\}$ satisfies $b(x) = h(x)$ for every $x\in X_1$ and $b(x) = -h(-x)$ for every $x\in X_2$.
Thus if $L$ takes the new data points as input,
with probability at least $1 - \delta > 1/2$, it outputs a model $f:X\to[-1,1]$ satisfying
\begin{equation}
\label{eq:cm-non-duality}
\bE_{(x,y)\sim\mu}[yf(x)] \ge {\sup}_{s\in S}\bE_{(x,y)\sim\mu}[ys(x)] - \varepsilon.
\end{equation}
Consider the function $s\in S$ satisfying $s(x) = 1$ whenever $b(x) = 1$ and $s(x) = 0$ whenever $b(x) = - 1$. We have
\[
{\Pr}_{(x,y)\sim\mu}[ys(x)] = {\Pr}_{(x,y)\sim\mu}[b(x) = 1] = 1/2.
\]
Therefore, \eqref{eq:cm-non-duality} implies
\[
\bE_{(x,y)\sim\mu}[|f(x) - y|] = 1 - \bE_{(x,y)\sim\mu}[yf(x)] \le 1/2 + \varepsilon,
\]
which then implies 
\begin{align*}
1/2 + \varepsilon & \ge \frac 1{2m}\sum_{x\in X_1}|f(x) - b(x)| + \frac 1{2m}\sum_{x\in X_2}|f(x) - b(x)|\\
& \ge \frac 1{2m}\sum_{x\in X_1}|(f(x) - b(x)) - (f(-x) - b(-x))|\\
& = \frac 1{2m}\sum_{x\in X_1}|f(x) - f(-x) - 2h(x)|\\
& = \frac 1m\sum_{x\in X_1}\left|\frac{f(x) - f(-x)}{2} - h(x)\right|.
\end{align*}
Therefore, if we define $g:X_1\to \{-1,1\}$ by $g(x) = \frac{f(x)- f(-x)}{2}$ for every $x\in X_1$, with probability more than $1/2$,
\[
\bE_{(x,y)\sim\unif_{X_1}}[|g(x) - h(x)|] \le 1/2 + \varepsilon.
\]
By \Cref{lm:hard-task}, we get $n \ge (1/2 - \varepsilon)m$, as desired.
\end{proof}
\subsection{Comparative Regression}
For $m\in \bZ_{>0}$, consider a set $X$ with $|X| = m$. Define $S = \{-1/2,1/2\}^X$ and $B = \{-1,1\}^X$. Let $\ell:\{-1,1\}\times[-1,1]\to \bR$ be the squared loss: $\ell(y,u) = (y - u)^2$ for every $y\in \{-1,1\}$ and $u\in [-1,1]$.
\begin{lemma}
\label{lm:compr-non-duality}
For every $\varepsilon,\delta\in\bR_{\ge 0}$,
\begin{equation}
\label{eq:compr-non-duality-1}
\scompr(S,B,\ell,\varepsilon,\delta) = 0,
\end{equation}
However, even for the uniform distribution $\mu_X$ over $X$, for any $\varepsilon\in (0,3/4)$ and $\delta\in (0,1/2)$,
\begin{equation}
\label{eq:compr-non-duality-2}
\scompr\sps{\mu_X}(B,S,\ell,\varepsilon,\delta) \ge \left(1 - \sqrt{1/4 + \varepsilon}\right)m.
\end{equation}
\end{lemma}
\begin{proof}
To prove \eqref{eq:compr-non-duality-1}, it suffices to show that the learner $L$ which always outputs the constant function $f:X\to[-1,1]$ with $f(x) = 0$ for every $x\in X$ belongs to $\compr_0(S,B,\ell,\varepsilon,\delta)$. Let $\mu$ be a distribution over $X\times\{-1,1\}$ satisfying $\bE_{(x,y)\sim\mu}[y|x] = s(x)$ for some $s\in S$. 
We have
\[
\bE_{(x,y)\sim\mu}[\ell(y,f(x))] = \bE_{(x,y)\sim\mu}[y^2] = 1,
\]
and for every $b\in B$,
\begin{align*}
\bE_{(x,y)\sim\mu}[\ell(y,b(x))] & = \bE_{(x,y)\sim\mu}[(y - b(x))^2] \\
& = \bE_{x\sim\mu|_X}[\bE_{y\sim\ber^*(s(x))}[(y - b(x))^2]]\\
& =  \bE_{x\sim\mu|_X}\left[\frac{1 + s(x)}{2}(1 - b(x))^2 + \frac{1 - s(x)}{2}(-1 - b(x))^2\right]\\
& \ge 1,
\end{align*}
where the last inequality holds because $s(x)\in \{-1/2,1/2\}$ and $b(x)\in \{-1,1\}$.
Therefore,
\[
\bE_{(x,y)\sim\mu}[\ell(y,f(x))] \le \inf_{b\in B}\bE_{(x,y)\sim\mu}[\ell(y,b(x))].
\]
This proves that $L\in \compr_0(S,B,\ell,\varepsilon,\delta)$.

Now we prove \eqref{eq:compr-non-duality-2}.
Let $L$ be a learner in $\compr_n\sps{\mu_X}(B,S,\varepsilon,\delta)$ for some $n\in \bZ_{\ge 0}$. It suffices to show that $n \ge \left(1 - \sqrt{1/4 + \varepsilon}\right)m$. Let $\mu$ be a distribution over $X\times\{-1,1\}$ satisfying $\mu|_X = \mu_X$ and $\Pr_{(x,y)\sim\mu}[b(x) = y] = 1$ for some $b\in B$. When $L$ takes $n$ data points $(x_1,y_1),\ldots,(x_n,y_n)$ drawn i.i.d.\ from $\mu$, with probability at least $1-\delta > 1/2$, it outputs a model $f:X\to[-1,1]$ satisfying
\begin{equation}
\label{eq:compr-non-duality-3}
\bE_{(x,y)\sim\mu}[\ell(y,f(x))]\le\inf_{s\in S}\bE_{(x,y)\sim\mu}[\ell(y,s(x))] + \varepsilon.
\end{equation}
For the hypothesis $s\in S$ satisfying $s(x) = b(x)/2$ for every $x\in X$, we have
\[
\bE_{(x,y)\sim\mu}[\ell(y,s(x))] = \bE_{(x,y)\sim\mu}[(y - s(x))^2] = \bE_{x\sim\mu_X}[(b(x) - s(x))^2] = 1/4.
\]
Therefore, \eqref{eq:compr-non-duality-3} implies that 
\[
(\bE_{(x,y)\sim\mu}|f(x) - y|)^2 \le \bE_{(x,y)\sim\mu}[(f(x) - y)^2] = \bE_{(x,y)\sim\mu}[\ell(y,f(x))] \le 1/4 + \varepsilon,
\]
where the first inequality holds by Jensen's inequality. Now we know that with probability more than $1/2$, $\bE_{(x,y)\sim\mu}|f(x) - y| \le \sqrt{1/4 + \varepsilon}$.
By \Cref{lm:hard-task}, we get $n \ge \left(1 - \sqrt{1/4 + \varepsilon}\right)m$, as desired.
\end{proof}
\subsection{Distribution-Specific Realizable Multicalibration}
We give an example showing that sample complexity duality does not hold for distribution-specific realizable multicalibration. 

For a positive integer $m$,
we choose $X$ to be $\{\bot\}\cup\{a_{ij}: i\in \{1,2,3,4\}, j\in \{1,\ldots,m\}\}$, and we choose $\mu_X$ to be the distribution over $X$ that places $1/2$ probability mass on $\bot$, and uniformly distributes the remaining $1/2$ probability mass on $\bot$. Therefore, the probability mass on every $a_{ij}$ is $1/(8m)$.

For $h = (h_1,\ldots,h_m)\in \{-1,1\}^m$, we define $s_h:X\to [-1,1]$ such that $s_h(\bot) = 1$, and for every $j\in \{1,\ldots,m\}$,
\begin{align*}
s_h(a_{1j}) & = (h_j + 2)/3,\\
s_h(a_{2j}) & = (-h_j + 2)/3.\\
s_h(a_{3j}) & = (h_j - 2)/3,\\
s_h(a_{4j}) & = (-h_j - 2)/3,
\end{align*}

For $p = (p_1,\ldots,p_m)\in \{-1,1\}^m$ and $r = (r_1,\ldots,r_m)\in\{-1,1\}^m$, we define $b_{p,r}:X\to [-1,1]$ such that $b_{p,r}(\bot) = 0$, and for every $j\in \{1,\ldots,m\}$,
\begin{align*}
b_{p,r}(a_{1j}) & = r_j,\\
b_{p,r}(a_{2j}) & = p_j,\\
b_{p,r}(a_{3j}) & = -r_j,\\
b_{p,r}(a_{4j}) & = -p_j.
\end{align*}

We define $S = \{s_h:h\in \{-1,1\}^m\}$ and $B = \{b_{p,r}:p,r\in\{-1,1\}^m\}$.
\begin{lemma}
\label{lm:mc-non-duality}
Let $m,\mu_X,S,B$ be defined as above. For every $\varepsilon,\delta\in (0,1/2)$, we have
\begin{equation}
\label{eq:mc-non-duality-a}
\smc\sps {\mu_X}(B,S,\varepsilon,\delta)\le O(\varepsilon^{-2}\log(1/\delta)), 
\end{equation}
and for every $\varepsilon\in (0,1/28)$ and $\delta\in (0,1/2)$,
\begin{equation}
\label{eq:mc-non-duality-b}
\smc\sps {\mu_X}(S,B,\varepsilon,\delta)\ge (1 - 28\varepsilon)m.
\end{equation}
\end{lemma}
\begin{proof}

For any $s\in S$ and $b\in B$, there exist $h,p,r\in\{-1,1\}^m$ such that $s = s_h$ and $b = b_{p,r}.$ Therefore,
\[
\bE_{x\sim\mu_X}[s(x)b(x)] = \frac 12 s(\bot)b(\bot) + \frac 1{8m}\sum_{j=1}^m\sum_{i=1}^4s(a_{ij})b(a_{ij}) = \frac 1{8m} \cdot \frac 43\sum_{j=1}^m(p_j + r_j).
\]
This implies that $\bE_{x\sim\mu_X}[s(x)b(x)]$ only depends on $b\in B$ and does not depend on $s\in S$. Also, since $p_j,r_j\in \{-1,1\}$, the equation above implies that $\bE_{x\sim\mu_X}[s(x)b(x)]\in [-1/3,1/3]$.

We first show \eqref{eq:mc-non-duality-a} by designing a learning in $\mc_n\sps{\mu_X}(B,S,\varepsilon,\delta)$ for $n = O(\varepsilon^{-2}\log(1/\delta))$.
Let $\mu$ be a distribution over $X\times [-1,1]$ such that $\mu|_X = \mu_X$ and $\bE_{(x,y)\sim\mu}[y|x] = b(x)$ for some $b\in B$.
Given $O(\varepsilon^{-2}\log(1/\delta))$ data points drawn i.i.d.\ from $\mu$, by the Chernoff bound, the learner can compute an estimator $u\in [-1/3,1/3]$ such that with probability at least $1-\delta$,
\begin{equation}
\label{eq:mc-non-duality-1}
|u - \bE_{x\sim\mu_X}[s(x)b(x)]| \le \varepsilon \text{ for every } s\in S.
\end{equation}
Here we use the fact that $\bE_{x\sim\mu_X}[s(x)b(x)]$ does not depend on $s\in S$. After obtaining $u$, the learner simply outputs the constant function $f:X\to [-1,1]$ such that $f(x) = 2u$ for every $x\in X$. Now for every $s \in S$,
\[
\bE_{x\sim\mu_X}[s(x)f(x)] = \frac 12 s(\bot)f(\bot) + \frac 1{8m}\sum_{j=1}^m\sum_{i=1}^4s(a_{ij})f(a_{ij}) = u + 0 = u.
\]
Therefore, \eqref{eq:mc-non-duality-1} implies that for every $s\in S$,
\[
|\bE_{x\sim\mu_X}[(f(x) - b(x))s(x)]| = |u - \bE_{x\sim\mu_X}[s(x)b(x)]| \le \varepsilon,
\]
which then implies that
\[
\mce_{\mu,S}(f) = \mae_{\mu,S}(f) \le \varepsilon,
\]
as desired. Here, the first equation holds because $f$ is a constant function.

Now we show \eqref{eq:mc-non-duality-b}. 
Let $L$ be a learner in $\mc_n\sps{\mu_X}(S,B,\varepsilon,\delta)$ for some $n\in \bZ_{\ge 0}$. It suffices to show that $n \ge (1 - 28\varepsilon)m$. Define $J :=\{1,\ldots,m\}$. For $h\in \{-1,1\}^m$, let $(j_1,y_1),\ldots,(j_n,y_n)\in J\times\{-1,1\}$ be data points such that for every $\ell = 1,\ldots,n$, $j_\ell$ is drawn independently from the uniform distribution $\unif_J$ over $J$, and $y_\ell = h_{j_\ell}$.
For every data point $(j_\ell,y_\ell)$, we randomly and independently construct a new data point $(x_\ell,y'_\ell)\in X\times [-1,1]$ as follows. With probability $1/2$, we choose $(x_\ell,y'_\ell) = (\bot,1)$. With the remaining probability $1/2$, we first choose $x_\ell$ uniformly at random from $\{a_{ij_{\ell}}:i=1,2,3,4\}$. We then choose $y'_\ell$ so that
\[
y'_\ell = \begin{cases}
(y_\ell + 2)/3, & \text{if }x_\ell = a_{1j_\ell},\\
(-y_\ell + 2)/3, & \text{if }x_\ell = a_{2j_\ell},\\
(y_\ell - 2)/3, & \text{if }x_\ell = a_{3j_\ell},\\
(-y_\ell - 2)/3, & \text{if }x_\ell = a_{4j_\ell}.
\end{cases}
\]
It is clear that every new data point $(x_\ell,y'_\ell)$ is distributed independently from a distribution $\mu$ over $X\times[-1,1]$ satisfying $\mu|_X = \mu_X$ and $\bE_{(x,y)\sim\mu}[y|x] = s_h(x)$.
When $L$ takes the new data points as input, with probability at least $1-\delta > 1/2$, it outputs a model $f:X\to[-1,1]$ satisfying
\begin{equation}
\label{eq:mc-non-duality-2}
\mce_{\mu,B}(f) \le\varepsilon.
\end{equation}
Inequality \eqref{eq:mc-non-duality-2} implies that
\begin{align}
\varepsilon & \ge \mce_{\mu,B}(f)\notag\\
& \ge \mae_{\mu,B}(f)\notag\\
& = \sup_{b\in B}\left|\frac 1{8m}\sum_{j=1}^m\sum_{i=1}^4\Big((f(a_{ij}) - s_h(a_{ij}))b(a_{ij}))\Big)\right|\notag\\
& = \sup_{p,r\in\{-1,1\}^m}\left|\frac 1{8m}\sum_{j=1}^m(f(a_{1j}) - f(a_{3j})-(4/3))r_j + \frac 1{8m}\sum_{j=1}^m(f(a_{2j}) - f(a_{4j})-(4/3))p_j\right|\notag\\
& = \frac 1{8m}\sum_{j=1}^m|f(a_{1j}) - f(a_{3j})-(4/3)| + \frac 1{8m}\sum_{j=1}^m|f(a_{2j}) - f(a_{4j})-(4/3)|\notag\\
& \ge \frac 1{24m}\sum_{j=1}^m\Big(\one(f(a_{1j})< 0) + \one(f(a_{2j})< 0) + \one(f(a_{3j})\ge 0) + \one(f(a_{4j})\ge 0)\Big)\label{eq:mc-non-duality-c}\\
& = \frac 1{24m}\sum_{j=1}^m\sum_{i=1}^4|\one(f(a_{ij})\ge 0) - \one(i\in \{1,2\})|.\label{eq:mc-non-duality-3}
\end{align}
Here, \eqref{eq:mc-non-duality-c} holds because $f(a_{1j}),f(a_{2j}),f(a_{3j}),f(a_{4j})$ all lie in the interval $[-1,1]$.
Therefore, inequality \eqref{eq:mc-non-duality-2} also implies that
\begin{align*}
\varepsilon & \ge \mce_{\mu,B}(f) \\
& \ge \sup_{b\in B}\left|\frac 1{8m}\sum_{j=1}^m\sum_{i=1}^4\Big((f(a_{ij}) - s_h(a_{ij}))b(a_{ij}))\Big)\one(f(a_{ij})\ge 0)\right|\\
& \ge \sup_{b\in B}\left|\frac 1{8m}\sum_{j=1}^m\sum_{i=1}^4\Big((f(a_{ij}) - s_h(a_{ij}))b(a_{ij}))\Big)\one(i\in \{1,2\})\right| - 6\varepsilon,\tag{by \eqref{eq:mc-non-duality-3}}
\end{align*}
and thus
\begin{align*}
7\varepsilon & \ge \sup_{b\in B}\left|\frac 1{8m}\sum_{j=1}^m\sum_{i=1}^2\Big((f(a_{ij}) - s_h(a_{ij}))b(a_{ij}))\Big)\right|\\
& = \sup_{p,r\in \{-1,1\}^m}\left|\frac 1{8m}\sum_{j=1}^m(f(a_{1j}) - s_h(a_{1j}))r_j + \frac 1{8m}\sum_{j=1}^m(f(a_{2j}) - s_h(a_{2j}))p_j\right|\\
& = \frac 1{8m}\sum_{j=1}^m|f(a_{1j}) - s_h(a_{1j})| + \frac 1{8m}\sum_{j=1}^m|f(a_{2j}) - s_h(a_{2j})|\\
& = \frac 1{8m}\sum_{j=1}^m|f(a_{1j}) - (h_j + 2)/3| + \frac 1{8m}\sum_{j=1}^m|f(a_{2j}) - (-h_j + 2)/3|\\
& \ge \frac 1{8m}\sum_{j=1}^m\left|\Big(f(a_{1j}) - (h_j + 2)/3\Big) - \Big(f(a_{2j}) - (-h_j + 2)/3\Big)\right|\\
& = \frac 1{4m}\sum_{j=1}^m\left|\frac{f(a_{1j}) - f(a_{2j})}{2} - h_j\right|.
\end{align*}
Define $g:\{1,\ldots,m\}\to [-1,1]$ such that $g(j) = \frac{f(a_{1j}) - f(a_{2j})}{2}$ for every $j\in \{1,\ldots,m\}$. The inequality above implies
\[
\bE_{j\sim\unif_J}|g(j) - h_j| \le 28\varepsilon.
\]
By \Cref{lm:hard-task}, we have $n \ge (1 - 28\varepsilon)m$, as desired.
\end{proof}
\begin{remark}
\label{remark:mc-non-duality}
In the proof of \Cref{lm:mc-non-duality}, we show that $\bE_{x\sim\mu_X}[s(x)b(x)]$ does not depend on $s\in S$. This implies that $\sma\sps{\mu_X}(S,B,\varepsilon,\delta) = 0$ because a learner can simply output any $s\in S$. This gives a sample complexity separation between multiaccuracy and multicalibration in the distribution-specific realizable setting. 
\end{remark}

\section{Proof of Claim~\ref{claim:ma-mc-vc-bound}}
\label{sec:proof-ma-mc-vc-bound}
We recall \Cref{claim:ma-mc-vc-bound}:
\claimmamcvcbound*
\begin{proof}
Let $X'$ be a finite subset of $X$ shattered by both $\tilde S_{\eta_1}\sps {r_1}$ and $\tilde B_{\eta_2}\sps \theta$ for some $\theta\in \bR$. It suffices to show that 
\begin{equation}
\label{eq:proof-ma-mc-vc-bound-1}
|X'|\le 2{\sup}_{\theta'\in \bR}\vc(S_{2\eta_1}\sps{2r_1 + f},B_{\eta_2}\sps {\theta'}).
\end{equation}

The fact that $X'$ is shattered by $\tilde S_{\eta_1}\sps {r_1}$ implies that for every $\xi:X'\to\{-1,1\}$, there exists $\tilde s\in \tilde S$ such that
\[
\tilde s_{\eta_1}\sps{r_1}(x) = \xi(x) \text{ for every }x\in X'.
\]
By the definition of $\tilde S$, for every $\tilde s\in \tilde S$, there exists $s\in S$ such that $\tilde s(x) = (s(x) - f(x))/2$ for every $x\in X$. This implies that $\tilde s(x)  - r_1(x) = \frac 12 (s(x) - f(x) - 2r_1(x))$ and thus $\tilde s_{r_1}\sps{\eta_1}(x) = s_{2\eta_1}\sps{2r_1 + f}(x)$ for every $x\in X$. Therefore, for every $\xi:X'\to\{-1,1\}$, there exists $s\in S$ such that
\[
s_{2\eta_1}\sps{2r_1 + f}(x) = \xi(x) \text{ for every }x\in X',
\]
which implies that $X'$ is shattered by $\tilde S_{2\eta_1}\sps{2r_1 + f}$.

The fact that $X'$ is shattered by $\tilde B_{\eta_2}\sps {\theta}$ implies that for every $\xi:X'\to\{-1,1\}$, there exists $\tilde b\in \tilde B$ such that
\[
\tilde b_{\eta_2}\sps{\theta}(x) = \xi(x) \text{ for every }x\in X'.
\]
Define
\[
X'_1:=\{x\in X': \chi_{\bsigma}(f(x)) = 1\} \quad \text{and} \quad
X'_{-1}:=\{x\in X': \chi_\bsigma(f(x)) = -1\}.
\]
By the definition of $\tilde B$, for every $\tilde b\in \tilde B$, there exists $b\in B$ such that $\tilde b(x) = \chi_\bsigma(f(x))b(x)$ for every $x\in X$, which implies that $\tilde b(x) = b(x)$ for every $x\in X_1'$ and $\tilde b(x) = -b(x)$ for every $x\in X_{-1}'$.
Therefore, for every $\xi:X'\to\{-1,1\}$, there exists $b\in B$ such that
\begin{align*}
& b_{\eta_2}\sps{\theta}(x) = \xi(x) \text{ for every }x\in X_1', \text{ and}\\
& b_{\eta_2}\sps{-\theta}(x) = -\xi(x) \text{ for every }x\in X_{-1}'.
\end{align*}
This means that $X_1'$ is shattered by $B_{\eta_2}\sps \theta$ and $X_{-1}'$ is shattered by $B_{\eta_2}\sps {-\theta}$.

Now we have
\begin{align}
|X_1'| & \le \vc(S_{2\eta_1}\sps{2r_1 + f},B_{\eta_2}\sps \theta) \le {\sup}_{\theta'\in \bR}\vc(S_{2\eta_1}\sps{2r_1 + f},B_{\eta_2}\sps {\theta'})\label{eq:proof-ma-mc-vc-bound-2}\\
|X_{-1}'| & \le \vc(S_{2\eta_1}\sps{2r_1 + f},B_{\eta_2}\sps {-\theta}) \le {\sup}_{\theta'\in \bR}\vc(S_{2\eta_1}\sps{2r_1 + f},B_{\eta_2}\sps {\theta'}).\label{eq:proof-ma-mc-vc-bound-3}
\end{align}

It is clear that $X'_1$ and $X'_{-1}$ form a partition of $X'$, 
so $|X'| = |X'_1| + |X'_{-1}|$. Combining this with \eqref{eq:proof-ma-mc-vc-bound-2} and \eqref{eq:proof-ma-mc-vc-bound-3} proves \eqref{eq:proof-ma-mc-vc-bound-1}.
\end{proof}
\section{Proof of Theorem~\ref{thm:omni}}
\label{sec:proof-omni}
We recall \Cref{thm:omni}:
\thmomni*
\begin{lemma}
\label{lm:omni-helper}
In the setting of \Cref{thm:omni},
let $u,u'\in [-1,1]$ be two real numbers. Then
\[
\bE_{y\sim \ber^*(u)}[\ell(y,\tau(u'))]\le  \bE_{y\sim \ber^*(u)}[\ell(y,\tau(u))] + 2|u - u'|\kappa.
\]
\end{lemma}
\begin{proof}
Since $\ell(y,\cdot)$ is $\kappa$-Lipchitz, we can choose a function $g:\{-1,1\}\to\bR$ such that $\ell'(y,q):= \ell(y,q) + g(y)$ always lies in the interval $[-\kappa,\kappa]$ for every $y\in \{-1,1\}$ and $q\in [-1,1]$. (For example, choosing $g(y) = \ell(y,0)$ suffices.) Now for any $u\in [-1,1]$ and $q\in [-1,1]$,
\[
\bE_{y\sim \ber^*(u)}[\ell'(y,q)] = \bE_{y\sim \ber^*(u)}[\ell(y,q)] + \bE_{y\sim \ber^*(u)}[g(y)],
\]
which implies that for every $u,u'\in [-1,1]$,
\begin{align}
& \tau(u')\in {\arg\min}_{q\in [-1,1]}\bE_{y\sim \ber^*(u')}[\ell'(y,q)], \quad \textnormal{and} \label{eq:omni-helper-1}\\
& \bE_{y\sim \ber^*(u)}[\ell(y,\tau(u'))] - \bE_{y\sim \ber^*(u)}[\ell(y,\tau(u))]\notag\\
= {} & \bE_{y\sim \ber^*(u)}[\ell'(y,\tau(u'))] - \bE_{y\sim \ber^*(u)}[\ell'(y,\tau(u))]. \label{eq:omni-helper-2}
\end{align}
Also, for any $q\in [-1,1]$,
\begin{align}
& \left|\bE_{y\sim \ber^*(u)}[\ell'(y,q)] - \bE_{y\sim \ber^*(u')}[\ell'(y,q)]\right|\\
= {} & \left|\frac{1 + u}{2}\ell'(1,q) + \frac{1 - u}{2}\ell'(-1,q) - \frac{1 + u'}{2}\ell'(1,q) - \frac{1 - u'}{2}\ell'(-1,q)\right|\notag\\
= {} & \left|\frac 12 (u - u')(\ell'(1,q) - \ell'(-1,q))\right|\notag\\
\le {} & |u - u'|\kappa. \label{eq:omni-helper-3}
\end{align}
Therefore,
\begin{align*}
\bE_{y\sim \ber^*(u)}[\ell'(y,\tau (u'))] & \le \bE_{y\sim \ber^*(u')}[\ell'(y,\tau (u'))] + |u - u'|\kappa \tag{by \eqref{eq:omni-helper-3}}\\
& \le \bE_{y\sim \ber^*(u')}[\ell'(y,\tau (u))] + |u - u'|\kappa \tag{by \eqref{eq:omni-helper-1}}\\
& \le \bE_{y\sim \ber^*(u)}[\ell'(y,\tau (u))] + 2|u - u'|\kappa \tag{by \eqref{eq:omni-helper-3}}
\end{align*}
The proof is completed by combining the inequality above with \eqref{eq:omni-helper-2}.
\end{proof}
\begin{proof}[Proof of \Cref{thm:omni}]
Let $V$ be the range of $f$.
For every $v\in V$, define $X_v:= \{x\in X:f(x) = v\}$ and $y_v:= \bE_{(x,y)\sim\mu}[y|x\in X_v]$. Clearly, $(X_v)_{v\in V}$ partition $X$.
Define a function $f':X\rightarrow [-1,1]$ such that for every $x\in X_v$, $f'(x) = y_v$. 
The range of $f'$ is thus $V':=\{y_v:v\in V\}$.
We first show that $\mce_B(f')\le \alpha +  \varepsilon$. Indeed, for every $b\in B$,
\begin{align}
& \sum_{v'\in V'}|\bE_{(x,y)\sim\mu}[(y - f'(x))\one(f'(x) = v')b(x)]|\notag\\
= {} & \sum_{v'\in V'}\left|\sum_{v\in V: y_v = v'}\bE_{(x,y)\sim\mu}[(y - f'(x))\one(x\in X_v)b(x)]\right|\notag\\
\le {} & \sum_{v\in V}|\bE_{(x,y)\sim\mu}[(y - f'(x))\one(x\in X_v)b(x)]|\notag\\
\le {} & \sum_{v\in V}|\bE_{(x,y)\sim\mu}[(y - f(x))\one(x\in X_v)b(x)]| + \sum_{v\in V}|\bE_{(x,y)\sim\mu}[(f(x) - f'(x))\one(x\in X_v)b(x)]| \notag\\
\le {} & \alpha + \sum_{v\in V}|\bE_{(x,y)\sim\mu}[(f(x) - f'(x))\one(x\in X_v)b(x)]|\notag\\
\le {} & \alpha + \varepsilon.\label{eq:omni-constant}
\end{align}
Here, inequality \eqref{eq:omni-constant} holds because $f(x) = v$ and $f'(x) = y_v$ for every $x\in X_v$, which implies
\begin{align}
& \sum_{v\in V}|\bE_{(x,y)\sim\mu}[(f(x) - f'(x))\one(x\in X_v)b(x)]|\notag\\
= {} & \sum_{v\in V}\Big|(v - y_v)\bE_{(x,y)\sim\mu}[\one(x\in X_v)b(x)]\Big|\notag\\
\le {} & \sum_{v\in V}|v - y_v|{\Pr}_{(x,y)\sim\mu}[x\in X_v]\notag\\
= {} & \sum_{v\in V}|\bE_{(x,y)\sim\mu}[(f(x) - y)\one(x\in X_v)]|\tag{because $y_v = \bE_{(x,y)\sim\mu}[y|x\in X_v]$}\\
\le {} & \varepsilon.\label{eq:overall-cal-2}
\end{align}
By \citep[Theorem 19]{DBLP:conf/innovations/GopalanKRSW22}, we have 
\[
\bE_{(x,y)\sim\mu}[\ell(y,\tau(f'(x)))] \le \inf_{b\in B}\bE_{(x,y)\sim\mu}[\ell(y,b(x))] + (\alpha + \varepsilon)\kappa. 
\]
It remains to show that
\[
\bE_{(x,y)\sim\mu}[\ell(y,\tau(f(x)))] \le \bE_{(x,y)\sim\mu}[\ell(y,\tau(f'(x)))] + 2\varepsilon \kappa.
\]
By \eqref{eq:overall-cal-2}, it suffices to show that for every $v\in V$,
\[
\bE_{(x,y)\sim\mu}[\ell(y,\tau(f(x)))|x\in X_v] \le \bE_{(x,y)\sim\mu}[\ell(y,\tau(f'(x)))|x\in X_v] + 2|v - y_v| \kappa.
\]
Note again that $f(x) = v$ and $f'(x) = y_v$ for every $x\in X_v$, so the inequality above is equivalent to
\[
\bE_{(x,y)\sim\mu}[\ell(y,\tau(v))|x\in X_v] \le \bE_{(x,y)\sim\mu}[\ell(y,\tau(y_v))|x\in X_v] + 2|v - y_v|\kappa.
\]
Since $y_v = \bE_{(x,y)\sim\mu}[y|x\in X_v]$, the conditional distribution of $y$ given $x\in X_v$ is exactly $\ber^*(y_v)$, and thus the inequality above follows from \Cref{lm:omni-helper}.
\end{proof}
\section{Proof of Theorem~\ref{thm:compr}}
We recall \Cref{thm:compr}:
\thmcompr*
\label{sec:proof-compr}
Similarly to the definition of $\mce\sps \blambda$ in \eqref{eq:mcee-1}, for a partition $\blambda = (\Lambda_1,\ldots,\Lambda_k)$ of $[-1,1]$, a distribution $\mu$ over $X\times [-1,1]$, and a model $f:X\to [-1,1]$, we define
\begin{align*}
\ce\sps \blambda_{\mu}(f) & := \sum_{i=1}^k\sup_{\sigma\in \{-1,1\}}\bE_{(x,y)\sim\mu}\Big[(f(x) - y)\one(f(x)\in \Lambda_i)\sigma\Big]\\
& = \sup_{\bsigma\in\{-1,1\}^k}\bE_{(x,y)\sim\mu}\Big[(f(x) - y)\chi_\bsigma(f(x))\Big],
\end{align*}
where $\chi_\bsigma:[-1,1]\rightarrow\{-1,1\}$ is defined as in \Cref{sec:ma-mc-upper} such that $\chi_\bsigma(u) = \sigma_j$ when $u\in \Lambda_j$ for $\bsigma = (\sigma_1,\ldots,\sigma_k)\in\{-1,1\}^k$.

Our proof of \Cref{thm:compr} uses the following lemma:
\begin{lemma}
\label{lm:compr}
Let $S,B\subseteq([-1,1]\cup\{*\})^X$ be real-valued hypothesis classes. 
Suppose the parameters of \Cref{alg:omni} satisfy $\alpha,\gamma,\varepsilon\in \bR_{>0},\delta\in (0,1/2), W' > 4/\gamma^2, W > W + 4/\varepsilon^2$, 
\begin{align}
n\sps 1 & \ge \sup_{f:X\rightarrow[-1,1]}\sup_{\bsigma\in \{-1,1\}^k}\swcm((S - f)/2, B_{\bsigma,f},\alpha/2, \gamma/2, \delta/(4W')),\notag\\
n\sps 2 & \ge C\gamma^{-2}(k + \log(W'/\delta)),\notag\\
n\sps 3 & \ge C\varepsilon^{-2}(k + \log(W/\delta))\notag
\end{align}
for a sufficiently absolute large absolute constant $C>0$. 
Also, suppose the input data points to \Cref{alg:omni} are drawn i.i.d.\ from a distribution $\mu$ over $X\times\{-1,1\}$ satisfying ${\Pr}_{x\sim\mu|_X}[s(x)\ne *] = 1$ and $\bE_{(x,y)\sim\mu}[y|x] = s(x)$ for some $s\in S$.
Then with probability at least $1-\delta$, the output model $f$ of \Cref{alg:omni} satisfies
\[
\mce_{\mu,B}\sps \blambda(f) \le \alpha \quad \text{and} \quad \ce_{\mu}\sps \blambda(f)\le \varepsilon.
\]
\end{lemma}
\LinesNumbered
\begin{algorithm}
\caption{Omnipredictor for $(S,B)$}
\label{alg:omni}
\SetKwInOut{Parameters}{Parameters}
\SetKwInOut{Input}{Input}
\SetKwInOut{Output}{Output}
\Parameters{$S,B\subseteq([-1,1]\cup\{*\})^X$, $n,n\sps 1,n\sps 2, n\sps 3, W, W'\in \bZ_{> 0}$ satisfying $n = W(n\sps 1 + n\sps 2) + W'n\sps 3$, $\alpha,\gamma,\varepsilon,\delta \in \bR_{\ge 0}$, a partition $\blambda = (\Lambda_1,\ldots,\Lambda_k)$ of $[-1,1]$.}
\Input{data points $(x_1,y_1),\ldots,(x_n,y_n)\in X\times\{-1,1\}$.}
\Output{model $f:X\rightarrow [-1,1]$.}
Partition the input data points into $2W + W'$ datasets: $\Psi\sps {j',1} = \left(\left(x_i\sps {j',1}, y_i\sps {j',1}\right)\right)_{i=1}^{n\sps 1}$ and $\Psi \sps {j',2} = \left(\left(x_i\sps {j',2}, y_i\sps {j',2}\right)\right)_{i=1}^{n\sps 2}$ for $j' = 1,\ldots,W'$ and $\Psi\sps {j,3} = \left(\left(x_i\sps {j,3}, y_i\sps {j,3}\right)\right)_{i=1}^{n\sps 3}$ for $j = 1,\ldots,W$\;
Initialize $f:X\rightarrow[-1,1]$ to be the constant zero function: $f(x) = 0$ for every $x\in X$\;
$(j,j')\gets (1,1)$\;
\While{$j\le W$ and $j' \le W'$\label{line:omni-for-1}}
{
\eIf{there exists $\bsigma\in\{-1,1\}^k$ such that $Q_\bsigma:= \frac 1{n\sps 3}\sum_{i=1}^{n\sps 3}(y_i\sps {j,3} - f(x_i\sps{j,3}))\chi_{\bsigma}(f(x))\ge 3\varepsilon/4$}
{
Update $f(x)$ to $\proj_{[-1,1]}(f(x) + \varepsilon \chi_{\bsigma}(f(x))/2)$ for every $x\in X$\;
}
{
Define $\tilde y_i = (y_i\sps{j',1} - f(x_i\sps{j',1}))/2$ for $i = 1,\ldots, n\sps 1$ and define $\Psi' = ((x_i\sps{j',1},\tilde y_i))_{i=1}^{n\sps 1}$\label{line:omni-1}\;
\For{$\bsigma\in \{-1,1\}^k$}
{
Invoke learner $L\in \wcm_{n\sps 1}((S - f)/2,B_{\bsigma,f},\alpha/2,\gamma/2, \delta/(2W))$ on $\Psi'$ to obtain $f_\bsigma$\label{line:omni-invoke}\;
}
Choose $f'$ from $\{f_\bsigma:\bsigma\in \{-1,1\}^k\}$ that maximizes $Q_{f'}:=\frac{1}{n\sps 2}\sum_{i=1}^{n\sps 2}(y_i\sps {j',2}  - f(x_i\sps{j',2}))f'(x_i\sps {j',2})$\label{line:omni-4}\;
\eIf{$Q_{f'} \ge 3\gamma/4$\label{line:omni-if}}
{Update $f(x)$ to $\proj_{[-1,1]}(f(x) + \gamma f'(x)/2)$ for every $x\in X$\label{line:omni-2}\;}
{\Break\label{line:omni-3}\;}
$j'\gets j' + 1$\;
}
$j\gets j + 1$\;
}\label{line:omni-for-2}
\Return $f$\label{line:omni-return}\;
\end{algorithm}
We omit the proof of \Cref{lm:compr} because the proof is very similar to the proofs of \Cref{lm:ma-mc} and \Cref{lm:boosting}. Below we prove \Cref{thm:compr} using \Cref{lm:compr}.
\begin{proof}[Proof of \Cref{thm:compr}]
We choose $\gamma =\beta/16, \varepsilon = \beta/8,k = \lceil 8/\beta\rceil,\alpha = \beta/8 + 2\eta_1 + 4\eta_2, W' = \lfloor 4/\gamma^2\rfloor + 1, W = W' + \lfloor 4/\varepsilon^2\rfloor + 1$ in \Cref{lm:compr}. We also choose $\blambda$ as in \Cref{claim:mce-mcee}. Define $\eta_1' = \eta_1/2$. For $\bsigma\in \{-1,1\}^k$ and $f:X\to [-1,1]$, define $\tilde S = ((S - f)/2)$ and $\tilde B = B_{\bsigma,f}$. By \Cref{claim:ma-mc-vc-bound}, for every $\theta\in \bR$, $\vc(\tilde S_{\eta_1'}\sps 0, \tilde B_{\eta_2}\sps \theta)\le 2m$.
Therefore,
\begin{align*}
& \swcm((S - f)/2, B_{\bsigma,f},\alpha/2, \gamma/2, \delta/(2W))\\
\le {} & \scm((S - f)/2, B_{\bsigma,f},(\alpha - \gamma)/2, \delta/(2W)) \tag{by \Cref{claim:wcm-cm}}\\
= {} & \scm((S - f)/2, B_{\bsigma,f},\beta/32 + 2\eta_1' + 2\eta_2, \delta/(2W))\\
\le {} & O\left(
\frac{m}{\beta^4}\log^2_+\left(\frac{m}{\beta}\right)\log\left(\frac 1{\eta_1}\right) + \frac 1{\beta^4}\log\left(\frac 1{\eta_1}\right)\log\left(\frac 1{\beta\delta}\right) + \frac{1}{\beta^2}\log\left(\frac 1{\eta_2}\right) \right). \tag{by \Cref{thm:cm}}
\end{align*}
This means that the requirement of \Cref{lm:compr} can be satisfied by
\begin{align*}
& n\sps 1 \le O\left(
\frac{m}{\beta^4}\log^2_+\left(\frac{m}{\beta}\right)\log\left(\frac 1{\eta_1}\right) + \frac 1{\beta^4}\log\left(\frac 1{\eta_1}\right)\log\left(\frac 1{\beta\delta}\right) + \frac{1}{\beta^2}\log\left(\frac 1{\eta_2}\right) \right), \text{ and}\\
& n\sps 2, n\sps 3 \le O\left(\frac 1{\beta^3} + \frac 1{\beta^2}\log\left(\frac 1{\beta\delta}\right)\right).
\end{align*}
By \Cref{claim:mce-mcee}, the guarantees $\mce_{\mu,B}\sps\blambda (f)\le \alpha$ and $\ce_\mu\sps \blambda(f) \le\varepsilon$ imply that the output model $f$ of \Cref{alg:omni} can be easily transformed to $f'$ satisfying $\mce_{\mu,B}(f')\le \alpha + 1/k$ and $\ce_\mu(f')\le \varepsilon + 1/k$. By \Cref{thm:omni}, $f'$ can then be easily transformed to $\tau\circ f'$ which achieves the goal of comparative regression:
\begin{align*}
\bE_{(x,y)\sim\mu}[\ell(y,\tau(f(x)))] & \le \inf_{b\in B}\bE_{(x,y)\sim\mu}[\ell(y,b(x))] + (\alpha + 3\varepsilon + 4/k)\kappa\\
& \le \inf_{b\in B}\bE_{(x,y)\sim\mu}[\ell(y,b(x))] + (\beta + 2\eta_1 + 4\eta_2)\kappa.
\end{align*}
Since \Cref{alg:omni} takes $n = W'(n\sps 1 + n\sps 2) + Wn\sps 3$ data points, we have
\begin{align*}
& \scompr(S,B,\ell,\kappa(\beta + 2\eta_1 + 4\eta_2),\delta)\\
\le {} & W'(n\sps 1 + n\sps 2) + Wn\sps 3\\
\le {} & O\left(
\frac{m}{\beta^6}\log^2_+\left(\frac{m}{\beta}\right)\log\left(\frac 1{\eta_1}\right) + \frac 1{\beta^6}\log\left(\frac 1{\eta_1}\right)\log\left(\frac 1{\beta\delta}\right) + \frac{1}{\beta^4}\log\left(\frac 1{\eta_2}\right)   \right).
\qedhere
\end{align*}
\end{proof}
\bibliography{ref}

\begin{thebibliography}{79}
\providecommand{\natexlab}[1]{#1}
\providecommand{\url}[1]{\texttt{#1}}
\expandafter\ifx\csname urlstyle\endcsname\relax
  \providecommand{\doi}[1]{doi: #1}\else
  \providecommand{\doi}{doi: \begingroup \urlstyle{rm}\Url}\fi

\bibitem[Alon et~al.(1993)Alon, Cesa-Bianchi, Ben-David, and
  Haussler]{MR1328428}
Noga Alon, Nicol\`o Cesa-Bianchi, Shai Ben-David, and David Haussler.
\newblock Scale-sensitive dimensions, uniform convergence, and learnability.
\newblock In \emph{34th {A}nnual {S}ymposium on {F}oundations of {C}omputer
  {S}cience ({P}alo {A}lto, {CA}, 1993)}, pages 292--301. IEEE Comput. Soc.
  Press, Los Alamitos, CA, 1993.
\newblock \doi{10.1109/SFCS.1993.366858}.
\newblock URL \url{https://doi.org/10.1109/SFCS.1993.366858}.

\bibitem[Alon et~al.(2019)Alon, Livni, Malliaris, and Moran]{MR4003389}
Noga Alon, Roi Livni, Maryanthe Malliaris, and Shay Moran.
\newblock Private {PAC} learning implies finite {L}ittlestone dimension.
\newblock In \emph{S{TOC}'19---{P}roceedings of the 51st {A}nnual {ACM}
  {SIGACT} {S}ymposium on {T}heory of {C}omputing}, pages 852--860. ACM, New
  York, 2019.
\newblock \doi{10.1145/3313276.3316312}.
\newblock URL \url{https://doi.org/10.1145/3313276.3316312}.

\bibitem[Alon et~al.(2021)Alon, Ben-Eliezer, Dagan, Moran, Naor, and
  Yogev]{MR4398855}
Noga Alon, Omri Ben-Eliezer, Yuval Dagan, Shay Moran, Moni Naor, and Eylon
  Yogev.
\newblock Adversarial laws of large numbers and optimal regret in online
  classification.
\newblock In \emph{S{TOC} '21---{P}roceedings of the 53rd {A}nnual {ACM}
  {SIGACT} {S}ymposium on {T}heory of {C}omputing}, pages 447--455. ACM, New
  York, 2021.
\newblock \doi{10.1145/3406325.3451041}.
\newblock URL \url{https://doi.org/10.1145/3406325.3451041}.

\bibitem[Alon et~al.(2022)Alon, Hanneke, Holzman, and Moran]{MR4399723}
Noga Alon, Steve Hanneke, Ron Holzman, and Shay Moran.
\newblock A theory of {PAC} learnability of partial concept classes.
\newblock In \emph{2021 {IEEE} 62nd {A}nnual {S}ymposium on {F}oundations of
  {C}omputer {S}cience---{FOCS} 2021}, pages 658--671. IEEE Computer Soc., Los
  Alamitos, CA, 2022.
\newblock \doi{10.1109/FOCS52979.2021.00070}.
\newblock URL \url{https://doi.org/10.1109/FOCS52979.2021.00070}.

\bibitem[Artstein et~al.(2004{\natexlab{a}})Artstein, Milman, Szarek, and
  Tomczak-Jaegermann]{MR2105957}
S.~Artstein, V.~Milman, S.~Szarek, and N.~Tomczak-Jaegermann.
\newblock On convexified packing and entropy duality.
\newblock \emph{Geom. Funct. Anal.}, 14\penalty0 (5):\penalty0 1134--1141,
  2004{\natexlab{a}}.
\newblock ISSN 1016-443X.
\newblock \doi{10.1007/s00039-004-0486-3}.
\newblock URL \url{https://doi.org/10.1007/s00039-004-0486-3}.

\bibitem[Artstein et~al.(2004{\natexlab{b}})Artstein, Milman, and
  Szarek]{MR2113023}
S.~Artstein, V.~Milman, and S.~J. Szarek.
\newblock Duality of metric entropy.
\newblock \emph{Ann. of Math. (2)}, 159\penalty0 (3):\penalty0 1313--1328,
  2004{\natexlab{b}}.
\newblock ISSN 0003-486X.
\newblock \doi{10.4007/annals.2004.159.1313}.
\newblock URL \url{https://doi.org/10.4007/annals.2004.159.1313}.

\bibitem[Balcan et~al.(2009)Balcan, Beygelzimer, and Langford]{MR2472318}
Maria-Florina Balcan, Alina Beygelzimer, and John Langford.
\newblock Agnostic active learning.
\newblock \emph{J. Comput. System Sci.}, 75\penalty0 (1):\penalty0 78--89,
  2009.
\newblock ISSN 0022-0000.
\newblock \doi{10.1016/j.jcss.2008.07.003}.
\newblock URL \url{https://doi.org/10.1016/j.jcss.2008.07.003}.

\bibitem[Balcan et~al.(2010)Balcan, Hanneke, and Vaughan]{MR3108162}
Maria-Florina Balcan, Steve Hanneke, and Jennifer~Wortman Vaughan.
\newblock The true sample complexity of active learning.
\newblock \emph{Mach. Learn.}, 80\penalty0 (2-3):\penalty0 111--139, 2010.
\newblock ISSN 0885-6125.
\newblock \doi{10.1007/s10994-010-5174-y}.
\newblock URL \url{https://doi.org/10.1007/s10994-010-5174-y}.

\bibitem[Bartlett and Long(1995)]{bartlett1995more}
Peter~L Bartlett and Philip~M Long.
\newblock More theorems about scale-sensitive dimensions and learning.
\newblock In \emph{Proceedings of the eighth annual conference on Computational
  learning theory}, pages 392--401, 1995.

\bibitem[Bartlett and Long(1998)]{MR1629694}
Peter~L. Bartlett and Philip~M. Long.
\newblock Prediction, learning, uniform convergence, and scale-sensitive
  dimensions.
\newblock volume~56, pages 174--190. 1998.
\newblock \doi{10.1006/jcss.1997.1557}.
\newblock URL \url{https://doi.org/10.1006/jcss.1997.1557}.
\newblock Eighth Annual Workshop on Computational Learning Theory (COLT) (Santa
  Cruz, CA, 1995).

\bibitem[Bartlett et~al.(1996)Bartlett, Long, and Williamson]{MR1408000}
Peter~L. Bartlett, Philip~M. Long, and Robert~C. Williamson.
\newblock Fat-shattering and the learnability of real-valued functions.
\newblock volume~52, pages 434--452. 1996.
\newblock \doi{10.1006/jcss.1996.0033}.
\newblock URL \url{https://doi.org/10.1006/jcss.1996.0033}.
\newblock Seventh Annual Workshop on Computational Learning Theory (COLT) (New
  Brunswick, NJ, 1994).

\bibitem[Ben-David et~al.(1995)Ben-David, Cesa-Bianchi, Haussler, and
  Long]{MR1322634}
Shai Ben-David, Nicol\`o Cesa-Bianchi, David Haussler, and Philip~M. Long.
\newblock Characterizations of learnability for classes of
  {$\{0,\cdots,n\}$}-valued functions.
\newblock \emph{J. Comput. System Sci.}, 50\penalty0 (1):\penalty0 74--86,
  1995.
\newblock ISSN 0022-0000.
\newblock \doi{10.1006/jcss.1995.1008}.
\newblock URL \url{https://doi.org/10.1006/jcss.1995.1008}.

\bibitem[Ben-David et~al.(2009)Ben-David, P{\'a}l, and
  Shalev-Shwartz]{ben2009agnostic}
Shai Ben-David, D{\'a}vid P{\'a}l, and Shai Shalev-Shwartz.
\newblock Agnostic online learning.
\newblock In \emph{COLT}, volume~3, page~1, 2009.

\bibitem[Blais et~al.(2021)Blais, Ferreira~Pinto, and Harms]{MR4398860}
Eric Blais, Renato Ferreira~Pinto, Jr., and Nathaniel Harms.
\newblock V{C} dimension and distribution-free sample-based testing.
\newblock In \emph{S{TOC} '21---{P}roceedings of the 53rd {A}nnual {ACM}
  {SIGACT} {S}ymposium on {T}heory of {C}omputing}, pages 504--517. ACM, New
  York, 2021.
\newblock \doi{10.1145/3406325.3451104}.
\newblock URL \url{https://doi.org/10.1145/3406325.3451104}.

\bibitem[Blum and Lykouris(2020)]{DBLP:conf/innovations/BlumL20}
Avrim Blum and Thodoris Lykouris.
\newblock Advancing subgroup fairness via sleeping experts.
\newblock In Thomas Vidick, editor, \emph{11th Innovations in Theoretical
  Computer Science Conference, {ITCS} 2020, January 12-14, 2020, Seattle,
  Washington, {USA}}, volume 151 of \emph{LIPIcs}, pages 55:1--55:24. Schloss
  Dagstuhl - Leibniz-Zentrum f{\"{u}}r Informatik, 2020.
\newblock \doi{10.4230/LIPIcs.ITCS.2020.55}.
\newblock URL \url{https://doi.org/10.4230/LIPIcs.ITCS.2020.55}.

\bibitem[Blum et~al.(1994)Blum, Furst, Jackson, Kearns, Mansour, and
  Rudich]{blum1994weakly}
Avrim Blum, Merrick Furst, Jeffrey Jackson, Michael Kearns, Yishay Mansour, and
  Steven Rudich.
\newblock Weakly learning dnf and characterizing statistical query learning
  using fourier analysis.
\newblock In \emph{Proceedings of the twenty-sixth annual ACM symposium on
  Theory of computing}, pages 253--262, 1994.

\bibitem[Blumer et~al.(1989)Blumer, Ehrenfeucht, Haussler, and
  Warmuth]{MR1072253}
Anselm Blumer, Andrzej Ehrenfeucht, David Haussler, and Manfred~K. Warmuth.
\newblock Learnability and the {V}apnik-{C}hervonenkis dimension.
\newblock \emph{J. Assoc. Comput. Mach.}, 36\penalty0 (4):\penalty0 929--965,
  1989.
\newblock ISSN 0004-5411.
\newblock \doi{10.1145/76359.76371}.
\newblock URL \url{https://doi.org/10.1145/76359.76371}.

\bibitem[Bourgain et~al.(1989)Bourgain, Pajor, Szarek, and
  Tomczak-Jaegermann]{MR1008716}
J.~Bourgain, A.~Pajor, S.~J. Szarek, and N.~Tomczak-Jaegermann.
\newblock On the duality problem for entropy numbers of operators.
\newblock In \emph{Geometric aspects of functional analysis (1987--88)}, volume
  1376 of \emph{Lecture Notes in Math.}, pages 50--63. Springer, Berlin, 1989.
\newblock \doi{10.1007/BFb0090048}.
\newblock URL \url{https://doi.org/10.1007/BFb0090048}.

\bibitem[Brukhim et~al.(2022)Brukhim, Carmon, Dinur, Moran, and
  Yehudayoff]{brukhim2022characterization}
Nataly Brukhim, Daniel Carmon, Irit Dinur, Shay Moran, and Amir Yehudayoff.
\newblock A characterization of multiclass learnability.
\newblock \emph{arXiv preprint arXiv:2203.01550}, 2022.

\bibitem[Bun et~al.(2020)Bun, Livni, and Moran]{MR4232052}
Mark Bun, Roi Livni, and Shay Moran.
\newblock An equivalence between private classification and online prediction.
\newblock In \emph{2020 {IEEE} 61st {A}nnual {S}ymposium on {F}oundations of
  {C}omputer {S}cience}, pages 389--402. IEEE Computer Soc., Los Alamitos, CA,
  2020.
\newblock \doi{10.1109/FOCS46700.2020.00044}.
\newblock URL \url{https://doi.org/10.1109/FOCS46700.2020.00044}.

\bibitem[David et~al.(2016)David, Moran, and Yehudayoff]{david2016supervised}
Ofir David, Shay Moran, and Amir Yehudayoff.
\newblock Supervised learning through the lens of compression.
\newblock \emph{Advances in Neural Information Processing Systems}, 29, 2016.

\bibitem[Diana et~al.(2022)Diana, Gill, Kearns, Kenthapadi, Roth, and
  Sharifi-Malvajerdi]{diana2022proxies}
Emily Diana, Wesley Gill, Michael Kearns, Krishnaram Kenthapadi, Aaron Roth,
  and Saeed Sharifi-Malvajerdi.
\newblock Multiaccurate proxies for downstream fairness.
\newblock In \emph{2022 ACM Conference on Fairness, Accountability, and
  Transparency}, FAccT '22, page 1207–1239, New York, NY, USA, 2022.
  Association for Computing Machinery.
\newblock ISBN 9781450393522.
\newblock \doi{10.1145/3531146.3533180}.
\newblock URL \url{https://doi.org/10.1145/3531146.3533180}.

\bibitem[Dwork et~al.(2021)Dwork, Kim, Reingold, Rothblum, and Yona]{MR4398905}
Cynthia Dwork, Michael~P. Kim, Omer Reingold, Guy~N. Rothblum, and Gal Yona.
\newblock Outcome indistinguishability.
\newblock In \emph{S{TOC} '21---{P}roceedings of the 53rd {A}nnual {ACM}
  {SIGACT} {S}ymposium on {T}heory of {C}omputing}, pages 1095--1108. ACM, New
  York, 2021.
\newblock \doi{10.1145/3406325.3451064}.
\newblock URL \url{https://doi.org/10.1145/3406325.3451064}.

\bibitem[Dwork et~al.(2022)Dwork, Kim, Reingold, Rothblum, and
  Yona]{dwork2022beyond}
Cynthia Dwork, Michael~P Kim, Omer Reingold, Guy~N Rothblum, and Gal Yona.
\newblock Beyond bernoulli: Generating random outcomes that cannot be
  distinguished from nature.
\newblock In \emph{International Conference on Algorithmic Learning Theory},
  pages 342--380. PMLR, 2022.

\bibitem[Feldman(2010)]{DBLP:conf/innovations/Feldman10}
Vitaly Feldman.
\newblock Distribution-specific agnostic boosting.
\newblock In Andrew~Chi{-}Chih Yao, editor, \emph{Innovations in Computer
  Science - {ICS} 2010, Tsinghua University, Beijing, China, January 5-7, 2010.
  Proceedings}, pages 241--250. Tsinghua University Press, 2010.
\newblock URL
  \url{http://conference.iiis.tsinghua.edu.cn/ICS2010/content/papers/20.html}.

\bibitem[Filmus et~al.(2022)Filmus, Hanneke, Mehalel, and
  Moran]{filmus2022optimal}
Yuval Filmus, Steve Hanneke, Idan Mehalel, and Shay Moran.
\newblock Optimal prediction using expert advice and randomized {L}ittlestone
  dimension.
\newblock 2022.

\bibitem[Ghazi et~al.(2021)Ghazi, Golowich, Kumar, and Manurangsi]{MR4398834}
Badih Ghazi, Noah Golowich, Ravi Kumar, and Pasin Manurangsi.
\newblock Sample-efficient proper {PAC} learning with approximate differential
  privacy.
\newblock In \emph{S{TOC} '21---{P}roceedings of the 53rd {A}nnual {ACM}
  {SIGACT} {S}ymposium on {T}heory of {C}omputing}, pages 183--196. ACM, New
  York, 2021.
\newblock \doi{10.1145/3406325.3451028}.
\newblock URL \url{https://doi.org/10.1145/3406325.3451028}.

\bibitem[Gilbert(1952)]{gilbert1952comparison}
Edgar~N Gilbert.
\newblock A comparison of signalling alphabets.
\newblock \emph{The Bell system technical journal}, 31\penalty0 (3):\penalty0
  504--522, 1952.

\bibitem[Globus-Harris et~al.(2022{\natexlab{a}})Globus-Harris, Gupta, Jung,
  Kearns, Morgenstern, and Roth]{globus2022multicalibrated}
Ira Globus-Harris, Varun Gupta, Christopher Jung, Michael Kearns, Jamie
  Morgenstern, and Aaron Roth.
\newblock Multicalibrated regression for downstream fairness.
\newblock \emph{arXiv preprint arXiv:2209.07312}, 2022{\natexlab{a}}.

\bibitem[Globus-Harris et~al.(2022{\natexlab{b}})Globus-Harris, Kearns, and
  Roth]{10.1145/3531146.3533172}
Ira Globus-Harris, Michael Kearns, and Aaron Roth.
\newblock An algorithmic framework for bias bounties.
\newblock In \emph{2022 ACM Conference on Fairness, Accountability, and
  Transparency}, FAccT '22, page 1106–1124, New York, NY, USA,
  2022{\natexlab{b}}. Association for Computing Machinery.
\newblock ISBN 9781450393522.
\newblock \doi{10.1145/3531146.3533172}.
\newblock URL \url{https://doi.org/10.1145/3531146.3533172}.

\bibitem[Goldreich et~al.(1996)Goldreich, Goldwasser, and Ron]{MR1450632}
Oded Goldreich, Shafi Goldwasser, and Dana Ron.
\newblock Property testing and its connection to learning and approximation.
\newblock In \emph{37th {A}nnual {S}ymposium on {F}oundations of {C}omputer
  {S}cience ({B}urlington, {VT}, 1996)}, pages 339--348. IEEE Comput. Soc.
  Press, Los Alamitos, CA, 1996.
\newblock \doi{10.1109/SFCS.1996.548493}.
\newblock URL \url{https://doi.org/10.1109/SFCS.1996.548493}.

\bibitem[Golowich(2021)]{golowich2021differentially}
Noah Golowich.
\newblock Differentially private nonparametric regression under a growth
  condition.
\newblock In \emph{Conference on Learning Theory}, pages 2149--2192. PMLR,
  2021.

\bibitem[Gonen et~al.(2020)Gonen, Lovett, and Moshkovitz]{gonen2020towards}
Alon Gonen, Shachar Lovett, and Michal Moshkovitz.
\newblock Towards a combinatorial characterization of bounded-memory learning.
\newblock \emph{Advances in Neural Information Processing Systems},
  33:\penalty0 9028--9038, 2020.

\bibitem[Gopalan et~al.(2022{\natexlab{a}})Gopalan, Hu, Kim, Reingold, and
  Wieder]{gopalan2022loss}
Parikshit Gopalan, Lunjia Hu, Michael~P Kim, Omer Reingold, and Udi Wieder.
\newblock Loss minimization through the lens of outcome indistinguishability.
\newblock \emph{arXiv preprint arXiv:2210.08649}, 2022{\natexlab{a}}.

\bibitem[Gopalan et~al.(2022{\natexlab{b}})Gopalan, Kalai, Reingold, Sharan,
  and Wieder]{DBLP:conf/innovations/GopalanKRSW22}
Parikshit Gopalan, Adam~Tauman Kalai, Omer Reingold, Vatsal Sharan, and Udi
  Wieder.
\newblock Omnipredictors.
\newblock In Mark Braverman, editor, \emph{13th Innovations in Theoretical
  Computer Science Conference, {ITCS} 2022, January 31 - February 3, 2022,
  Berkeley, CA, {USA}}, volume 215 of \emph{LIPIcs}, pages 79:1--79:21. Schloss
  Dagstuhl - Leibniz-Zentrum f{\"{u}}r Informatik, 2022{\natexlab{b}}.
\newblock \doi{10.4230/LIPIcs.ITCS.2022.79}.
\newblock URL \url{https://doi.org/10.4230/LIPIcs.ITCS.2022.79}.

\bibitem[Gopalan et~al.(2022{\natexlab{c}})Gopalan, Kim, Singhal, and
  Zhao]{gopalan2022low}
Parikshit Gopalan, Michael~P Kim, Mihir~A Singhal, and Shengjia Zhao.
\newblock Low-degree multicalibration.
\newblock In \emph{Conference on Learning Theory}, pages 3193--3234. PMLR,
  2022{\natexlab{c}}.

\bibitem[Gopalan et~al.(2022{\natexlab{d}})Gopalan, Reingold, Sharan, and
  Wieder]{gopalan2022multicalibrated}
Parikshit Gopalan, Omer Reingold, Vatsal Sharan, and Udi Wieder.
\newblock Multicalibrated partitions for importance weights.
\newblock In \emph{International Conference on Algorithmic Learning Theory},
  pages 408--435. PMLR, 2022{\natexlab{d}}.

\bibitem[Graepel et~al.(2005)Graepel, Herbrich, and
  Shawe-Taylor]{graepel2005pac}
Thore Graepel, Ralf Herbrich, and John Shawe-Taylor.
\newblock Pac-bayesian compression bounds on the prediction error of learning
  algorithms for classification.
\newblock \emph{Machine Learning}, 59\penalty0 (1):\penalty0 55--76, 2005.

\bibitem[Haghtalab et~al.(2020)Haghtalab, Roughgarden, and
  Shetty]{haghtalab2020smoothed}
Nika Haghtalab, Tim Roughgarden, and Abhishek Shetty.
\newblock Smoothed analysis of online and differentially private learning.
\newblock \emph{Advances in Neural Information Processing Systems},
  33:\penalty0 9203--9215, 2020.

\bibitem[Haghtalab et~al.(2022)Haghtalab, Roughgarden, and Shetty]{MR4399748}
Nika Haghtalab, Tim Roughgarden, and Abhishek Shetty.
\newblock Smoothed analysis with adaptive adversaries.
\newblock In \emph{2021 {IEEE} 62nd {A}nnual {S}ymposium on {F}oundations of
  {C}omputer {S}cience---{FOCS} 2021}, pages 942--953. IEEE Computer Soc., Los
  Alamitos, CA, 2022.
\newblock \doi{10.1109/FOCS52979.2021.00095}.
\newblock URL \url{https://doi.org/10.1109/FOCS52979.2021.00095}.

\bibitem[Haussler et~al.(1989)Haussler, Littlestone, and Warmuth]{MR1109741}
D.~Haussler, N.~Littlestone, and M.~K. Warmuth.
\newblock Predicting {$\{0,1\}$}-functions on randomly drawn points (extended
  abstracts).
\newblock In \emph{Proceedings of the 1988 {W}orkshop on {C}omputational
  {L}earning {T}heory ({C}ambridge, {MA}, 1988)}, pages 280--296. Morgan
  Kaufmann, San Mateo, CA, 1989.
\newblock \doi{10.1016/0315-0860(89)90031-1}.
\newblock URL \url{https://doi.org/10.1016/0315-0860(89)90031-1}.

\bibitem[H{\'e}bert-Johnson et~al.(2018)H{\'e}bert-Johnson, Kim, Reingold, and
  Rothblum]{hebert2018multicalibration}
Ursula H{\'e}bert-Johnson, Michael Kim, Omer Reingold, and Guy Rothblum.
\newblock Multicalibration: Calibration for the (computationally-identifiable)
  masses.
\newblock In \emph{International Conference on Machine Learning}, pages
  1939--1948. PMLR, 2018.

\bibitem[Hopkins et~al.(2020{\natexlab{a}})Hopkins, Kane, and
  Lovett]{hopkins2020power}
Max Hopkins, Daniel Kane, and Shachar Lovett.
\newblock The power of comparisons for actively learning linear classifiers.
\newblock \emph{Advances in Neural Information Processing Systems},
  33:\penalty0 6342--6353, 2020{\natexlab{a}}.

\bibitem[Hopkins et~al.(2020{\natexlab{b}})Hopkins, Kane, Lovett, and
  Mahajan]{MR4232108}
Max Hopkins, Daniel Kane, Shachar Lovett, and Gaurav Mahajan.
\newblock Point location and active learning: learning halfspaces almost
  optimally.
\newblock In \emph{2020 {IEEE} 61st {A}nnual {S}ymposium on {F}oundations of
  {C}omputer {S}cience}, pages 1034--1044. IEEE Computer Soc., Los Alamitos,
  CA, 2020{\natexlab{b}}.
\newblock \doi{10.1109/FOCS46700.2020.00100}.
\newblock URL \url{https://doi.org/10.1109/FOCS46700.2020.00100}.

\bibitem[Hopkins et~al.(2020{\natexlab{c}})Hopkins, Kane, Lovett, and
  Mahajan]{hopkins2020noise}
Max Hopkins, Daniel Kane, Shachar Lovett, and Gaurav Mahajan.
\newblock Noise-tolerant, reliable active classification with comparison
  queries.
\newblock In \emph{Conference on Learning Theory}, pages 1957--2006. PMLR,
  2020{\natexlab{c}}.

\bibitem[Hopkins et~al.(2022)Hopkins, Kane, Lovett, and
  Mahajan]{hopkins2022realizable}
Max Hopkins, Daniel~M Kane, Shachar Lovett, and Gaurav Mahajan.
\newblock Realizable learning is all you need.
\newblock In \emph{Conference on Learning Theory}, pages 3015--3069. PMLR,
  2022.

\bibitem[Hu et~al.(2022{\natexlab{a}})Hu, Livni-Navon, Reingold, and
  Yang]{hu2022omnipredictors}
Lunjia Hu, Inbal Livni-Navon, Omer Reingold, and Chutong Yang.
\newblock Omnipredictors for constrained optimization.
\newblock \emph{arXiv preprint arXiv:2209.07463}, 2022{\natexlab{a}}.

\bibitem[Hu et~al.(2022{\natexlab{b}})Hu, Peale, and Reingold]{hu2022metric}
Lunjia Hu, Charlotte Peale, and Omer Reingold.
\newblock Metric entropy duality and the sample complexity of outcome
  indistinguishability.
\newblock In \emph{International Conference on Algorithmic Learning Theory},
  pages 515--552. PMLR, 2022{\natexlab{b}}.

\bibitem[Jung et~al.(2021)Jung, Lee, Pai, Roth, and Vohra]{jung2021moment}
Christopher Jung, Changhwa Lee, Mallesh Pai, Aaron Roth, and Rakesh Vohra.
\newblock Moment multicalibration for uncertainty estimation.
\newblock In \emph{Conference on Learning Theory}, pages 2634--2678. PMLR,
  2021.

\bibitem[Jung et~al.(2020)Jung, Kim, and Tewari]{jung2020equivalence}
Young Jung, Baekjin Kim, and Ambuj Tewari.
\newblock On the equivalence between online and private learnability beyond
  binary classification.
\newblock \emph{Advances in Neural Information Processing Systems},
  33:\penalty0 16701--16710, 2020.

\bibitem[Kalai et~al.(2008)Kalai, Mansour, and Verbin]{MR2582918}
Adam~Tauman Kalai, Yishay Mansour, and Elad Verbin.
\newblock On agnostic boosting and parity learning.
\newblock In \emph{S{TOC}'08}, pages 629--638. ACM, New York, 2008.
\newblock \doi{10.1145/1374376.1374466}.
\newblock URL \url{https://doi.org/10.1145/1374376.1374466}.

\bibitem[Kane et~al.(2017)Kane, Lovett, Moran, and Zhang]{MR3734243}
Daniel~M. Kane, Shachar Lovett, Shay Moran, and Jiapeng Zhang.
\newblock Active classification with comparison queries.
\newblock In \emph{58th {A}nnual {IEEE} {S}ymposium on {F}oundations of
  {C}omputer {S}cience---{FOCS} 2017}, pages 355--366. IEEE Computer Soc., Los
  Alamitos, CA, 2017.
\newblock \doi{10.1109/FOCS.2017.40}.
\newblock URL \url{https://doi.org/10.1109/FOCS.2017.40}.

\bibitem[Kearns(1993)]{kearns1993efficient}
Michael Kearns.
\newblock Efficient noise-tolerant learning from statistical queries.
\newblock In \emph{Proceedings of the twenty-fifth annual ACM symposium on
  Theory of Computing}, pages 392--401, 1993.

\bibitem[Kearns and Ron(2000)]{MR1800309}
Michael Kearns and Dana Ron.
\newblock Testing problems with sublearning sample complexity.
\newblock \emph{J. Comput. System Sci.}, 61\penalty0 (3):\penalty0 428--456,
  2000.
\newblock ISSN 0022-0000.
\newblock \doi{10.1006/jcss.1999.1656}.
\newblock URL \url{https://doi.org/10.1006/jcss.1999.1656}.

\bibitem[Kearns et~al.(2018)Kearns, Neel, Roth, and Wu]{kearns2018preventing}
Michael Kearns, Seth Neel, Aaron Roth, and Zhiwei~Steven Wu.
\newblock Preventing fairness gerrymandering: Auditing and learning for
  subgroup fairness.
\newblock In \emph{International Conference on Machine Learning}, pages
  2564--2572. PMLR, 2018.

\bibitem[Kearns and Schapire(1994)]{MR1279411}
Michael~J. Kearns and Robert~E. Schapire.
\newblock Efficient distribution-free learning of probabilistic concepts.
\newblock volume~48, pages 464--497. 1994.
\newblock \doi{10.1016/S0022-0000(05)80062-5}.
\newblock URL \url{https://doi.org/10.1016/S0022-0000(05)80062-5}.
\newblock 31st Annual Symposium on Foundations of Computer Science (FOCS) (St.
  Louis, MO, 1990).

\bibitem[Kim et~al.(2019)Kim, Ghorbani, and Zou]{kim2019multiaccuracy}
Michael~P Kim, Amirata Ghorbani, and James Zou.
\newblock Multiaccuracy: Black-box post-processing for fairness in
  classification.
\newblock In \emph{Proceedings of the 2019 AAAI/ACM Conference on AI, Ethics,
  and Society}, pages 247--254, 2019.

\bibitem[Kim et~al.(2022)Kim, Kern, Goldwasser, Kreuter, and
  Reingold]{kim2022universal}
Michael~P Kim, Christoph Kern, Shafi Goldwasser, Frauke Kreuter, and Omer
  Reingold.
\newblock Universal adaptability: Target-independent inference that competes
  with propensity scoring.
\newblock \emph{Proceedings of the National Academy of Sciences}, 119\penalty0
  (4):\penalty0 e2108097119, 2022.

\bibitem[Kivinen(1995)]{MR1303564}
J.~Kivinen.
\newblock Learning reliably and with one-sided error.
\newblock \emph{Math. Systems Theory}, 28\penalty0 (2):\penalty0 141--172,
  1995.
\newblock ISSN 0025-5661.
\newblock \doi{10.1007/BF01191474}.
\newblock URL \url{https://doi.org/10.1007/BF01191474}.

\bibitem[Kivinen(1989)]{MR1076241}
Jyrki Kivinen.
\newblock Reliable and useful learning.
\newblock In \emph{Proceedings of the {S}econd {A}nnual {W}orkshop on
  {C}omputational {L}earning {T}heory ({S}anta {C}ruz, {CA}, 1989)}, pages
  365--380. Morgan Kaufmann, San Mateo, CA, 1989.

\bibitem[Kivinen(1990)]{kivinen1990reliable}
Jyrki Kivinen.
\newblock Reliable and useful learning with uniform probability distributions.
\newblock In \emph{ALT}, pages 209--222, 1990.

\bibitem[Li et~al.(2000)Li, Long, and Srinivasan]{MR1754872}
Yi~Li, Philip~M. Long, and Aravind Srinivasan.
\newblock Improved bounds on the sample complexity of learning.
\newblock In \emph{Proceedings of the {E}leventh {A}nnual {ACM}-{SIAM}
  {S}ymposium on {D}iscrete {A}lgorithms ({S}an {F}rancisco, {CA}, 2000)},
  pages 309--318. ACM, New York, 2000.

\bibitem[Linial et~al.(1991)Linial, Mansour, and Rivest]{MR1088804}
Nathan Linial, Yishay Mansour, and Ronald~L. Rivest.
\newblock Results on learnability and the {V}apnik-{C}hervonenkis dimension.
\newblock \emph{Inform. and Comput.}, 90\penalty0 (1):\penalty0 33--49, 1991.
\newblock ISSN 0890-5401.
\newblock \doi{10.1016/0890-5401(91)90058-A}.
\newblock URL \url{https://doi.org/10.1016/0890-5401(91)90058-A}.

\bibitem[Littlestone(1988)]{littlestone1988learning}
Nick Littlestone.
\newblock Learning quickly when irrelevant attributes abound: A new
  linear-threshold algorithm.
\newblock \emph{Machine learning}, 2\penalty0 (4):\penalty0 285--318, 1988.

\bibitem[Littlestone and Warmuth(1986)]{littlestone1986relating}
Nick Littlestone and Manfred Warmuth.
\newblock Relating data compression and learnability.
\newblock 1986.

\bibitem[Long(2001)]{MR2042042}
Philip~M. Long.
\newblock On agnostic learning with {$\{0,*,1\}$}-valued and real-valued
  hypotheses.
\newblock In \emph{Computational learning theory ({A}msterdam, 2001)}, volume
  2111 of \emph{Lecture Notes in Comput. Sci.}, pages 289--302. Springer,
  Berlin, 2001.
\newblock \doi{10.1007/3-540-44581-1\_19}.
\newblock URL \url{https://doi.org/10.1007/3-540-44581-1_19}.

\bibitem[Milman(2007)]{MR2296760}
Emanuel Milman.
\newblock A remark on two duality relations.
\newblock \emph{Integral Equations Operator Theory}, 57\penalty0 (2):\penalty0
  217--228, 2007.
\newblock ISSN 0378-620X.
\newblock \doi{10.1007/s00020-006-1479-4}.
\newblock URL \url{https://doi.org/10.1007/s00020-006-1479-4}.

\bibitem[Pietsch(1972)]{MR0361822}
Albrecht Pietsch.
\newblock \emph{Theorie der {O}peratorenideale ({Z}usammenfassung)}.
\newblock Wissenschaftliche Beitr\"{a}ge der Friedrich-Schiller-Universit\"{a}t
  Jena. Friedrich-Schiller-Universit\"{a}t, Jena, 1972.

\bibitem[Rivest and Sloan(1989)]{MR1109724}
Ronald~L. Rivest and Robert Sloan.
\newblock Learning complicated concepts reliably and usefully (extended
  abstract).
\newblock In \emph{Proceedings of the 1988 {W}orkshop on {C}omputational
  {L}earning {T}heory ({C}ambridge, {MA}, 1988)}, pages 69--79. Morgan
  Kaufmann, San Mateo, CA, 1989.

\bibitem[Rosenberg et~al.(2022)Rosenberg, Bhattacharjee, Fawaz, and
  Jha]{rosenberg2022exploration}
Harrison Rosenberg, Robi Bhattacharjee, Kassem Fawaz, and Somesh Jha.
\newblock An exploration of multicalibration uniform convergence bounds.
\newblock \emph{arXiv preprint arXiv:2202.04530}, 2022.

\bibitem[Rothblum and Yona(2021)]{rothblum2021multi}
Guy~N Rothblum and Gal Yona.
\newblock Multi-group agnostic {PAC} learnability.
\newblock In \emph{International Conference on Machine Learning}, pages
  9107--9115. PMLR, 2021.

\bibitem[Shabat et~al.(2020)Shabat, Cohen, and Mansour]{shabat2020sample}
Eliran Shabat, Lee Cohen, and Yishay Mansour.
\newblock Sample complexity of uniform convergence for multicalibration.
\newblock \emph{Advances in Neural Information Processing Systems},
  33:\penalty0 13331--13340, 2020.

\bibitem[Sivakumar et~al.(2021)Sivakumar, Bun, and
  Gaboardi]{sivakumar2021multiclass}
Satchit Sivakumar, Mark Bun, and Marco Gaboardi.
\newblock Multiclass versus binary differentially private pac learning.
\newblock \emph{Advances in Neural Information Processing Systems},
  34:\penalty0 22943--22954, 2021.

\bibitem[Tosh and Hsu(2022)]{tosh2022simple}
Christopher~J Tosh and Daniel Hsu.
\newblock Simple and near-optimal algorithms for hidden stratification and
  multi-group learning.
\newblock In \emph{International Conference on Machine Learning}, pages
  21633--21657. PMLR, 2022.

\bibitem[Valiant(1984)]{valiant1984theory}
Leslie~G Valiant.
\newblock A theory of the learnable.
\newblock \emph{Communications of the ACM}, 27\penalty0 (11):\penalty0
  1134--1142, 1984.

\bibitem[Vapnik and Chervonenkis(1971)]{vapnik1971vcdim}
V.~N. Vapnik and A.~Y. Chervonenkis.
\newblock On the uniform convergence of relative frequencies of events to their
  probabilities.
\newblock \emph{Theory Probab. Appl.}, 16:\penalty0 264--280, 1971.

\bibitem[Varshamov(1957)]{varshamov1957estimate}
Rom~Rubenovich Varshamov.
\newblock Estimate of the number of signals in error correcting codes.
\newblock \emph{Docklady Akad. Nauk, SSSR}, 117:\penalty0 739--741, 1957.

\bibitem[Vershynin(2018)]{MR3837109}
Roman Vershynin.
\newblock \emph{High-dimensional probability}, volume~47 of \emph{Cambridge
  Series in Statistical and Probabilistic Mathematics}.
\newblock Cambridge University Press, Cambridge, 2018.
\newblock ISBN 978-1-108-41519-4.
\newblock \doi{10.1017/9781108231596}.
\newblock URL \url{https://doi.org/10.1017/9781108231596}.
\newblock An introduction with applications in data science, With a foreword by
  Sara van de Geer.

\bibitem[Zhao et~al.(2021)Zhao, Kim, Sahoo, Ma, and Ermon]{zhao2021calibrating}
Shengjia Zhao, Michael Kim, Roshni Sahoo, Tengyu Ma, and Stefano Ermon.
\newblock Calibrating predictions to decisions: A novel approach to multi-class
  calibration.
\newblock \emph{Advances in Neural Information Processing Systems},
  34:\penalty0 22313--22324, 2021.

\end{thebibliography}
\end{document}